%% file: main.tex
\documentclass{article}

\usepackage{microtype}
\usepackage{graphicx}
\usepackage{subcaption}
\usepackage{booktabs} %

\usepackage{hyperref}
\input{preamble}

\usepackage[accepted]{icml2026}

\usepackage{amsmath}
\usepackage{amssymb}
\usepackage{mathtools}
\usepackage{amsthm}

\usepackage{enumitem}
\usepackage{booktabs}
\usepackage[capitalize,noabbrev]{cleveref}

\theoremstyle{plain}
\newtheorem{theorem}{Theorem}[section]

\newtheorem{lemma}[theorem]{Lemma}
\newtheorem{corollary}[theorem]{Corollary}
\theoremstyle{definition}

\newtheorem{assumption}[theorem]{Assumption}
\theoremstyle{remark}

\newcommand{\revised}[1]{\textcolor{black}{#1}}
\usepackage[textsize=tiny]{todonotes}

\icmltitlerunning{Optimal Regularization for Performative Learning}

\begin{document}

\twocolumn[
  \icmltitle{Optimal Regularization for Performative Learning}

  \icmlsetsymbol{equal}{*}

  \begin{icmlauthorlist}
    \icmlauthor{Edwige Cyffers}{yyy}
    \icmlauthor{Alireza Mirrokni}{sch}
    \icmlauthor{Marco Mondelli}{sch}
  \end{icmlauthorlist}

  \icmlaffiliation{yyy}{CNRS, LAMSADE, Dauphine PSL}
  \icmlaffiliation{sch}{Institute of Science and Technology Austria,
Klosterneuburg, Austria}

  \icmlcorrespondingauthor{Edwige Cyffers}{edwige.cyffers@cnrs.fr}
  \icmlcorrespondingauthor{Marco Mondelli}{marco.mondelli@ist.ac.at}

  \icmlkeywords{Performative Learning, Regularization, High Dimension Linear Regression}

  \vskip 0.3in
]

\printAffiliationsAndNotice{}  %

\begin{abstract}
In performative learning, the data distribution reacts to the deployed model—for example, because strategic users adapt their features to game it—which creates a more complex dynamic than in classical supervised learning. One should thus not only optimize the model for the current data but also take into account that the model might steer the distribution in a new direction, without knowing the exact nature of the potential shift. We explore how regularization can help cope with performative effects by studying its impact in high-dimensional ridge regression. 
We show that, while performative effects worsen the test risk in the population setting, when moving to the over-parameterized regime where the number of features exceeds the number of samples, the optimal regularization in the presence of performativity helps reduce the variance in the estimated parameters, thereby improving performance.
Furthermore, we prove that the optimal regularization scales with the overall strength of the performative effect, making it possible to set the regularization in anticipation of this effect. We illustrate this finding through empirical evaluations of the optimal regularization parameter on both synthetic and real-world datasets.
\end{abstract}

\input{body}

\section*{Impact Statement}

This work studies performative learning, which contributes to the study of the long-term impact of machine learning when deployed. Taking into account the risk of modifying the data distribution is likely to improve the impact of machine learning. However, the two evaluations we propose, namely evaluation on the initial distribution and evaluation on the final distribution, can both lead to unfortunate situations. Evaluating on the initial distribution might amount to turning a blind eye to the evolution of society, while the second might encourage intentionally steering the data distribution toward less diverse distributions that are easier to tackle. We mitigate this risk by deriving results for both regimes, which allows readers to pick the setting that makes sense for a given use case, in the absence of alternative evaluation criteria currently proposed for performative learning.

The choice to use high-dimensional tools to study performative learning should open interesting research directions and contribute to widening the tools available to study performativity. High-dimensional theory has proven useful for better explaining deep learning in classical machine learning, and we hope for similar benefits in the context of performative learning.

The focus on regularization offers rather practical and actionable takeaways from the paper, advocating for regularization as a convenient and agnostic mitigation for performativity. We trust readers to understand that regularization is only one tool among many for building more resilient models.

\section*{Acknowledgments}

This research was funded in whole or in part by the Austrian Science Fund (FWF) 10.55776/COE12. For the purpose of open access, the authors have applied a CC BY public copyright license to any Author Accepted Manuscript version arising from this submission. Alireza Mirrokni was an intern at ISTA while working on this project. Edwige Cyffers is supported by the National Research Agency under France 2030, reference “ANR-23-IACL-0008” and this work was partially done while visiting the Simons Institute for the Theory of Computing.

\section*{Conflict of Interest}
The authors have no conflict of interest.

\bibliography{example_paper}
\bibliographystyle{icml2026}

\newpage
\appendix
\onecolumn

\input{app}

\end{document}

%% file: preamble.tex
\usepackage{amsmath}
\usepackage{amssymb}
\newcommand{\E}{\mathbb{E}}

\newcommand{\di}{\operatorname{diag}}
\newcommand{\tr}{\operatorname{Tr}}
\newcommand{\thetapop}{\theta^*_{\rm{pop}}}
\newcommand{\thetaperfok}{\theta_0^{\text{perfo}}}

\newcommand{\N}{\mathcal{N}}
\newcommand{\R}{\mathbb{R}}

\newcommand{\D}{\mathcal{D}}
\usepackage{amsthm} 
\usepackage{graphicx}
\usepackage{subcaption}

\newcommand{\fixedriskeq}{\mathcal R_{\mathrm{eq}}}
\newcommand{\ofixedriskeq}{\overline{\mathcal R}_{\mathrm{eq}}}

\newcommand{\fixedriskeqs}{\mathcal R_{\mathrm{eq}}^*}
\newcommand{\lambdaeqs}{\lambda^*_{\mathrm{eq}}}

\newcommand{\lambdaeqsz}{\lambda^*_{\mathrm{eq}, D=0}}
\newcommand{\fixedriskeqa}{\mathcal R_{\mathrm{eq}, 1}}
\newcommand{\fixedriskeqb}{\mathcal R_{\mathrm{eq}, 2}}
\newcommand{\tfixedriskeqa}{\widetilde{\mathcal R}_{\mathrm{eq}, 1}}
\newcommand{\tfixedriskeqb}{\widetilde{\mathcal R}_{\mathrm{eq}, 2}}

\newcommand{\tfixedriskeq}{{\widetilde{\mathcal R}}_{\mathrm{eq}}}

%% file: body.tex
\section{Introduction}
\looseness=-1
When machine learning predictions affect user outcomes, deployed models can induce shifts in the data distribution. These shifts may result from strategic user behavior---where individuals try to secure favorable outcomes such as loan approval or college admission \citep{pmlr-v130-bechavod21a, JMLR:v24:22-0131, pmlr-v202-wang23ap}---or from self-fulfilling prophecies, for example in economic forecasts, recommendation systems, or predictive policing \citep{morgenstern1928wirtschaftsprognose, pmlr-v81-ensign18a,Ursu2015}. Such distribution shifts can undermine predictive performance over time, amplifying bias and reducing model quality \citep{pmlr-v202-taori23a,pmlr-v235-pan24d}. Performative learning \citep{perdomo_performative} addresses this feedback loop by parameterizing the data distribution with the same parameter as the model. This allows optimization to account not only for the training loss but also for the steering of the data distribution.

Unfortunately, while optimizing model parameters is a classical problem in machine learning, estimating the performative effect on the distribution is generally infeasible, as the distribution is unknown to the learner. Several algorithms have been proposed to approximate this effect \citep{miller_outside, izzo2021learn, NEURIPS2024_7de66547}, typically by assuming that it depends on a small number of parameters in a sufficiently simple way that can be learned across the first few deployments. However, these methods are limited to relatively toy examples in small dimensions and may be impractical in high-dimensional settings. In particular, many approaches require numerous repeated deployments, alternating between loss minimization and distribution steering. Yet in practice, deployment often happens only once after full training. This makes repeated risk minimization (RRM) \citep{perdomo_performative}---where one trains until convergence before deployment, and the number of deployments is small---the default in many applications, even though it remains largely unaddressed by existing mitigation methods. This motivates a shift away from exact estimation of performative effects toward the study of principled choices of loss functions and models, and in particular regularization is a natural and tractable candidate. 

In this work, we study how regularization mitigates performative effects in repeated retraining. Unlike estimation-based methods, regularization does not depend on a precise characterization of the distribution shift, avoids their limitations, and introduces little computational overhead. Prior work suggests its potential benefits: \citet{perdomo_performative} proved that retraining converges to an optimal solution under assumptions tied to the strong convexity of the loss, which ridge regularization can enforce; more recently, \citet{NEURIPS2024_7de66547} showed that, in classification, the performative optimum can be interpreted as a regularized version of the non-performative problem, with numerical evidence that ridge penalties perform well in small-dimensional classification tasks. However, these results, as most of the performative literature, do not cover the high-dimensional regime. In particular, regularization may encourage reliance on spurious features \citep{bombari2025spuriouscorrelationshighdimensional} in higher dimension, especially when such features are reinforced by performativity.

To better understand these tradeoffs, we study the role of ridge regularization in linear regression in the presence of performativity and spurious features. We consider both \emph{(i)} the population regime, with enough data to recover exactly the unknown vector of regression coefficients at each deployment, and \emph{(ii)} the over-parameterized regime, where the number of data samples is a fixed fraction of the number of parameters. This last setting, though simple, captures behaviors relevant to deep learning, such as double descent \citep{doi:10.1073/pnas.1903070116,hastie2022surprises}, benign overfitting \citep{Bartlett_2020} and adversarial robustness \citep{fawzi2016analysisclassifiersrobustnessadversarial, ribeiro2023regularization}. In performative learning, it also brings the additional advantage that parameters and data live in the same space, simplifying the encoding of performative effects. The theoretical framework we develop enables us to provide strong evidence for the effectiveness of regularization under performativity, and to show how regularization should be scaled.
More precisely, our contributions are summarized below.
\begin{enumerate}[leftmargin=1em]
\item In the population regime, we characterize how the risk depends on magnitude and direction of the performative effect, as well as on spurious features (Theorem \ref{thm:pop}). We find that the optimal regularization is proportional to the strength of the performative effect and it mitigates the performance loss due to performativity: zero excess risk is achieved with identity covariance and constant entries of the performative vector, while the risk remains significant in the presence of a complex covariance structure and highly variable entries of the performative vector (Corollary \ref{cor:pop}).
\item In the proportional regime with random data, we establish a deterministic equivalent of the performative fixed point, depending only on population covariance and regularization (Theorem \ref{thm:over}). The analysis of this deterministic equivalent then unveils a remarkable phenomenology: for small noise variance, the optimal regularization moves in the same direction as the performative effect on predictive features, while it moves in the opposite direction as the performative effect on spurious features; remarkably, the optimally-regularized risk improves in the presence of a performative effect that reinforces existing trends. 
\item  We illustrate these behaviors on both synthetic data and real-world datasets (Housing, LSAC), showing empirically that our findings provide valuable insights beyond the assumptions in the theory %
and extend to other regularizers.
\end{enumerate}

\vspace{-.5em}
\section{Related work}
\vspace{-.5em}

\paragraph{Performative learning.} Performative learning was introduced by \citet{perdomo_performative}, who showed that retraining converges under assumptions including strong convexity. Subsequent works demonstrated that retraining can enable adaptation over time \citep{li2022, Drusvyatskiy2023, brown2022performative, wang2023constrainedoptimizationdecisiondependentdistributions} but can also fail dramatically \citep{miller_outside, izzo2021learn, NEURIPS2024_7de66547}. Several works evaluate performance on the initial distribution rather than the induced one \citep{demirel2024adjustingpretrainedbackbonesperformativity, Tsoy2025OnTI,pmlr-v202-taori23a}. Label shift is standard in domain adaptation \citep{label2, Cai2021label} and can encode performative effects, such as placebo effects or traffic prediction in \citep{nikolalabel, hardt2023performative}. The role of model choice has been studied in the related setting of collective action \citep{pmlr-v235-ben-dov24a}. Improved performance due to performativity was observed by \citet{pmlr-v130-bechavod21a}, and our work provides further evidence supporting this claim. Finally, performative learning effects tend to be harder to learn in high-dimensional settings, as noted by \citet{pmlr-v162-jagadeesan22a, bracale2025learningdistributionmapreverse}. Our work is the first to use tools from high-dimensional statistics to study performative learning, to our knowledge.

\vspace{-.7em}

\paragraph{High-dimensional regression and role of ridge regularization.} The high-dimensional setting where numbers of features and samples %
scale proportionally was considered by a rich line of work: the test error of ridgeless and ridge regression is characterized by \citet{hastie2022surprises, wu2020optimal,richards2021asymptotics,tsigler2023benign}; %
max-margin classification is studied by   \citet{montanari2019generalization,deng2022model}, 
model compression by \citet{chang2021provable}, distribution shift by \citet{patil2024optimal,mallinar2024minimumnorm}, transfer learning by  \citet{yang2020precise,song2024generalization} and learning from surrogate data by  \citet{kolossovtowards,jain2024scaling,rezaei2025high}.The role of ridge regularization has also been studied. \citet{hastie2022surprises} optimally tune the ridge penalty, while \citet{richards2021asymptotics} give conditions for the optimality of ridgeless interpolation. The sign of the optimal ridge penalty was studied for the standard in-distribution regression setup \citep{wu2020optimal, tsigler2023benign}, as well as out-of-distribution \citep{patil2024optimal}: these works give conditions under which the optimal ridge is negative, associating the phenomenon of negative optimal regularization to over-parameterization. Our paper shows that such a phenomenon occurs also in the population setting, due to performative effects.
The distribution of the empirical risk minimizer was established by  \citet{han2023distribution}. Leveraging this characterization, spurious correlations were studied by \citet{bombari2025spuriouscorrelationshighdimensional} and weak-to-strong generalization by \citet{ildizhigh}. We will also build on these tools to analyze the risk of repeated risk minimization.

\vspace{-.3em}
\section{Preliminaries and problem setup}
\vspace{-.3em}

In this section, we introduce our performative regression setting. We consider a sequence of model deployments $(\theta_k)_{k\ge 0}$, and let $\D(\theta)$ be the dataset generated in reaction to the deployment of $\theta$. At each deployment, $n$ new samples are collected, and the model is fully retrained. %
This setting, known as repeated risk minimization (RRM) \citep{perdomo_performative}, reflects real-world scenarios where deployments are costly and thus limited in number. It also aligns with the fact that convergence to the fixed point is fast and requires only a few iterations to reach equilibrium in practice.
We encode the performative effect as a shift in the label, where each feature’s contribution varies depending on an additional linear term in the model parameter. %

\begin{assumption}[Regression performative model]
For $\theta \in \R^p$, samples from $\D(\theta)$ are taken i.i.d.\ with features $x$ having zero mean and covariance $\Sigma$ %
drawn independently of $\theta$ and with the label $y$ given by
\vspace{-.3em}
\begin{equation}\label{eq:data}    
y = x^\top \thetapop + x^\top D \theta + w, \quad w \sim \N(0, \sigma^2).
\end{equation}
We assume $p = 2d$, $(\thetapop)^{\top} = (a^\top, 0)$ with $a$ having \revised{zero mean and} covariance $I_d/d$, and $D = \di(b, c)$ where $b, c \in \R^d$ with $\|b\|_{\infty}, \|c\|_{\infty} < 1$.
\label{assum:model}
\end{assumption}

This model generalizes the one-dimensional setting (Example 2.2) in \citet{perdomo_performative}, where labels follow a binomial distribution with parameter $\tfrac{1}{2} + x \theta^*  + x \bar{b} \theta$, for $\theta^* \in (0, \tfrac{1}{2})$ and $\bar{b} < \tfrac{1}{2}-\mu$. Focusing on label shifts is natural in regression: it keeps the feature distribution centered and unchanged across deployments (this can be enforced via pre-processing) despite performative effects, and it can encode scenario such as placebo effects \citep{nikolalabel}. We do not cover feature shifts. However, when a feature shift affects all data points in a regression task, replacing $x$ by $x+\theta$ can be absorbed by centering the data, and prior work provides numerical evidence that regularization remains beneficial when feature shifts affect only one class in binary classification~\cite{NEURIPS2024_7de66547}.  %

The performative term $x^\top D \theta$ enforces coordinate-wise effects. This is %
consistent with %
previous works \citep{NEURIPS2024_7de66547,izzo2021learn, hardt2023performative} and also %
close to the model 
$y = x^\top \thetapop + \mu^{\top} \theta + w$ studied by \citet{miller_outside}, where the performative effect does not depend on $x$ but only on a fixed vector $\mu$. Assuming linearity in $\theta$ is reasonable, as performative effects are expected to be moderate to avoid iterations to diverge. We specify in the rest of the paper when the fact that $D$ is diagonal is needed. Intuitively, diagonal coefficients can be interpreted directly as the modifications made by a strategic agent, depending on how the feature is used and the cost of modifying it. Most existing methods impose explicit constraints on the performative effect \citep{miller_outside}, and our setting is no more restrictive: we only require $\|b\|_\infty, \|c\|_\infty < 1$. %
We set the second half of $\thetapop$ to zero to represent spurious features, and $c$ captures the corresponding performative effect. %
This %
enables us to express correlations between predictive and spurious features via %
the block structure of the covariance
\vspace{-.3em}
\begin{equation*}\label{eq:block}
\Sigma=\begin{bmatrix}\Sigma_1&\Sigma_{12}\\ \Sigma_{12} &\Sigma_2\end{bmatrix},  
\end{equation*}
where $\Sigma_1$ denotes the covariance for the predictive part, $\Sigma_2$ the covariance for the spurious part, and $\Sigma_{12}$ the covariance between the two blocks. 
Under this setting, RRM corresponds to solving %
\begin{equation}
    \hspace{-.2em}\theta_{k} = \arg\min_{\theta\in\mathbb{R}^p}\left\{\frac{1}{2n}\sum_{i=1}^n \ell(x_i^{\scriptscriptstyle(k-1)}, y_i^{\scriptscriptstyle(k-1)}; \theta) \hspace{-.2em}+\hspace{-.2em}\frac{\lambda}{2}\|\theta\|_2^2\right\},
    \label{eq:ERM}
\end{equation}
where $\{(x_i^{(k-1)},y_i^{(k-1)})\}_{i=1}^n\stackrel{\mathrm{i.i.d.}}{\sim}\mathcal{D}(\theta_{k-1})$ and $\ell$ is the squared loss. This defines a recurring sequence in both population and over-parameterized regimes. In the population case, the sequence converges in parameter space to a fixed vector $\thetapop$  (\Cref{sec:pop}), while in the over-parameterized case the vector varies at each iteration but the excess risk still converges deterministically (\Cref{sec:over}).

We evaluate the test risk when the final model is deployed on the untouched distribution $\D(\theta = 0)$. %
Testing on $\D(\theta = 0)$ is particularly relevant for long-term fairness, as it prevents bias amplification over time \citep{pmlr-v81-ensign18a, pmlr-v202-taori23a} or steering the distribution toward undesirable regimes that decrease the risk by collapsing the data distribution to make it easier to predict (e.g., reducing entropy or producing a single possible label in a classification task) \citep{Tsoy2025OnTI,demirel2024adjustingpretrainedbackbonesperformativity}. This choice also enables testing whether regularization increases reliance on spurious features \citep{bombari2025spuriouscorrelationshighdimensional}, which one wants to avoid. 
We thus aim to minimize the following excess risk:
\vspace{-.3em}
\begin{equation}
\begin{aligned}
    \mathcal{R}(\Sigma, \theta, \thetapop) := &\E_{\D(\theta = 0)}\left[ (y - x^\top \theta)^2]\right] - \sigma^2 \\=&
    \|\Sigma^{1/2}(\theta-\thetapop)\|_2^2 , 
\end{aligned}
    \label{eq:risk}
\end{equation}
where %
we have subtracted the Bayes risk $\sigma^2$. %
In \Cref{sec:pop}, we analyze this risk in the population setting, where enough data is available to exactly recover the parameter vector and  the optimal solution is $\thetapop$ (as suggested by the notation). %
In \Cref{sec:over}, we then focus on the %
over-parameterized regime where $p>n$. %

We note that, in the population setting, testing on $\D(\theta)$ gives zero excess risk, thus trivializing the problem. We provide results for testing on $\D(\theta)$ in the over-parameterized setting at the end of Section \ref{sec:over}. 

\vspace{-.5em}
\section{Analysis in the population setting}
\label{sec:pop}
\vspace{-.5em}

In this section, we tackle the population regime where there are enough samples from $\D(\theta_k)$ at each deployment to compute exactly the next regressor, as would typically happen in a low-dimensional setting. The sequence $(\theta_k)_k$ is thus deterministically defined by
\vspace{-.3em}
\begin{equation}\label{eq:recpop}
\begin{aligned}
    \theta_k &= (\Sigma +\lambda I_p)^{-1} \E_{(x,y)\sim \D(\theta_{k-1})}[xy]\\ &=  (\Sigma +\lambda I_p)^{-1} (\Sigma \thetapop + \Sigma D \theta_{k-1}),
\end{aligned}
\end{equation}
where in the second equality we plug back the definition of the current data distribution in (\ref{eq:data}).

\begin{figure*}[h]
    \centering
    \begin{subfigure}{0.32\textwidth}
        \includegraphics[width=\linewidth]{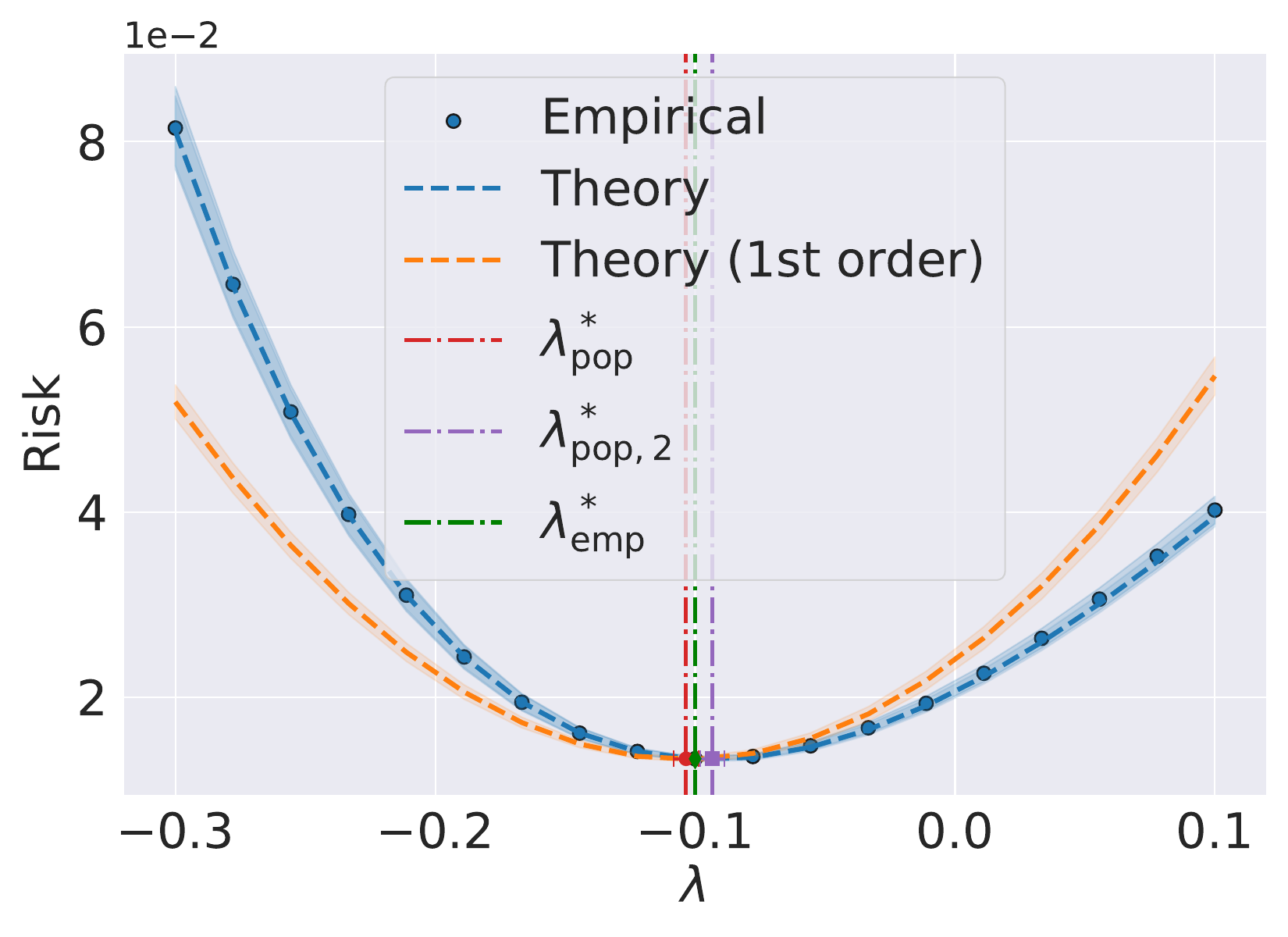}
        \caption{\(\bar b = -0.1\)}
        \label{fig:expandtheo}
    \end{subfigure}
    \hfill
    \begin{subfigure}{0.32\textwidth}
        \includegraphics[width=\linewidth]{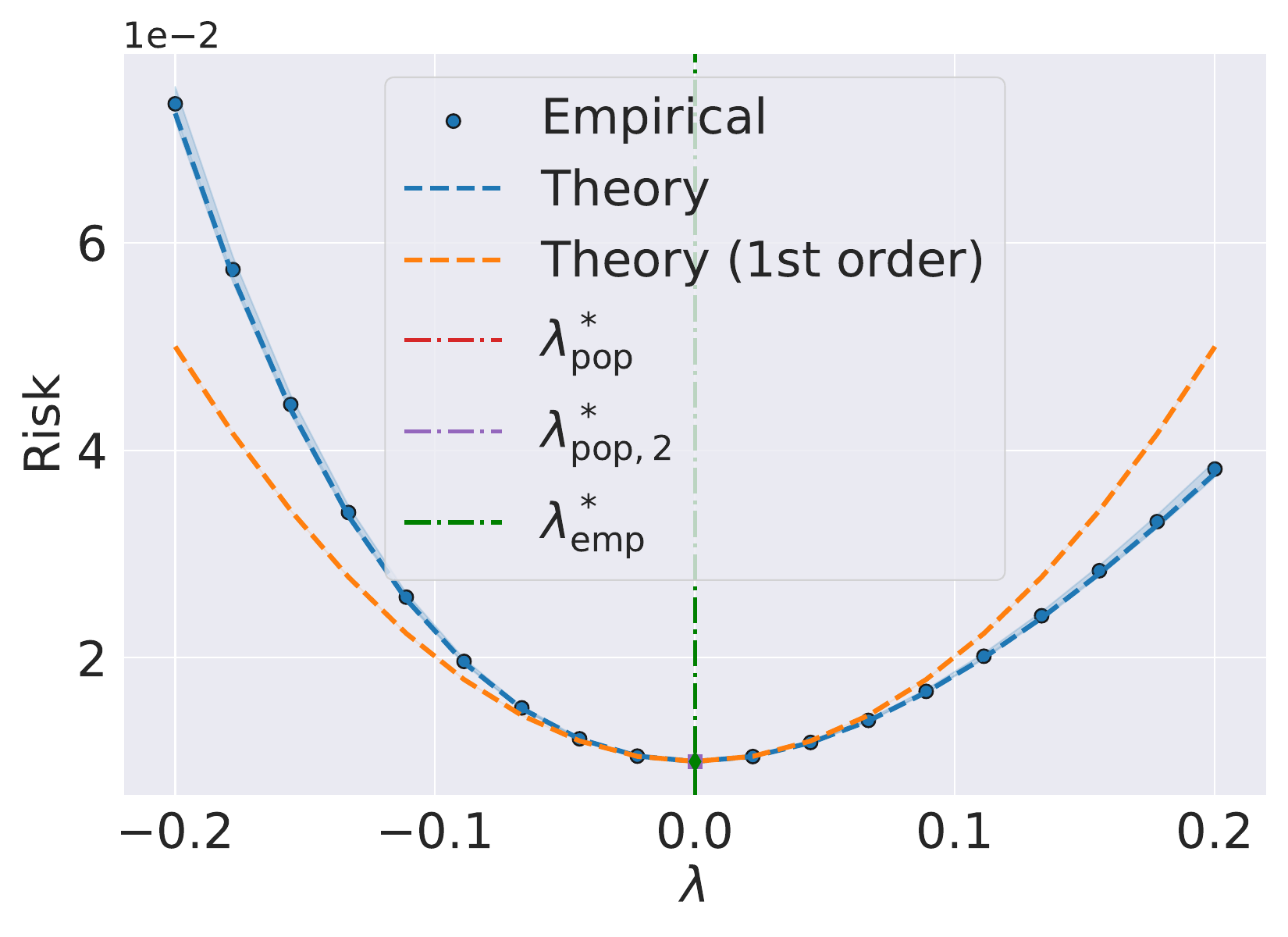}
        \caption{\(\bar b = 0\)}
        \label{fig:covariance}
    \end{subfigure}
    \hfill
    \begin{subfigure}{0.32\textwidth}
        \includegraphics[width=\linewidth]{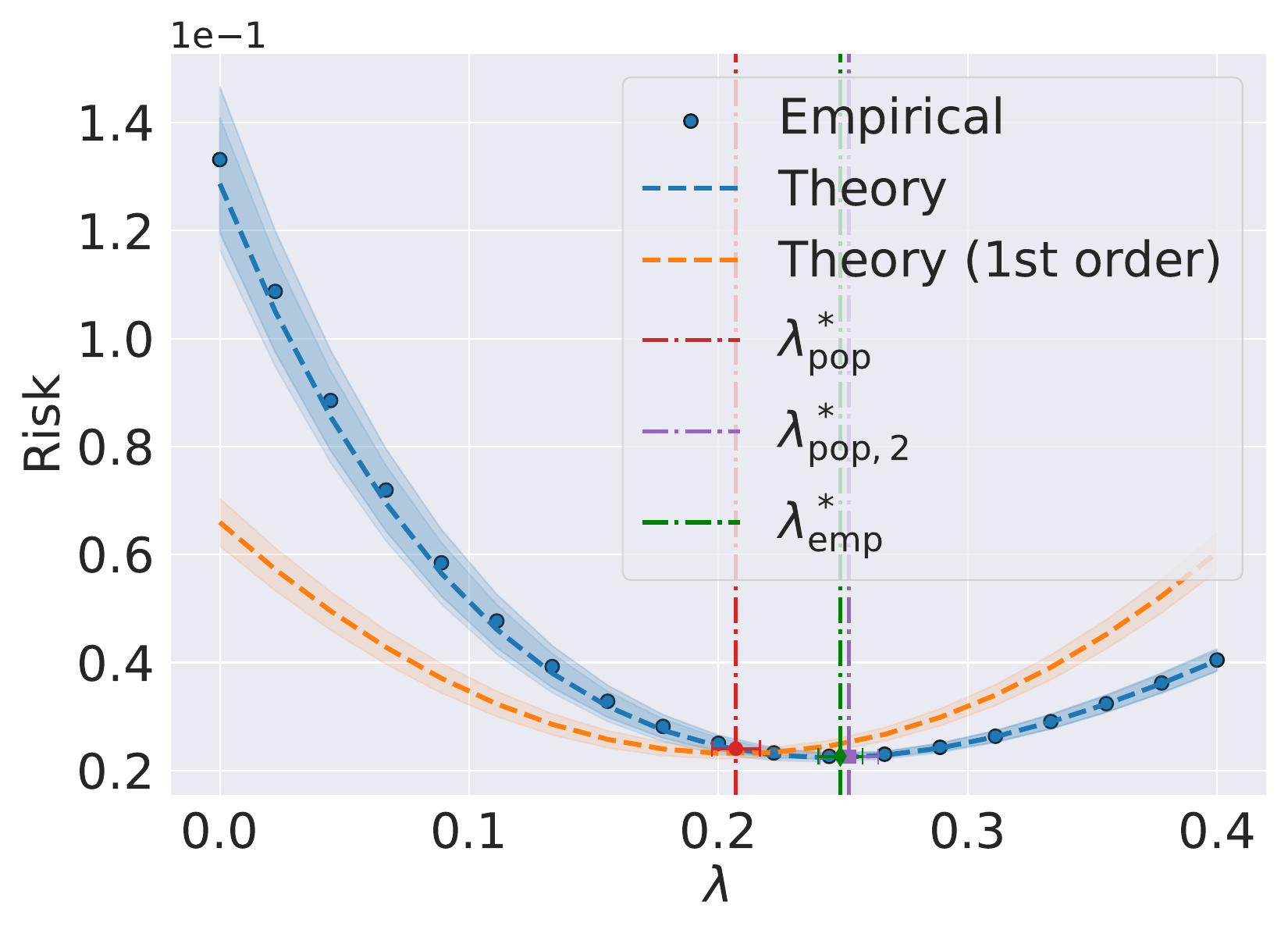}
        \caption{\(\bar b = 0.2\)}
        \label{fig:c}
    \end{subfigure}
\vspace{-.3em}
    \caption{Excess risk at the performative fixed point $\theta^{\infty}$ in (\ref{eq:fppop}), as a function of ridge regularization $\lambda$, for $d = 100$, $\Sigma=I_p$, entries of $b$ uniform in \(\left[\min\{0, 2 \bar b\}, \max\{0, 2 \bar b\} \right]\), $c=0$ and $\sigma = 0.1$. Empirical values (blue dots) are computed from 20 i.i.d.\ trials on $a$ and 5 i.i.d.\ trials on $b$, with error band at 1 standard deviation. Theoretical predictions (blue dashed curves) are from (\ref{eq:fppopavg}) and match perfectly empirical values. First-order approximations (orange dashed curves) are given by $\widetilde{\mathcal R}_{\rm pop}(D, \lambda, \Sigma)$ in (\ref{eq:popapex}) and still provide a good match when $\lambda$ is near-optimal. The green vertical line is the optimal regularization obtained by numerically optimizing the excess risk of $\theta^\infty$ ($\lambda^*_{\rm emp}$), the red one is the first-order approximation ($\lambda_{\rm pop}^*$ from (\ref{eq:optpopap1})) and the violet one the second-order approximation ($\lambda^*_{{\rm pop}, 2}$ minimizing (\ref{eq:pophigh})).}
    \label{fig:pop1}
\end{figure*}

\vspace{-.5em}

\paragraph{Excess risk at the performative fixed point.}
By unrolling (\ref{eq:recpop}), \revised{for any arbitrary (possibly non-diagonal) matrix $D$}, we have that the sequence $(\theta_k)_k$ converges at an exponential rate to the fixed point 
\vspace{-.3em}
\begin{equation}\label{eq:fppop}
\theta^{\infty} = (I_p +\lambda \Sigma^{-1} - D)^{-1}\thetapop .    
\end{equation}
The formal statement, including an explicit convergence rate, is deferred to Lemma \ref{lemma:cr} in Appendix \ref{app:pop}. %
By inserting (\ref{eq:fppop}) into (\ref{eq:risk}) and taking the expectation with respect to $\thetapop$, we have
\vspace{-.3em}
\begin{equation}\label{eq:fppopavg}
\begin{split}
\E_{\thetapop}\mathcal{R}(\Sigma, \theta^{\infty}, \thetapop) ={}&\E_{\thetapop}[(\thetapop)^\top A^\top \Sigma A \thetapop] \\&{}=\frac{1}{d} \tr\left[(A^\top \Sigma A)_{1}\right],
\end{split}
\end{equation} 
where we define $A := (\Sigma + \lambda I_p- \Sigma D )^{-1} \Sigma - I_p$, use $(\thetapop)^{\top} = (a^\top, 0)$ with $a$ having \revised{zero mean and} covariance $I_d/d$ and, given a $p\times p$ matrix $M$, denote by $(M)_1$ its top-left $d\times d$ block. %
This leads to the following approximation for the excess risk, proved in Appendix~\ref{app:pop}. 
\begin{theorem}[Excess risk -- population]\label{thm:pop}
    Let $F = D - \lambda \Sigma^{-1}$. Then, %
    we have %
\vspace{-.3em}
\begin{equation}\label{eq:popapex}
    \begin{split}
\E_{\thetapop}\mathcal{R}(\Sigma, \theta^{\infty}, \thetapop) =
\widetilde{\mathcal R}_{\rm pop}(D, \lambda, \Sigma)  + O(\|F\|_{\mathrm{op}}^2),\\
\widetilde{\mathcal R}_{\rm pop}(D, \lambda, \Sigma)  :=  \frac{1}{d}\tr[\di(b^2)\Sigma_1]  -  2 \lambda \bar{b}  +  \frac{1}{d}\lambda^2 \tr(S_1),
    \end{split}
\end{equation}
where $\|\cdot\|_{\mathrm{op}}$ denotes the operator norm, $\bar{b} := \frac{1}{d}\tr[\di(b)] =
\frac{1}{d}\sum_{i=1}^d b_i$, $b^2:=[b_1^2, \ldots, b_d^2]\in\mathbb R^d$ and $S_1 = (\Sigma_1 - \Sigma_{12}\Sigma_2^{-1} \Sigma_{21})^{-1}$ is the Schur complement of $\Sigma$.
\end{theorem}

The matrix $F = D - \lambda \Sigma^{-1}$ naturally appears in the computation, and its norm can be bounded explicitly as a function of $b$, $c$, $\lambda$ and $\Sigma$ by applying  Weyl's inequality (see Lemma \ref{lemma:weyl} in Appendix \ref{app:pop}) and ensuring that the approximation $O(\|F\|_{\mathrm{op}}^2)$ is tighter than  $O(\max(\|b\|_{\infty}, \|c\|_{\infty}, \lambda)^2)$, since $b$ and $c$ can be partially canceled by $\lambda$. In fact, this cancellation occurs when $\lambda$ is near-optimal, resulting in an accurate approximation, see Figure \ref{fig:pop1}.

\vspace{-.5em}

\begin{figure*}[tbh]
    \centering
    \begin{subfigure}{0.32\textwidth}
        \includegraphics[width=\linewidth]{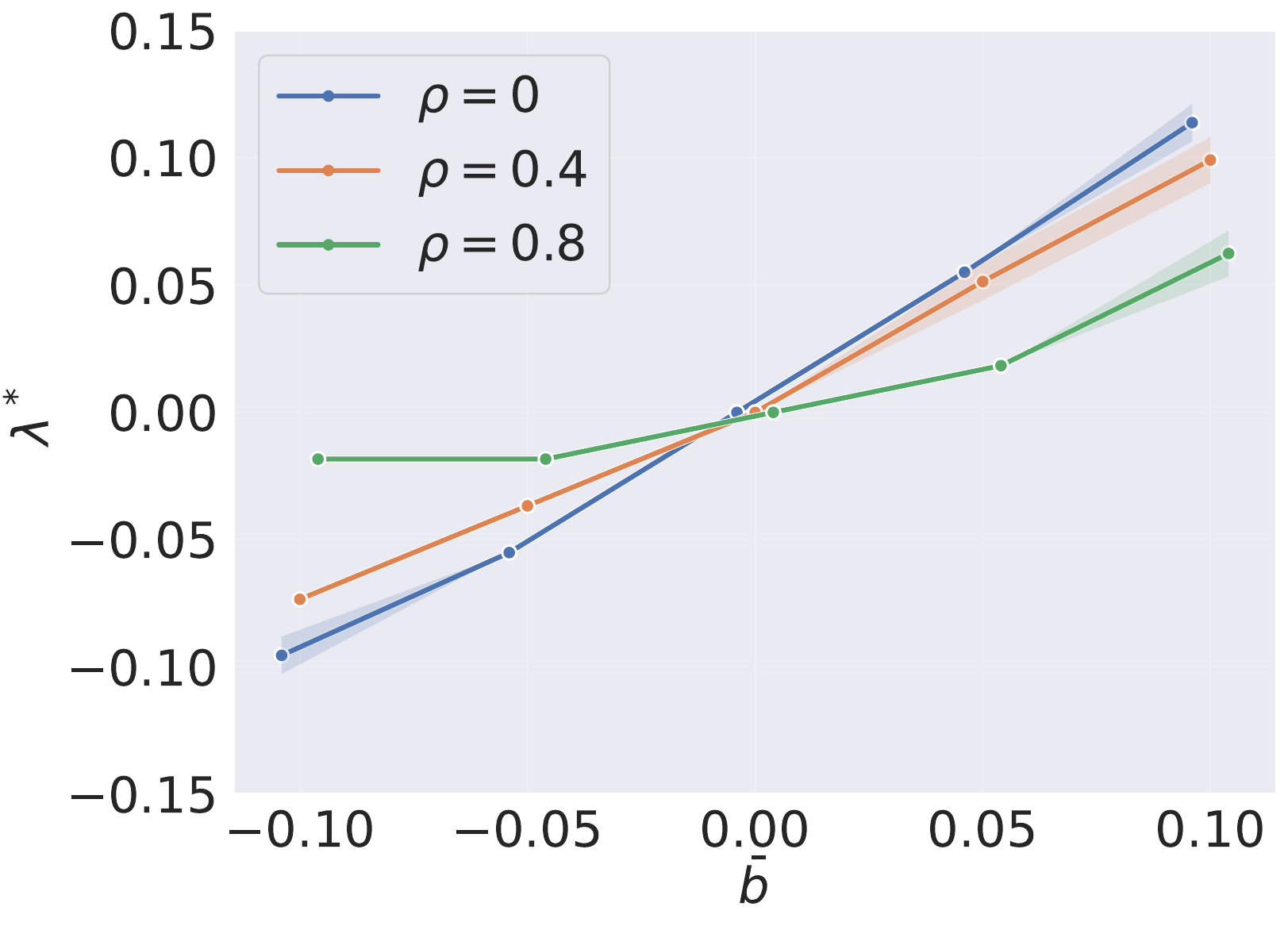}
        \caption{Spurious correlations}
        \label{fig:spurious}
    \end{subfigure}
    \hfill
    \begin{subfigure}{0.32\textwidth}
        \includegraphics[width=\linewidth]{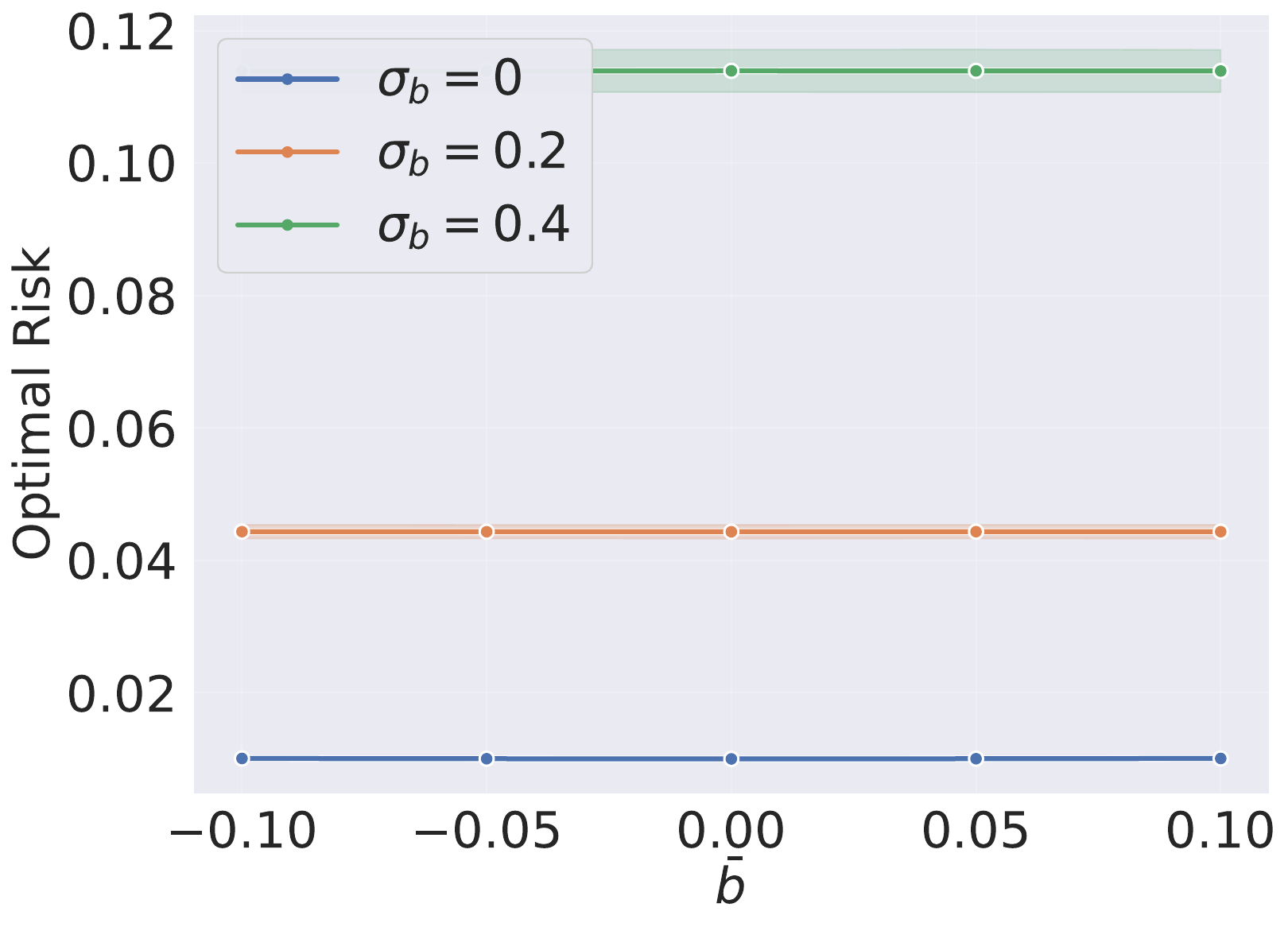}
        \caption{Variance of entries of $b$}
        \label{fig:variance}
    \end{subfigure}
    \hfill
    \begin{subfigure}{0.32\textwidth}
        \includegraphics[width=\linewidth]{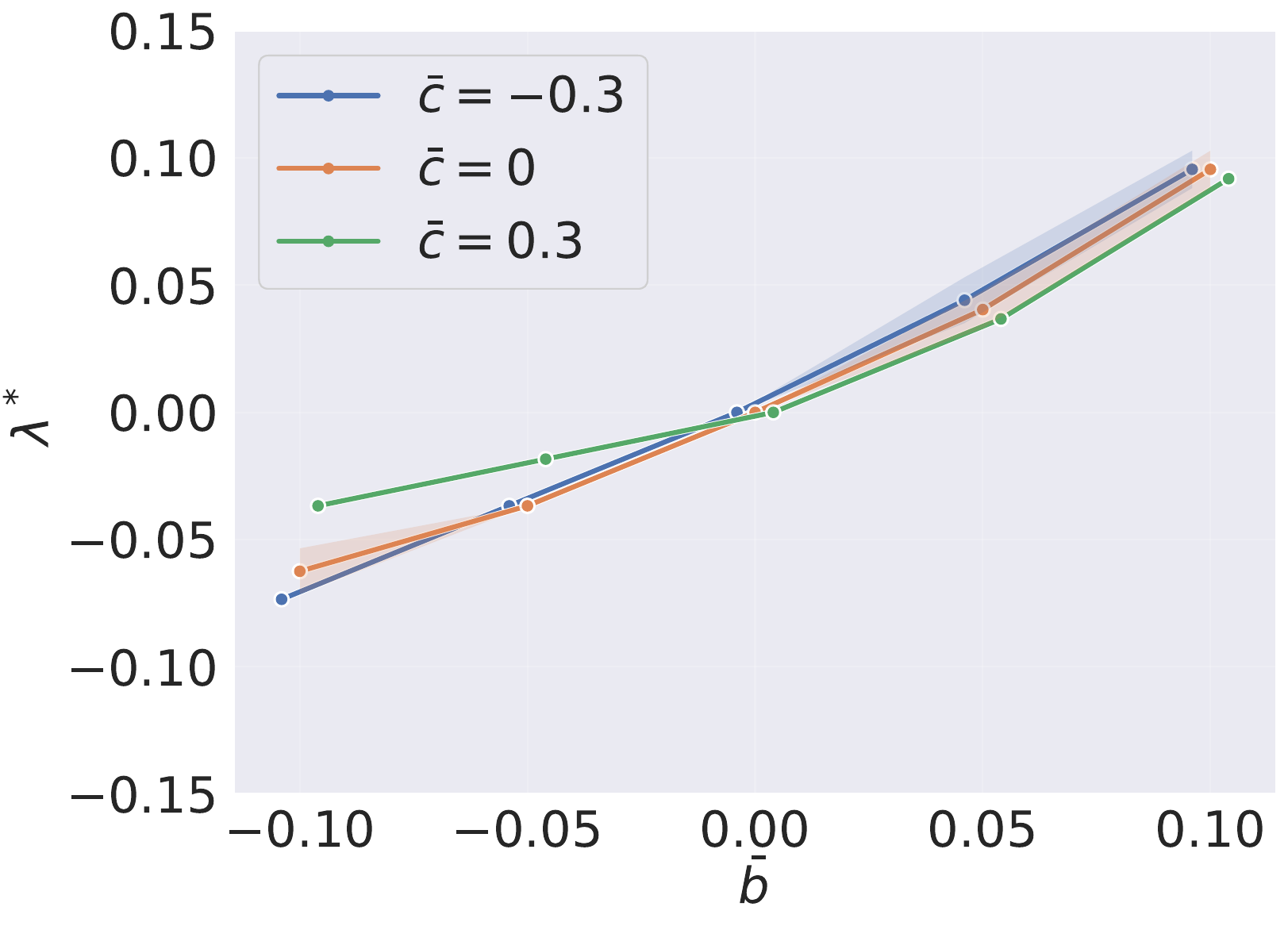}
        \caption{Spurious performative effect}
        \label{fig:spurperfo}
    \end{subfigure}
       \vspace{-.3em} \caption{Optimal regularization and risk for the performative fixed point $\theta^{\infty}$ in (\ref{eq:fppop}), with $d = 100$, $\Sigma_1=\Sigma_2=I_d$, $\Sigma_{12}=\rho I_d$. Values are computed from 20 i.i.d.\ trials on $a$ and 5 i.i.d.\ trials on $b$, with error band at 1 standard deviation. (a) Optimal regularization as a function of $\bar b$ for $\rho\in \{0, 0.4, 0.8\}$. The entries of $b$ are uniform in \(\left[\min\{0, 2 \bar b\}, \max\{0, 2 \bar b\} \right]\) and $c = 0$. (b) Optimal risk as a function of $\bar b$. Different curves correspond to
       different variances $\sigma_b^2$ of the entries of $b$, which  are uniform in \(\left[\bar b  - \sigma_b \sqrt{3}, \bar b + \sigma_b \sqrt{3} \right]\) for $\sigma_b\in\{0, 0.2, 0.4\}$. We pick $c=0$ and $\rho=0$. We note that, when 
       $\rho=0$, ${\mathcal R}_{\rm pop}^*(D, \Sigma)$ equals the empirical variance of the entries of $b$ and, as such, it does not depend on $\bar b$.  (c) Optimal regularization as a function of $\bar b$ for $\bar{c}\in \{-0.3, 0, 0.3\}$. The entries of $b$ are uniform in \(\left[\min\{0, 2 \bar b\}, \max\{0, 2 \bar b\} \right]\), the entries of $c$ are uniform in \(\left[\min\{0, 2 \bar c\}, \max\{0, 2 \bar c\} \right]\),  and $\rho = 0.5$.}
    \vspace{-1em}
    \label{fig:pop2}
\end{figure*}

\paragraph{Optimal regularization and optimally regularized risk.}

Leveraging the risk expression for small $F$ of Theorem \ref{thm:pop}, 
we next study the behavior of the optimal regularization and of the corresponding optimal risk. 
Formally, define
\vspace{-.3em}
\begin{equation}\label{eq:optpopdef}
\begin{split}
    \lambda^*_{\rm pop}(D, \Sigma):=\arg\min_{\lambda\in\mathbb R}\widetilde{\mathcal R}_{\rm pop}(D, \lambda, \Sigma),\\
        {\mathcal R}_{\rm pop}^*(D, \Sigma):=\min_{\lambda\in\mathbb R}\widetilde{\mathcal R}_{\rm pop}(D, \lambda, \Sigma).
\end{split}
\end{equation}
From (\ref{eq:popapex}), we note that $\widetilde{\mathcal R}_{\rm pop}(D, \lambda, \Sigma)$ is quadratic in $\lambda$, so the minimization in (\ref{eq:optpopdef}) can be solved explicitly, leading to the expressions below. %

\begin{corollary}[Optimal regularization -- population]\label{cor:pop}
    In the setting described above, we have %
\vspace{-.3em}
    \begin{equation}
        \begin{split}
    \label{eq:optpopap1}     
    \lambda^*_{\rm pop}(D, \Sigma) ={}& \frac{\bar{b}d}{\tr(S_1)}, \\  {\mathcal R}_{\rm pop}^*(D, \Sigma)={}&\frac{1}{d}\tr(\di(b^2)\Sigma_1) - \frac{\bar{b}^2d}{\tr(S_1)} .
    \end{split}
    \end{equation}
\end{corollary}
These formulas call for several comments. First, the optimal regularization $\lambda^*_{\rm pop}(D, \Sigma)$ is proportional to the strength of the performative effect $\bar{b}$, see Figure  \ref{fig:spurious} (and also the location of the minima in Figure \ref{fig:pop1}). %
The fact that $\lambda^*_{\rm pop}(D, \Sigma)$ grows with $\bar b$ captures an effect common in practice: when performativity reinforces existing trends%
—corresponding to ``rich-get-richer'' phenomena, such as a feature becoming more important over successive deployments—the optimal regularizer increases and helps to limit this effect. Conversely, when the performative effect already mitigates the influence of some feature, the optimal solution calls for less regularization. In the population case, this corresponds to negative regularization (see Figure \ref{fig:expandtheo}) which, although less common, has also been studied in the literature \citep{wu2020optimal,tsigler2023benign,patil2024optimal}. Spurious correlations tend to reduce the optimal regularization, although their effect is mild, see Figure \ref{fig:spurious}. 

Second, the optimal risk ${\mathcal R}_{\rm pop}^*(D, \Sigma)$ is always positive and, thus, worse than in the non-performative scenario, where it is zero. More specifically, zero excess risk can only be reached with a non-zero $b$ if $\Sigma = I_p$ and $b$ is aligned with the all-1 vector, see the blue line in Figure \ref{fig:variance}. Intuitively, as regularization impacts all features equally, it better compensates performativity in this uniform case. If the variance in the entries of $b$ grows, then ${\mathcal R}_{\rm pop}^*(D, \Sigma)$ increases, see Figure \ref{fig:variance}. %
We finally note that ${\tilde{\mathcal R}}_{\rm pop}(D, \lambda, \Sigma)$ does not depend on $c$ and, in fact, the effect of $c$ is only visible at higher order, as seen in this formula, proven in Appendix \ref{app:pop}:
\begin{equation}\label{eq:pophigh}    
\begin{aligned}
&\E_{\thetapop}\mathcal{R}(\Sigma, \theta^{\infty}, \thetapop) =\frac{1}{d} \biggl(\hspace{-0.15em} -2\lambda^3\tr \bigl[\left( \Sigma^{-2}\right)_1\bigr] \hspace{-0.15em}+\hspace{-0.15em} \lambda^2 \bigl(\tr\!\left[S_1\right]\\& + 6 \tr \bigl[\di(b) S_1\bigr]\bigr) - \lambda \Bigl(2\tr\!\bigl[\di(b)\Sigma_1\di(b)S_1\bigr] \\&
         + 2\tr\!\bigl[\di(b)\Sigma_{12}\di(c)S_{21}\bigr]
       + 2 d \bar b + 4\tr[\di(b^2)] \Bigr)\\
    & + \tr\left[\di(b^2)\Sigma_1\right] + 2\tr\left[\di(b^3) \Sigma_1\right]\biggr) + O(\|F\|_{\mathrm{op}}^4),
\end{aligned}
\end{equation}
with $S_{21}^\top = -(\Sigma_1 - \Sigma_{12}\Sigma_2^{-1} \Sigma_{21})^{-1}\Sigma_{12} \Sigma_2^{-1}$. %
Figure \ref{fig:pop1} illustrates that the minimizer of this second-order approximation ($\lambda^*_{{\rm pop}, 2}$) is close to the minimizer obtained numerically ($\lambda^*_{\rm emp}$). While $c$ tends to steer the optimal regularizer in the opposite direction (less regularization in the case of a self-reinforcing performative effect), it does so only through the cross term $2\tr\!\left[\di(b)\Sigma_{12}\di(c)S_{21}\right]$, which also depends on $b$. \Cref{fig:spurperfo} shows that $c$ moves the optimal regularization in a direction opposite to its sign, but its effect remains rather limited.

\section{Analysis in the over-parameterized setting}
\label{sec:over}

Next, we consider the case where $n, p$ are both large and scale proportionally, with $p/n=\kappa>1$. All \emph{constants} (e.g., $R, M$) are intended to be positive values independent of $n, p$. For mathematical convenience, we opt for a different normalization w.r.t.\ (\ref{eq:ERM}), and the estimator $\theta_{k}$ is given by
\begin{equation}    
\theta_{k} = \arg\min_{\theta\in\mathbb{R}^p}\left\{\frac{1}{2p}\sum_{i=1}^n \ell\left(x_i^{(k-1)}, y_i^{(k-1)}; \theta\right) +\frac{\lambda}{2}\|\theta\|_2^2\right\}.
\end{equation}
Solving for $\theta_{k}$ yields
\begin{equation}\label{eq:thetak}    
\theta_{k} = \frac{1}{p} \left(\frac{1}{p} X^{\scriptscriptstyle(k-1)  \top} X^{\scriptscriptstyle(k-1)} + \lambda I_p\right)^{-1} \hspace{-1em} X^{\scriptscriptstyle(k-1) \top} y^{\scriptscriptstyle(k-1)},
\end{equation}
where $X^{(k-1)}=[x_1^{(k-1)}, \ldots, x_n^{(k-1)}]\in \mathbb R^{n\times p}$ and $y^{(k-1)}=[y_1^{(k-1)}, \ldots, y_n^{(k-1)}]\in\mathbb R^n$. Note that, when $\theta_k$ is given by \eqref{eq:thetak}, the risk $\mathcal{R}(\Sigma, \theta_k, \thetapop)$ as defined in \eqref{eq:risk} is a random quantity since the data $\{X^{(\ell)}\}_{\ell=1}^{k-1}$ and the noise contained in the labels $\{y^{(\ell)}\}_{\ell=1}^{k-1}$ are random. This makes it challenging to characterize optimal ridge penalty and optimally-tuned risk. To address the challenge, we first establish a \emph{deterministic} equivalent of the risk at the performative fixed point. We next optimize such deterministic equivalent and study the effect of performativity on the optimal regularization.  

\paragraph{Deterministic equivalent of the performative fixed point.} First, note that, if we regard the performative effect as small and aim at characterizing its effect on the fixed point up to the leading (first) order, it suffices to do two iterations of the recursion in \eqref{eq:thetak}. In fact, the labels $y^{(k-1)}$ are linear in $D\theta_{k-1}$, so we expect that, after two iterations, the performative fixed point is reached up to fluctuations of order $O(\|D\|_{\rm op}^2)$. Now, the risk after two iterations is still a random quantity, so we apply techniques from \cite{han2023distribution,ildizhigh} to derive a %
deterministic equivalent. The formal statement is below and the proof is deferred to Appendix \ref{app:pf}.
\begin{theorem}[Excess risk -- over-parameterized]\label{thm:over}
\revised{Let Assumption \ref{assum:model} hold with $x \sim \N(0, \Sigma)$.} Let $R>0$ be a constant s.t.\ $\|\thetapop\|_2, \|\theta_0\|_2\le R$. Assume that $\kappa, \sigma, \lambda\in (1/M, M)$ and $\|\Sigma\|_{\mathrm{op}},\ \|\Sigma^{-1}\|_{\mathrm{op}} \le M$ for some constant $M>1$. Then, there exists a constant $C=C\left(M, R\right)$ such that for any $\delta \in (0,1/2]$, with probability at least $1-Cpe^{-p\delta^{4}/C}$, %
\begin{equation}    
\left|\mathcal{R}(\Sigma, \theta_{2}, \thetapop)-\fixedriskeq\left(\Sigma, \thetapop, D, \lambda\right)\right|\le \delta+O(\|D\|_{\mathrm{op}}^2),
\end{equation}
where
\begin{equation}   \label{eq:defdet} 
\resizebox{.49\textwidth}{!}{$\displaystyle
\begin{aligned}
&\fixedriskeq\left(\Sigma, \thetapop, D, \lambda\right)
= 
\tau
\langle \thetapop, \Xi\left(\tau I_p
-2
\Xi
\Sigma^2 D
\right)\Sigma \Xi
\thetapop\rangle
\\
&
\hspace{-.6em}+\hspace{-.2em} \kappa
\tr\left[\Sigma^{2} \Xi^{2}\right]
\hspace{-.2em}\frac{
\sigma^{2}
\hspace{-.2em}+\hspace{-.2em} 
\tau^{2}
\langle \thetapop,\Xi 
\left(I_p + 2 \Xi
\Sigma D\right)\Sigma \Xi \thetapop\rangle
}{
p - \kappa
\tr\left[\Sigma^{2}\Xi^{2}\right]
},
\end{aligned}$}
\end{equation}
with $\Xi = (\Sigma + \tau I_p)^{-1}$ %
and $\tau$ is the unique solution of %
\begin{equation}   \label{eq:tau} 
\kappa^{-1} - \frac{\lambda}{\tau} = \frac1p \tr\left[(\Sigma + \tau I_p)^{-1}\Sigma\right].
\end{equation}
\end{theorem}

In words, Theorem \ref{thm:over} shows that the risk $\mathcal{R}(\Sigma, \theta_2, \thetapop)$ is well approximated by the quantity \revised{$\fixedriskeq\left(\Sigma, \thetapop, D, \lambda\right)$} defined in \eqref{eq:defdet}. We highlight that this quantity does not depend on the initialization $\theta_0$: up a fluctuation of order $O(\|D\|_{\mathrm{op}}^2)$, \emph{the risk has reached a fixed point} after two iterations. While the data (and, consequently, $\mathcal{R}(\Sigma, \theta_2, \thetapop)$) are random, $\fixedriskeq\left(\Sigma, \thetapop, D, \lambda\right)$ provides a \emph{deterministic equivalent} that depends only on the population covariance $\Sigma$, the ground-truth vector $\thetapop$, the matrix $D$ capturing the performative effect and the regularization $\lambda$.

We note that the result of Theorem \ref{thm:over} holds for a general matrix $D$ with bounded operator norm---not necessarily a diagonal $D$ as in Assumption \ref{assum:model}.
The assumptions ($\|\thetapop\|_2, \|\theta_0\|_2\le R$, $\kappa, \sigma, \lambda\in (1/M, M)$, $\|\Sigma\|_{\mathrm{op}},\ \|\Sigma^{-1}\|_{\mathrm{op}} \le M$) are all standard in the related literature \citep{han2023distribution,ildizhigh}. We could handle the ridgeless case $\lambda=0$ in a similar way to \cite{han2023distribution,ildizhigh}. However, this requires changing some details and we have opted to avoid the notation clutter, since our focus is on the effect of regularization. 
The assumption on the features $x$ being Gaussian can be also relaxed. In fact, the results of \cite{han2023distribution} (see Theorem 2.4 therein) hold for $\Sigma^{-1/2}x$ having independent, zero mean, unit variance and uniformly subgaussian entries. We prove a formal extension of Theorem \ref{thm:over} to sub-Gaussian data in Appendix \ref{app:extension}.

\begin{figure*}[t]
    \centering
    \begin{subfigure}{0.32\textwidth}
        \includegraphics[width=\linewidth]{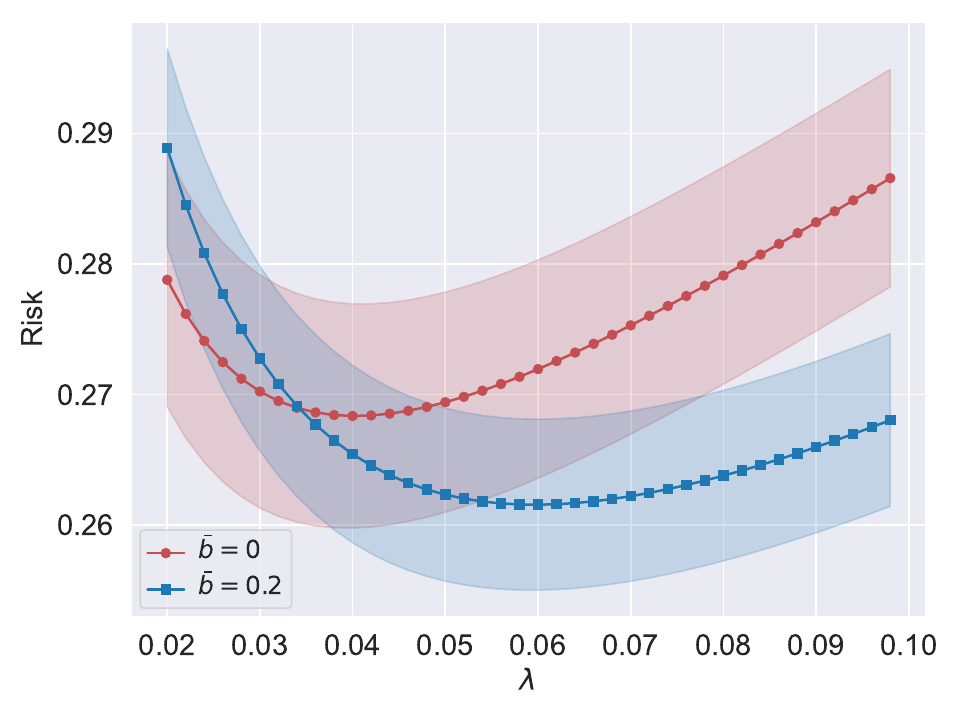}
        \caption{$\kappa=1.1$, $\sigma=0.2$, $\rho=0$}
        \label{fig:propa}
    \end{subfigure}
    \hfill
    \begin{subfigure}{0.32\textwidth}
        \includegraphics[width=\linewidth]{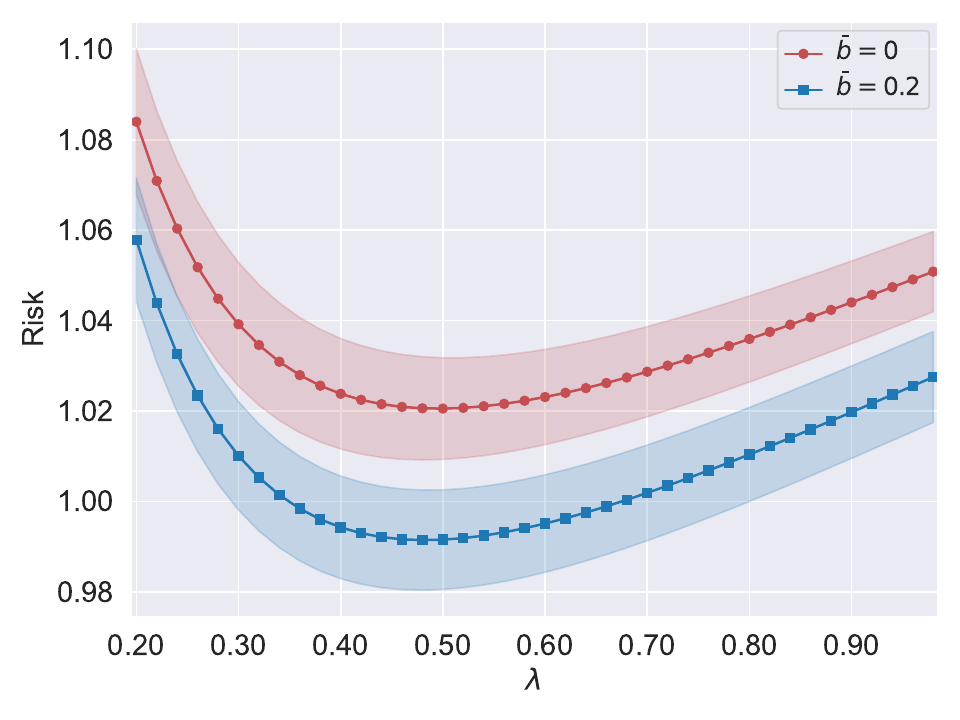}
                \caption{$\kappa=1.1$, $\sigma=0.7$, $\rho=0$}
        \label{fig:propb}
    \end{subfigure}
    \hfill
    \begin{subfigure}{0.32\textwidth}
        \includegraphics[width=\linewidth]{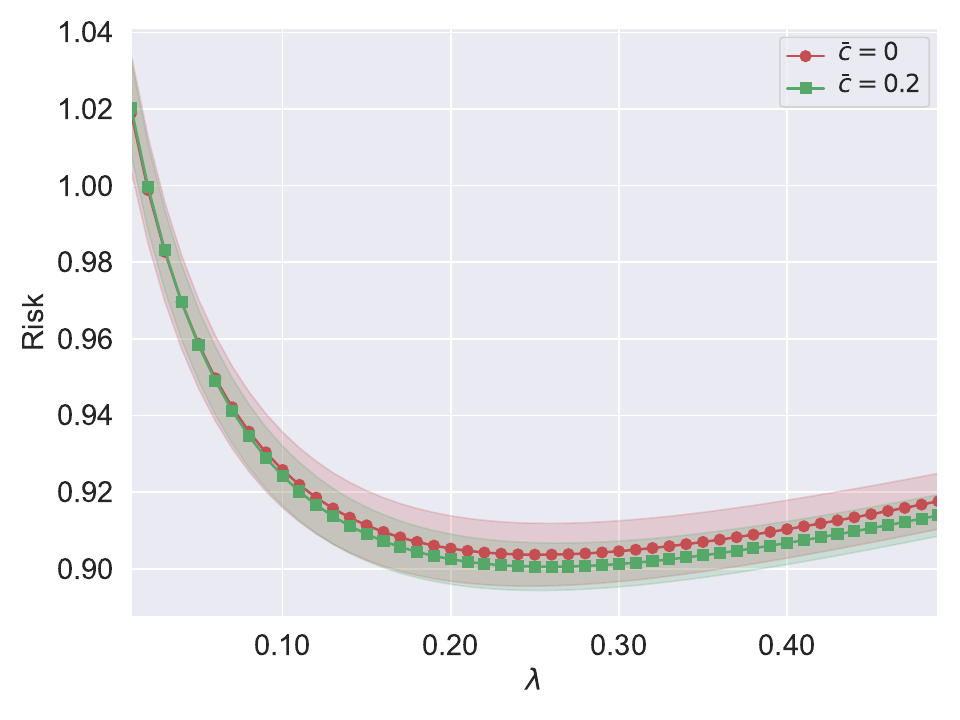}
        \caption{$\kappa=2$, $\sigma=0.5$, $\rho=0.5$}
        \label{fig:propc}
    \end{subfigure}
       \vspace{-.3em} \caption{Excess risk as a function of ridge regularization $\lambda$ with Gaussian data, for $n = 4000$, $\Sigma_1=\Sigma_2=I_d$, $\Sigma_{12}=\rho I_d$, entries of $b$ equal to $\bar b$, and entries of $c$ equal to $\bar c$.  Values are computed from 20 i.i.d.\ trials, with error band at 1 standard deviation. We perform $5$ steps of RRM to approximate the fixed point, as in the simulation setup of Section \ref{sec:num}. (a) In the low-noise regime ($\sigma=0.2$), taking $\bar b=0.2$ instead of $\bar b=0$ \emph{increases the optimal regularization} and \emph{reduces the optimal risk}. We set $\bar c=0$ to emphasize the dependence on $\bar b$. (b) In the large-noise regime ($\sigma=0.7$), taking $\bar b=0.2$ instead of $\bar b=0$ \emph{reduces both optimal regularization and optimal risk}. As in (a), we set $\bar c=0$. (c) Taking $\bar c=0.2$ instead of $\bar c=0$ \emph{reduces the optimal risk}, although the impact of $\bar c$ is less pronounced. We set $\bar b=0$ to emphasize the dependence on $\bar c$.
       }
    \vspace{-1em}
    \label{fig:prop}
\end{figure*}

\paragraph{Optimal regularization and optimally regularized risk.} Leveraging the characterization of Theorem \ref{thm:over}, we optimize the ridge regularization. We focus on the case $\Sigma=\begin{bmatrix}I_d&\rho I_d\\ \rho I_d&I_d\end{bmatrix}$ for small $\rho$ \revised{and require $D$ to be diagonal as in Assumption \ref{assum:model}}. While simplified, this setting captures the performative effect of both predictive and spurious features\revised{, which are mixed via the covariance matrix $\Sigma$}, leading to an interesting phenomenology. Lemma \ref{lemma:explicit} in Appendix \ref{app:pfequiv} computes $\mathbb E_{\thetapop}\fixedriskeq\left(\Sigma, \thetapop, D, \lambda\right)$, as well as the following expansion in $\rho$:
\vspace{-.3em}\begin{equation}\label{eq:expansion}
\begin{split}    
&\mathbb E_{\thetapop}\fixedriskeq\left(\Sigma, \thetapop, D, \lambda\right)=\widetilde{\mathcal R}(D, \lambda, \rho)+ O(\bar b\rho^2+ \rho^4),\\&
 \widetilde{\mathcal R}(D, \lambda, \rho):=\mathcal R_0(\lambda, \rho) + \bar b  A_1(\lambda) %
 + \bar c \rho^2 A_2(\lambda),
\end{split}
\end{equation}
with $\bar b=\tr[\di(b)]/d, \bar c=\tr[\di(c)]/d$. %
Explicit expressions for $\mathcal R_0(\lambda, \rho), A_1(\lambda)$ and $A_2(\lambda)$ are given in \eqref{eq:explexpr} in Appendix \ref{app:pfequiv}. We note that $\mathbb E_{\thetapop}\fixedriskeq\left(\Sigma, \thetapop, D, \lambda\right)$ is even in $\rho$, hence the odd powers of $\rho$ are absent from \eqref{eq:expansion}. Now, we define  optimal regularization and %
risk as 
\vspace{-.3em}\begin{equation}
\begin{split}
    \lambdaeqs(D, \rho):=\arg\min_{\lambda\ge 0}\widetilde{\mathcal R}(D, \lambda, \rho),\\
        \fixedriskeqs(D, \rho):=\min_{\lambda\ge 0}\widetilde{\mathcal R}(D, \lambda, \rho).
\end{split}
\end{equation}
Our goal is to characterize the performative effect on $\lambdaeqs(D, \lambda, \rho), \fixedriskeqs(D, \lambda, \rho)$ and, to do so, we compare these quantities to their values when $D=0$, defined as
\vspace{-.3em}\begin{equation}
    \lambdaeqsz(\rho):=\arg\min_{\lambda\ge 0}\mathcal R_0(\lambda, \rho),
        \fixedriskeqs(\rho)=\min_{\lambda\ge 0}\mathcal R_0(\lambda, \rho).
\end{equation}
This is formalized by the result below whose proof is deferred to Appendix \ref{app:pfequiv}.
\begin{theorem}[Optimal regularization -- over-parameterized]\label{thm:equiv}
In the setting described above, we have 
\vspace{-.3em}\begin{align}
     \begin{aligned}\lambdaeqs(D, \rho)&{}= \lambdaeqsz(\rho)+\bar b (B_1(\sigma, \kappa)+O(\rho^2)) \\&+\bar c \rho^2( C_1(\sigma, \kappa)+O(\rho^2))+O(\bar b^2+\bar c^2)
     ,\label{eq:thmequivl}\end{aligned}\\
    \begin{aligned}\fixedriskeqs(D, \rho)&{}= \fixedriskeqs(\rho)+\bar b (B_2(\sigma, \kappa)+O(\rho^2))\\&+\bar c \rho^2( C_2(\sigma, \kappa)+O(\rho^2))+O(\bar b^2+\bar c^2),
    \end{aligned}\label{eq:thmequivR}
\end{align}
where the functions $B_1(\sigma, \kappa), B_2(\sigma, \kappa), C_1(\sigma, \kappa), C_2(\sigma, \kappa)$ depend only on $\sigma, \kappa$ and they are explicitly given in  \eqref{eq:lfor}-\eqref{eq:Rfor}. Furthermore, these functions satisfy
\vspace{-.3em}\begin{align}
    B_1(\sigma, \kappa)&\ge 0 \quad \text{for } 0\le \sigma\le \sigma_{B_1}(\kappa), \,\,\kappa>1,\label{eq:relations1a}\\ B_1(\sigma, \kappa)&\le 0 \quad \text{for }  \sigma> \sigma_{B_1}(\kappa), \,\,\kappa>1,\label{eq:relations1b}\\
C_1(\kappa,\sigma)&\le 0 \quad \text{for } \sigma\ge 0, \kappa\ge 2,\label{eq:relations2}\\
B_2(\kappa,\sigma)&\le 0 \quad \text{for } \sigma\ge 0, \kappa> 1,\label{eq:relations3}\\
C_2(\kappa,\sigma)&\le 0 \quad \text{for } \sigma\ge 0, \kappa> 1,\label{eq:relations4}
\end{align}
with $\sigma^2_{B_1}(\kappa)=1/2-7\kappa^{-1}/18 +O(\kappa^{-2})$.
\end{theorem}

In words, \eqref{eq:thmequivl} gives a quantitative comparison between optimal regularization with performative effect ($\lambdaeqs(D, \rho)$) and without it ($\lambdaeqsz(\rho)$). Similarly, \eqref{eq:thmequivR} compares optimally-regularized risks $\fixedriskeqs(D, \rho)$ and $\fixedriskeqs(\rho)$ respectively with and without performativity. The study of the signs of the  auxiliary functions $B_1(\sigma, \kappa), B_2(\sigma, \kappa), C_1(\sigma, \kappa), C_2(\sigma, \kappa)$ leads to the considerations below:
\begin{itemize}[leftmargin=1em]
    \item \Cref{eq:relations1a,eq:relations1b} imply that \emph{(i)} if the noise variance $\sigma^2$ is small, then the optimal regularization moves in the same direction as the performative effect on the predictive features; and \emph{(ii)}  if the noise variance is large, the effect is reversed and the optimal regularization moves in the opposite direction to the performative effect. From a Bayesian perspective, informally, this acts as the noise level controlling the model's confidence: if the noise variance increases, the model moves back towards its "prior", and thus the shift in the regularization due to performativity decreases. This is illustrated in Figures \ref{fig:propa} and \ref{fig:propb}.      
    \item \Cref{eq:relations2} implies that, when $\kappa\ge 2$, the optimal regularization moves in the opposite direction to the performative effect on the spurious features. This effect is however significantly attenuated by the factor $\rho^2$ multiplying $\bar c$ in \eqref{eq:thmequivl} and, as such, it is hardly noticeable both with Gaussian data (as considered in this section) and in real-world settings (as considered in Section \ref{sec:num}). %
    \item \Cref{eq:relations3} implies that, when performativity reinforces existing trends ($\bar b>0$), the optimally-regularized risk improves in the presence of a performative effect on the predictive features. This occurs regardless of the size of the noise variance, and it is illustrated in Figures \ref{fig:propa} and \ref{fig:propb}. Instead, when performativity dampens existing trends ($\bar b<0$), the effect is reversed and the optimal risk worsens.
    \item Finally, \Cref{eq:relations4} implies that the dependence of the optimally-regularized risk on the performative effect on the spurious features is analogous: the optimal risk decreases when $\bar c>0$, and increases when $\bar c<0$. However, as for $\lambdaeqs(D, \rho)$, the impact of performativity on $\fixedriskeqs(D, \rho)$ is less pronounced for the spurious features, due to the factor $\rho^2$ multiplying $\bar c$ in \eqref{eq:thmequivR}. This is illustrated in Figure \ref{fig:propc}.
\end{itemize}

\paragraph{\revised{Evaluating the model on the shifted distribution.}} By following similar steps, we can also provide an analysis in the over-parameterized setting of the model tested on the shifted distribution (i.e., on the same distribution used to train it). We defer the details to Appendix \ref{app:test-bis} and summarize the main results below. Theorem \ref{thm:over2} is the equivalent of Theorem \ref{thm:over}, and it provides a deterministic equivalent of the performative fixed point. Lemma \ref{lemma:explicit-bis} specializes the risk expression from Theorem \ref{thm:over2} to a covariance of the form $\Sigma=\begin{bmatrix}I_d&\rho I_d\\ \rho I_d&I_d\end{bmatrix}$, and it is the equivalent of Lemma \ref{lemma:explicit} (which gives \eqref{eq:expansion}). Next, we provide expressions for the optimal $\tau$ in Lemma \ref{lemma:taustar-bis}, for the optimal regularization parameter in \eqref{eq:lambdastar-bis}-\eqref{eq:lfor-bis}, and for the optimal risk in Corollary \ref{cor:risk-bis}. Lemma \ref{lemma:B1-bis} shows that the optimal regularization moves in the same direction as the performative effect on the predictive features, as it is the case for testing over $\mathcal D(\theta=0)$ provided that the noise variance is  enough. Finally, Lemmas \ref{lemma:B2-bis} and \ref{lemma:C2-bis} show that the optimally regularized risk worsens in the presence of a performative effect (either on the predictive or on the spurious features) that reinforces existing trends ($\bar b, \bar c>0$). This is in contrast with the behavior of the optimally regularized risk tested over $\mathcal D(\theta=0)$, which instead decreases when either $\bar b>0$ or $\bar c>0$. 

\vspace{-.5em}

\section{Numerical experiments}\label{sec:num}

\vspace{-.5em}

\begin{figure*}[h!tb]
    \centering
    \begin{subfigure}{0.32\textwidth}
        \includegraphics[width=\linewidth]{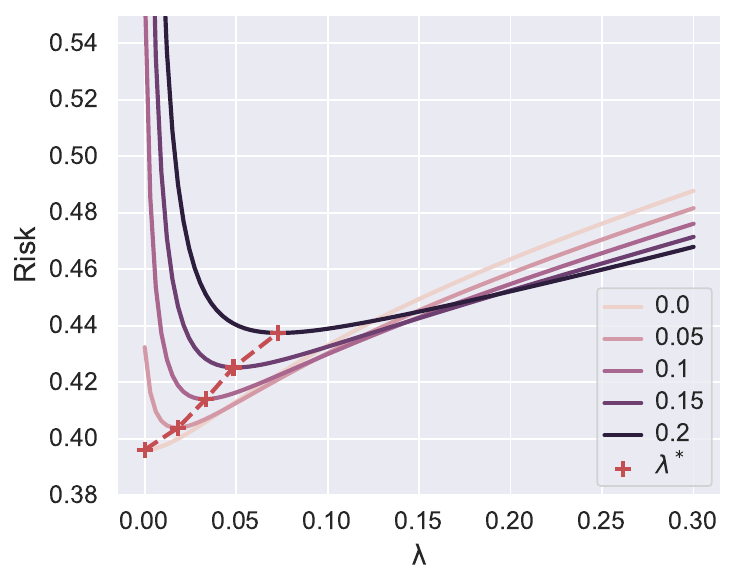}
        \caption{Housing ($n=4000$)}
        \label{fig:housing}
    \end{subfigure}
    \hfill
    \begin{subfigure}{0.32\textwidth}
        \includegraphics[width=\linewidth]{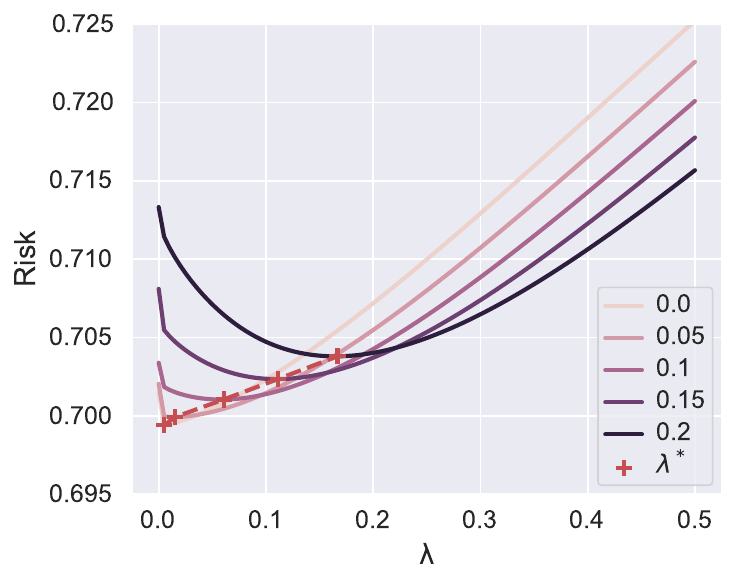}
        \caption{LSAC ($n=4000$)}
        \label{fig:LSACmany}
    \end{subfigure}
    \hfill
    \begin{subfigure}{0.32\textwidth}
        \includegraphics[width=\linewidth]{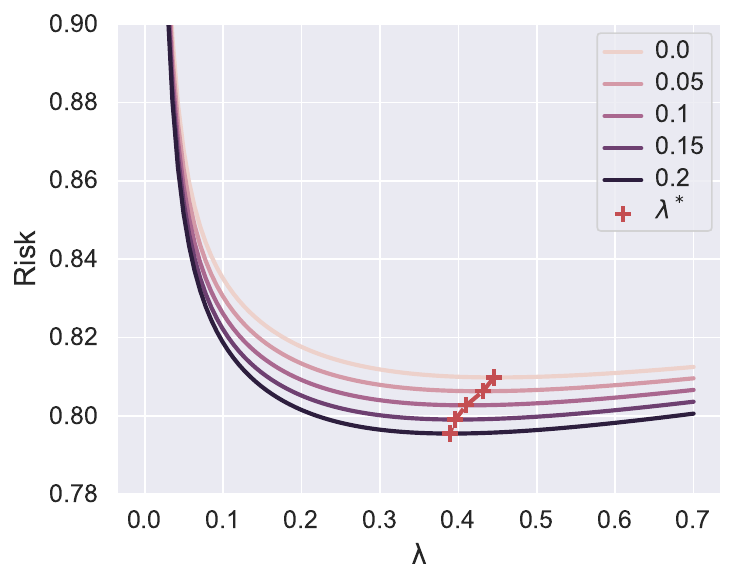}
        \caption{LSAC ($n=100$)}
        \label{fig:LSACfew}
    \end{subfigure}
       \vspace{-.3em} \caption{Excess risk as a function of ridge regularization $\lambda$ in real-world datasets (Housing, LSAC). Different curves (in different colors) correspond to different values of $\bar{b} \in \{0, 0.05, 0.1, 0.15, 0.2\}$, and we connect with a red dashed line the optima of the risk for various choices of $\bar b$. %
       The plots in (a)-(b) use $n=4000$ data points at each training step, which corresponds to the population setting ($n \gg d$); the plot in (c) uses $n=100$, a value closer to the number of features $d=22$.}
    \vspace{-1em}
    \label{fig:real}
\end{figure*}

In this section, we test the effect of regularization and performative shifts on real data. Since no dataset currently provides a real performative shift, the shift must be encoded synthetically. In practice, we take a real-world dataset, randomly split the samples across time steps, train a model on one split, compute the parameter $\theta$, and then shift the samples of the next split according to the theoretical model. This methodology follows previous work \citep{perdomo_performative, hardt2023performative, zezulka_performativity_2023}. These experiments allow us to test whether the theory remains predictive when \emph{(i)} the data is not i.i.d.\ random and $\thetapop$ is fixed by the task, when \emph{(ii)} the true relationship between the feature and the target is likely non-linear, and when \emph{(iii)} the regularization is not the ridge penalty. The code to reproduce the experiments is available at \url{https://github.com/totilas/regularization-vs-perf}.

We consider two datasets. First, we use the Housing dataset,\footnote{\url{https://www.openml.org/d/823}} where the goal is to predict house prices from housing features and local demographics. We follow the methodology of \citet{NEURIPS2024_7de66547} to choose performative features. The dataset has $8$ features and $20{,}640$ datapoints, which we split into five folds: four for training and one for test. Four training steps suffice experimentally to reach the fixed point, which is consistent with the theory, where the first-order effect stabilizes after only two iterations.  
Second, we use the Law School Admission Council (LSAC) dataset,\footnote{\url{https://storage.googleapis.com/lawschool_dataset/bar_pass_prediction.csv}} where the default task is to predict bar passage from demographic features and previous grades. We change the target to GPA to maintain a regression task, and randomly choose features affected by performativity. After dropping redundant columns or those too correlated with GPA, the dataset has $22$ features and $20{,}427$ samples, which again we split in five folds. We report detailed pre-processing, parameters and data covariance in Appendix \ref{app:real}.

When $n=4000$, both for the Housing (Figure \ref{fig:housing}) and the LSAC (Figure \ref{fig:LSACmany}) dataset, we note that \emph{(i)} the optimal regularizer increases proportionally to $\bar{b}$, and \emph{(ii)} the optimally-regularized risk becomes worse as $\bar{b}$ grows. This can be attributed to the fact that $n\gg d$, and it is consistent with our theoretical results in the population setting (Corollary \ref{cor:pop}). In contrast, when training with very few samples on LSAC (\Cref{fig:LSACfew}), the behavior of the regularized risk follows the predictions of Theorem \ref{thm:over} for the proportional setting in the large-noise regime: as $\bar{b}$ grows, the optimal regularizer gets smaller and the risk improves. 
Note that, even if Figures \ref{fig:LSACmany} and \ref{fig:LSACfew} consider the same dataset, the ranges of excess risk and regularizer are not the same due to the different sample sizes. We did not find numerical evidence for the role of $c$, suggesting that its effect may be dominated by data noise, consistent with our theoretical findings on the limited impact of spurious features.

In Appendix \ref{sec:fifig} we showcase the results of the same experiments when the ridge regularization is replaced with \emph{(i)} dropout (\Cref{fig:dropout}), \emph{(ii)} Lasso regularization (\Cref{fig:lasso}), and \emph{(iii)} elastic net (\Cref{fig:elasticnet}). This demonstrates that the relationship between regularization and performative strength persists beyond ridge regression.

To further assess whether our findings extend beyond linear models, we also run an experiment with neural networks. We follow the strategic classification setting of~\citet{pmlr-v206-mofakhami23a}, using the \texttt{GiveMeSomeCredit} dataset through the \texttt{whynot} package. In this setting, the strength of the performative effect is controlled by a parameter $\delta$: when a data point would receive a negative classification under the previous model, it may strategically modify its manipulable features by copying features from another data point, with probability depending on $\delta$. We build on the publicly available code of~\citet{pmlr-v206-mofakhami23a}, use the largest neural network reported there, and keep the preprocessing, learning rate, and other hyperparameters unchanged. For each value of $\delta$, we run the dynamics until convergence for several values of the $\ell_2$ regularization parameter $\lambda$, sharing randomness across values of $\delta$ and $\lambda$. Each setting is run several times, and we report the final test accuracy with standard deviations. As shown in \Cref{fig:nn-reg}, the same qualitative behavior appears: $\ell_2$ regularization mitigates the sharp drop in accuracy induced by the performative shift, and the optimal amount of regularization increases with the strength of the performative effect.

\begin{figure}[thb]
    \centering
    \includegraphics[width=\columnwidth]{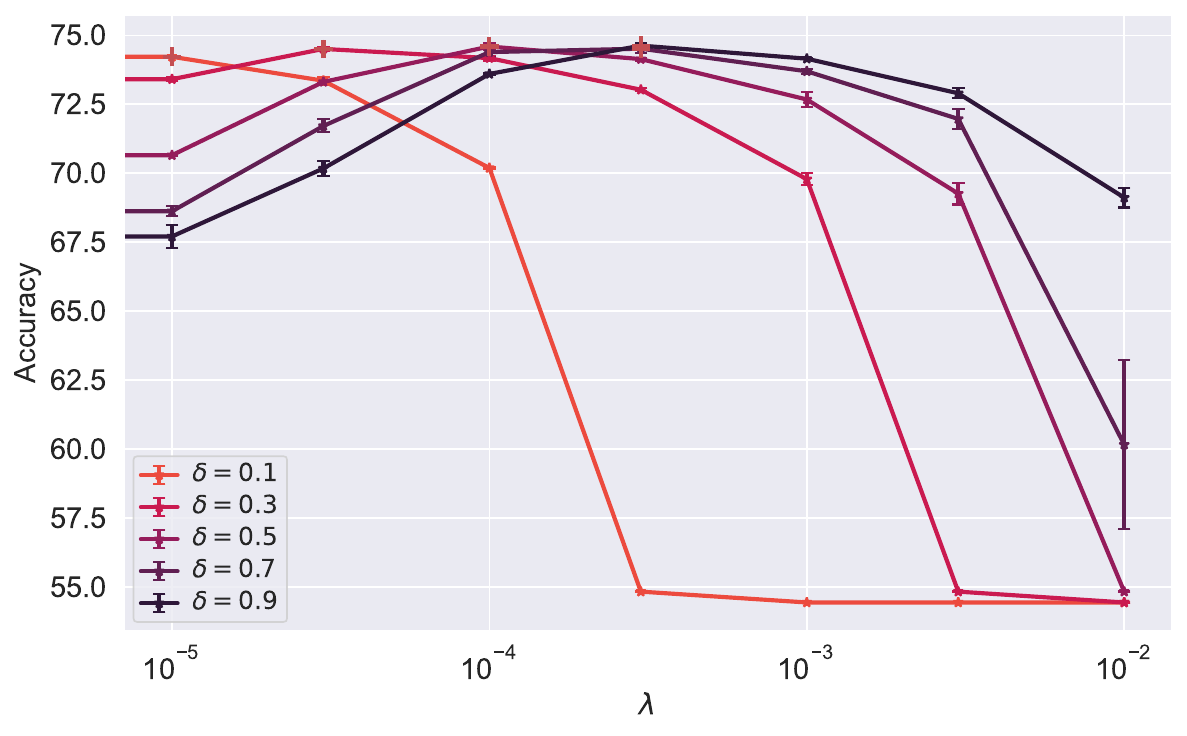}
    \caption{Final test accuracy of a neural network under performative shifts of increasing strength $\delta$, as a function of the $\ell_2$ regularization parameter $\lambda$. Regularization mitigates the loss in accuracy caused by the performative shift, and the best regularization level increases with $\delta$.}
    \label{fig:nn-reg}
\end{figure}

\vspace{-.5em}
\section{Conclusions}
\vspace{-.5em}

\looseness-1In this work, we demonstrate that regularization and performative effect are strongly related, as one can partially cancel out the other. In the population regime, the excess risk is worsened by performative effects. However, optimal ridge regularization mitigates this issue, especially when the data is isotropic and the entries of the vector modeling performativity have little variability. %
In the proportional regime, we provide a deterministic equivalent of the performative fixed point for random data. This in turn unveils a remarkable phenomenology: in contrast with the population setting, the optimal risk improves when performativity reinforces existing trends; furthermore, the optimal regularization follows the direction of the performative effect on the predictive features when the noise is small, while it goes in the opposite direction when the noise is large. Although the theoretical results focus on random data and a linear target model, our experiments on real-world data follow the theoretical predictions, suggesting their generality. Overall, these findings indicate that regularization could help in a wider range of scenarios, which we leave for future work. Beyond studying more complex data or models, interesting directions include the impact of other forms of regularization, such as early stopping or pruning, to mitigate performative effects.

%% file: app.tex
\section{Additional figures}
\label{sec:fifig}

\begin{figure}[htb]
    \centering
    \begin{subfigure}{0.32\textwidth}
        \includegraphics[width=\linewidth]{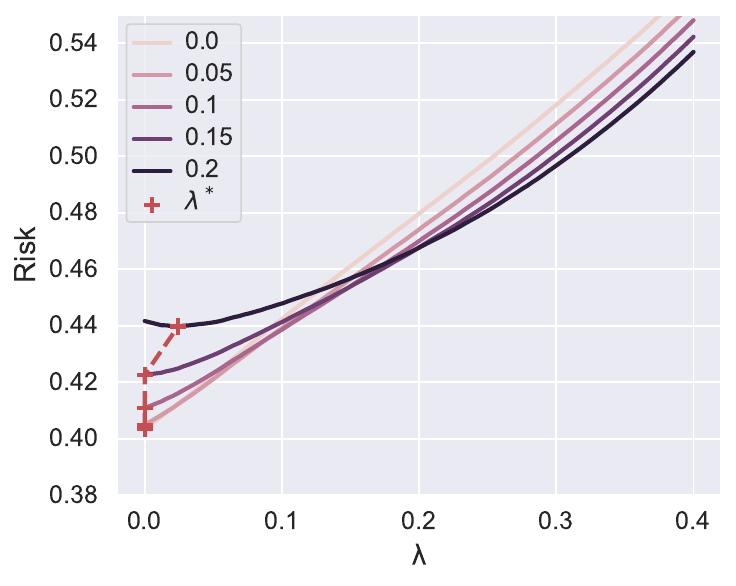}
        \caption{Housing ($n=4000$)}
        \label{fig:housingdropout}
    \end{subfigure}
    \hfill
    \begin{subfigure}{0.32\textwidth}
        \includegraphics[width=\linewidth]{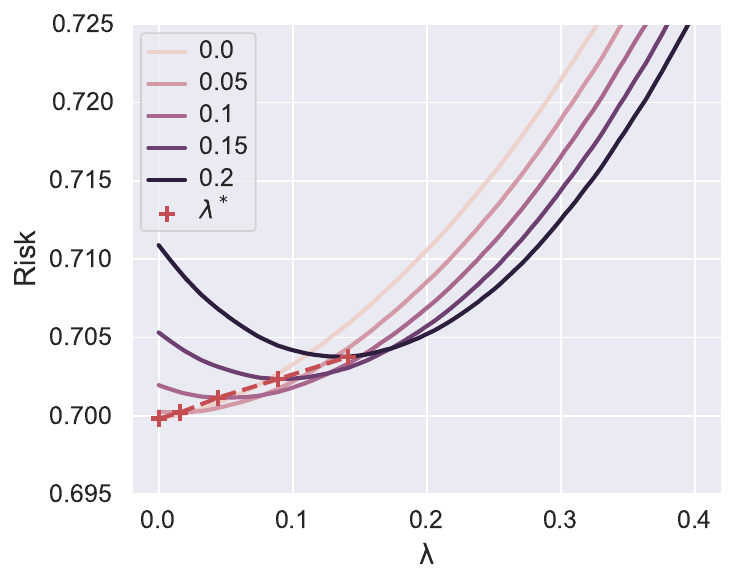}
        \caption{LSAC ($n=4000$)}
        \label{fig:LSACmanydropout}
    \end{subfigure}
    \hfill
    \begin{subfigure}{0.32\textwidth}
        \includegraphics[width=\linewidth]{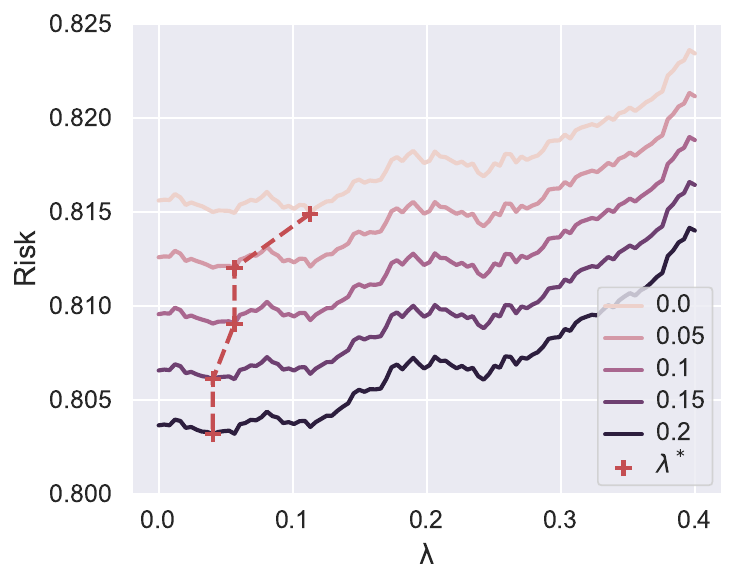}
        \caption{LSAC ($n=100$)}
        \label{fig:LSACfewdropout}
    \end{subfigure}
     \caption{Same experiments as in \Cref{fig:real}, using dropout regularization instead of ridge. Although dropout induces a more complex form of regularization (even for linear regression~\cite{dropoutversusl2}), the findings are similar. For the Housing dataset in the plot (a), dropout provides benefits only for large performative effects, which can be explained by the small dimension $d=8$. For LSAC in the plots (b)-(c), the observed effects are the same as for ridge regularization. Finally, for LSAC in the plot (c), being in the proportional setting with large noise, the optimal dropout rate increases with the strength of the performative effect and the optimal risk get smaller, as it was the case for ridge regularization.
     }
    \label{fig:dropout}
\end{figure}

\begin{figure}[htb]
    \centering
    \begin{subfigure}{0.32\textwidth}
        \includegraphics[width=\linewidth]{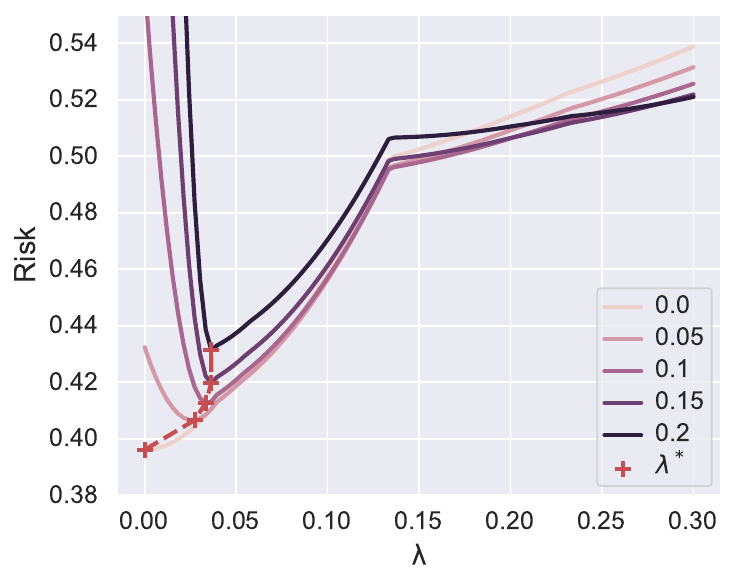}
        \caption{Housing ($n=4000$)}
        \label{fig:housinglasso}
    \end{subfigure}
    \hfill
    \begin{subfigure}{0.32\textwidth}
        \includegraphics[width=\linewidth]{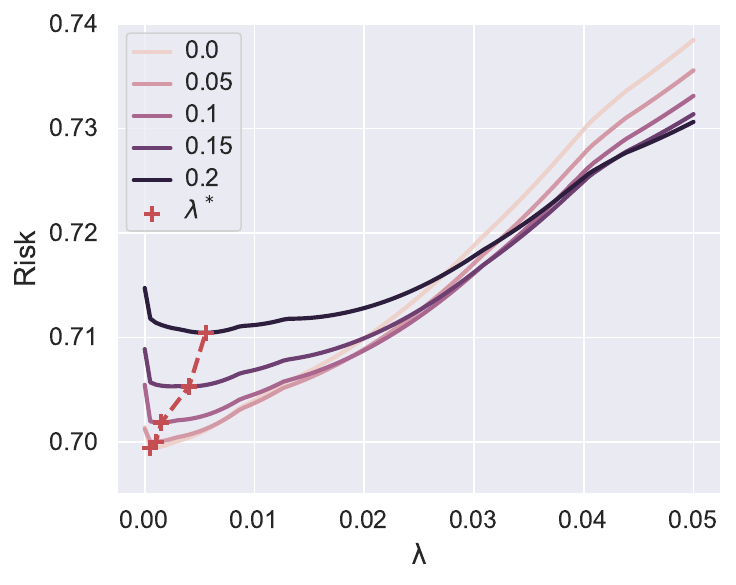}
        \caption{LSAC ($n=4000$)}
        \label{fig:LSACmanylasso}
    \end{subfigure}
    \hfill
    \begin{subfigure}{0.32\textwidth}
        \includegraphics[width=\linewidth]{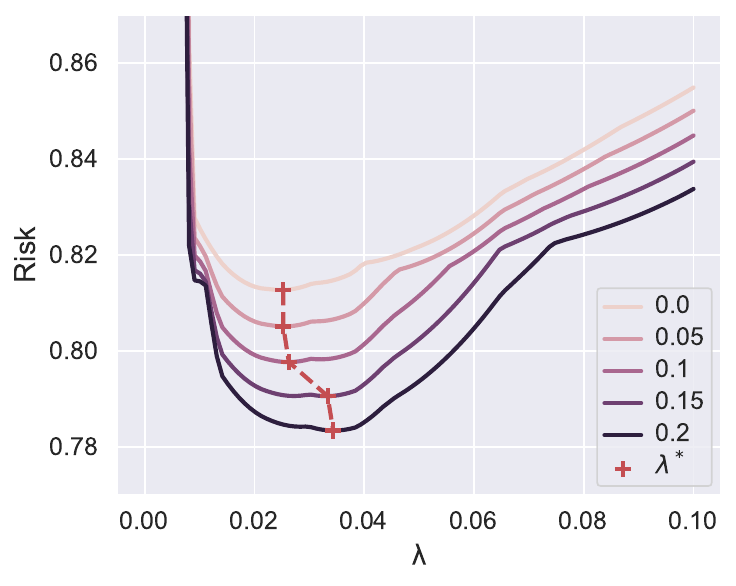}
        \caption{LSAC ($n=100$)}
        \label{fig:LSACfewlasso}
    \end{subfigure}
       \caption{Same experiments as in \Cref{fig:real}, using Lasso regularization instead of ridge. Similar conclusions hold: the performative effect worsens performance in the population regime and helps in the proportional regime; the optimal regularizer continues to be non-decreasing with the strength of the performative effect in the population regime. The optimal regularizer seems close to constant in the proportional regime, potentially due to the number of features being too small to observe a dependency on the Lasso regularization. Indeed, with $22$ features, the support of $\theta$ cannot change smoothly as the regularization increases.}
    \label{fig:lasso}
\end{figure}
\begin{figure}[htb]
    \centering
    \begin{subfigure}{0.32\textwidth}
        \includegraphics[width=\linewidth]{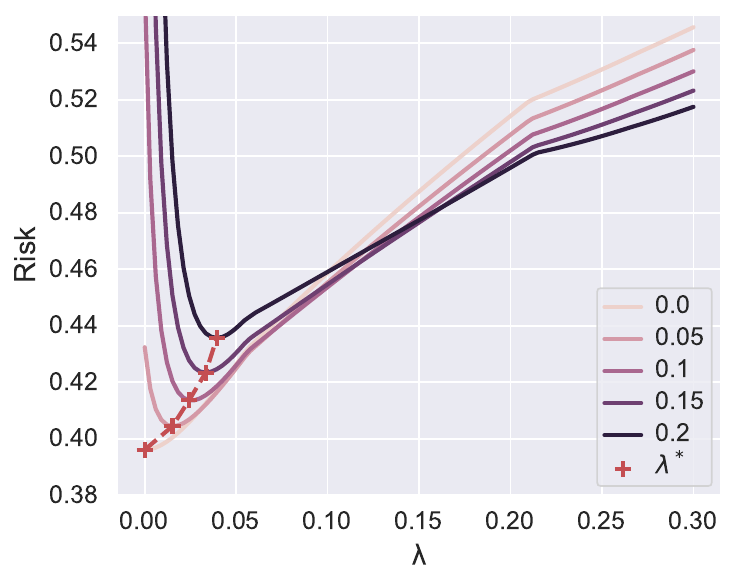}
        \caption{Housing ($n=4000$)}
        \label{fig:housingelasticnet}
    \end{subfigure}
    \hfill
    \begin{subfigure}{0.32\textwidth}
        \includegraphics[width=\linewidth]{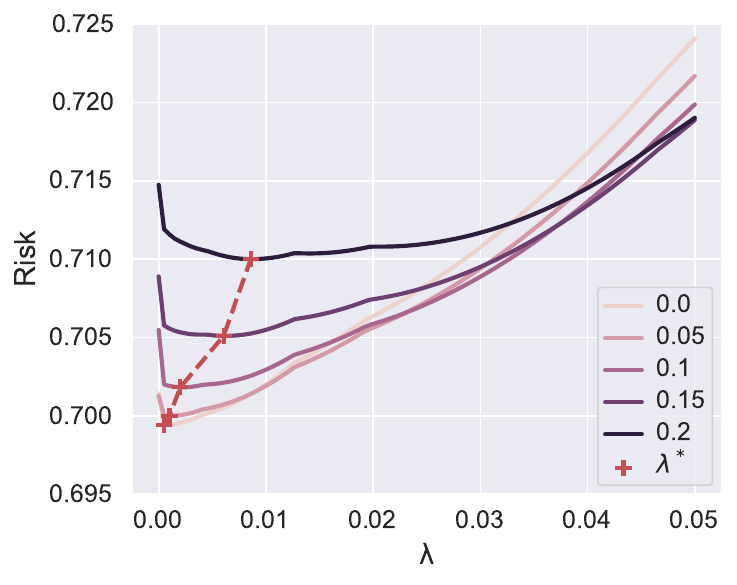}
        \caption{LSAC ($n=4000$)}
        \label{fig:LSACmanyelasticnet}
    \end{subfigure}
    \hfill
    \begin{subfigure}{0.32\textwidth}
        \includegraphics[width=\linewidth]{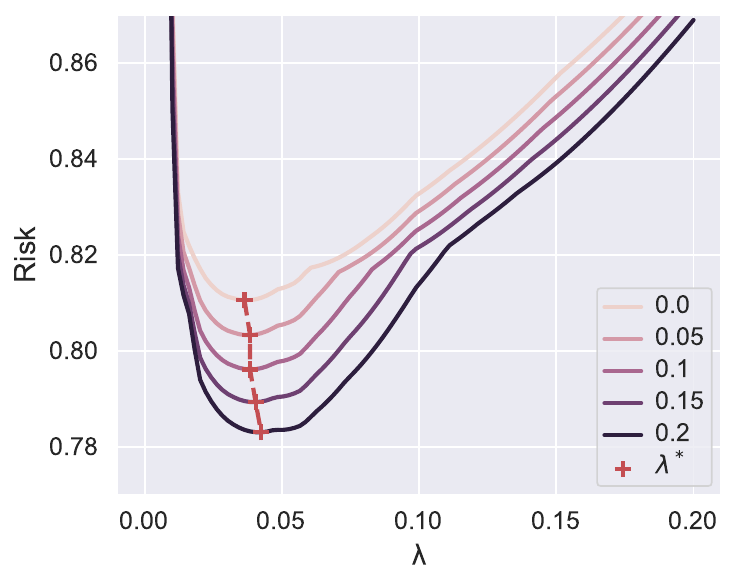}
        \caption{LSAC ($n=100$)}
        \label{fig:LSACfewelasticnet}
    \end{subfigure}
 \caption{Same experiments as in \Cref{fig:real}, using an elastic net regularization with an equal ratio between $\ell_1$ and $\ell_2$ penalties instead of ridge. As intuition suggests, the results lie between those obtained with ridge and Lasso regularization.}
    \label{fig:elasticnet}
\end{figure}

\newpage

\section{Additional proofs for \Cref{sec:pop}}
\label{app:pop}

This appendix contains the missing proofs for \Cref{sec:pop}. %
We start with the convergence at exponential rate to the fixed point $\theta^{\infty}$ in (\ref{eq:fppop}). Then, we prove Theorem \ref{thm:pop} giving the first-order approximation of the risk, as well as the expression in (\ref{eq:pophigh}) giving the higher-order approximation of the risk. Finally, we prove the upper bound on $\|F\|_{\mathrm{op}}$ in (\ref{eq:opnF}).

\begin{lemma}\label{lemma:cr}
    The sequence $(\theta_k)_k$ converges to the fixed point \[\theta^{\infty} = (I_p +\lambda \Sigma^{-1} - D)^{-1}\thetapop.\]
  Moreover, for any \(\varepsilon \in (0, 1)\), if we start at \(\theta_0 = 0\), after at most 
     \[k_{\varepsilon} = \left\lceil \frac{\ln \left(1/\varepsilon\right)}{\ln\left(1/\left(\frac{\|\Sigma\|_{\mathrm{op}}}{\|\Sigma\|_{\mathrm{op}} + \lambda} \max\left\{\|b \|_\infty, \|c \|_{\infty}\right\}\right)\right)}\right\rceil\] iterations, the relative error $
        \frac{\|\theta^{k_{\varepsilon}} - \theta^{\infty}\|_2}{\|\theta^{\infty}\|_2}$
        is smaller than $\varepsilon$.

\end{lemma}
\begin{proof}
Denoting $T = (\Sigma + \lambda I_p)^{-1} \Sigma D$, the recurrence relation is 
\[\theta^k = T \theta^{k-1} + (\Sigma + \lambda I_p)^{-1} \Sigma \thetapop.  \]
When going to the limit, $\sum_i T^i \rightarrow (I_p - T)^{-1}$. The convergence requires the matrix $T$ to have smaller eigenvalues than one, which is guaranteed by  $\|b\|_{\infty}$ and $\|c\|_{\infty}$ being smaller than one. Thus, we have 
\[\theta^{\infty} = (I_p - T)^{-1}(\Sigma + \lambda I_p)^{-1} \Sigma \thetapop. \]
Noticing that $I_p = (\Sigma + \lambda I_p)^{-1}(\Sigma + \lambda I_p)$ and using the definition of $T$ gives the expression of $\theta_{\infty}$.

Let $e_k=\theta^{k}-\theta^\infty$. Using $\theta^\infty=T\theta^\infty+(\Sigma+\lambda I_p)^{-1}\Sigma\,\thetapop$, we have
\[
e_k
= \theta^{k}-\theta^\infty
= T\theta^{k-1}+(\Sigma+\lambda I_p)^{-1}\Sigma \thetapop - \theta^\infty
= T(\theta^{k-1}-\theta^\infty)
= T e_{k-1}.
\]
Thus,
\[
e_k=T^k e_0 = -T^{k} \theta^\infty \implies \|e_k\|_2 \le \|T\|^k_{\mathrm{op}} \| \theta^\infty \|_2 \implies \frac{\|e_k\|_2}{\|\theta^\infty\|_2} =  \frac{\|\theta^{k}-\theta^\infty\|_2}{\|\theta^\infty\|_2} \le  \|T\|_{\mathrm{op}}^k. 
\]
Consequently, $\|T\|^k\le\varepsilon$ suffices, i.e.,
$k \ge \ln(1/\varepsilon)/\ln\left(1/\|T\|_{\mathrm{op}}\right)$. We finally note that
\begin{equation*}
    \begin{split}
        \|T\|_{\mathrm{op}} &\le \|(\Sigma+\lambda I_p)^{-1}\Sigma\|_{\mathrm{op}}\|D\|_{\mathrm{op}}
= \frac{\|\Sigma\|_{\mathrm{op}}}{\|\Sigma\|_{\mathrm{op}} + \lambda} \max\left\{\|b \|_\infty, \|c \|_{\infty}\right\}.
    \end{split}
\end{equation*}
Combining the last two inequalities gives the wanted convergence rate.
\end{proof}

\begin{proof}[Proof of Theorem \ref{thm:pop} and of the higher-order approximation in (\ref{eq:pophigh})]
    We start by computing the Taylor expansion:
    \begin{align*}
        A = (\Sigma + \lambda I_p - \Sigma D)^{-1} \Sigma - I_p
        = (I_p - (D-\lambda \Sigma^{-1}))^{-1} - I_p 
        = \sum_{i = 1}^{\infty} (D- \lambda\Sigma^{-1})^i
        = \sum_{i = 1}^{\infty} F^i.
    \end{align*}
    Let us define 
    \(A^{(k)} = \sum_{i=1}^{k} (D- \lambda\Sigma^{-1})^i \). For the two first orders, we have:
    \[ A^{(1) \, \top} \Sigma A^{(1)} = (D- \lambda\Sigma^{-1}) \Sigma (D- \lambda\Sigma^{-1}) = D\Sigma D - 2\lambda D + \lambda^2 \Sigma^{-1}.
    \]
    This is independent of $c$ and gives the simple formula
    \[R^{(1)}(\lambda) = \frac{1}{d} \tr(\di(b^2)\Sigma_1) - 2 \lambda \bar b +\frac{1}{d} \lambda^2 \tr(S_1), \]
    where $\bar{b} := \frac{1}{d}\tr[\di(b)] =
\frac{1}{d}\sum_{i=1}^d b_i$, $b^2:=[b_1^2, \ldots, b_d^2]\in\mathbb R^d$ and $S_1 = (\Sigma_1 - \Sigma_{12}\Sigma_2^{-1} \Sigma_{21})^{-1}$ denotes the Schur complement of $\Sigma$.
    We go further in the expansion to recover (\ref{eq:pophigh}):
    \begin{align*}
        A^{(2) \,\top}\Sigma A^{(2)}
        & = 
         A^{(1) \, \top} \Sigma A^{(1)} + A^{(1)\,\top} \Sigma \left(A^{(2)} - A^{(1)} \right) + \left(A^{(2)} - A^{(1)} \right)^\top \Sigma A^{(1)}  \\&\quad + \left(A^{(2)} - A^{(1)} \right)^\top \Sigma \left(A^{(2)} - A^{(1)} \right)
        \\& =  D \Sigma D + D^2 \Sigma D + D \Sigma D^2  - \lambda\left[ D\Sigma D \Sigma^{-1}
        + \Sigma^{-1} D \Sigma D + 2D + 4D^2\right] \\
        &\quad + \lambda^2\left[\Sigma^{-1}
        + 3(\Sigma^{-1} D + D \Sigma^{-1}) \right] -2 \lambda^3 \Sigma^{-2} + O\left(\|F\|_{\mathrm{op}}^4\right).
    \end{align*}
The final formula results from taking the trace of the first block. We write the matrix product block per block to prove that $$\tr\bigl[(D\Sigma D \Sigma^{-1}
        + \Sigma^{-1} D \Sigma D)_1\bigr] = 2\tr\!\bigl[\di(b)\Sigma_1\di(b)S_1\bigr]
         + 2\tr\!\bigl[\di(b)\Sigma_{12}\di(c)S_{21}\bigr],$$
         where $S_{21}^\top= -(\Sigma_1 - \Sigma_{12}\Sigma_2^{-1} \Sigma_{21})^{-1}\Sigma_{12} \Sigma_2^{-1}$. This concludes the proof.
\end{proof}

\begin{lemma}\label{lemma:weyl}
Let $F=D-\lambda\,\Sigma^{-1}$. Then, we have that
\begin{equation}\label{eq:fopn1}
\|F\|_{\mathrm{op}} \le 
\max\left(
\left| \max_{1\le i \le d}\{b_i, c_i\} - \frac{\lambda}{\lambda_{\max}(\Sigma)} \right|,
\left| \min_{1\le i \le d}\{b_i, c_i\} - \frac{\lambda}{\lambda_{\min}(\Sigma)} \right|
\right).
\end{equation}
\end{lemma}
\begin{proof}
By Weyl's inequalities for Hermitian matrices,
\begin{align*}
\lambda_{\max}(F) &\le \lambda_{\max}(D) - \lambda \lambda_{\min}(\Sigma^{-1}) 
= \max_{1\le i \le d}\{b_i, c_i\} - \frac{\lambda}{\lambda_{\max}(\Sigma)},\\
\lambda_{\min}(F) &\ge \lambda_{\min}(D) - \lambda \lambda_{\max}(\Sigma^{-1})
= \min_{1\le i \le d}\{b_i, c_i\} - \frac{\lambda}{\lambda_{\min}(\Sigma)}.
\end{align*}
Therefore, we have
\[
\begin{aligned}
\|F\|_{\mathrm{op}} &= \max\left\{ |\lambda_{\max}(F)|,|\lambda_{\min}(F)| \right\}
\\&\le \max\left(
\left| \max_{1\le i \le d}\{b_i, c_i\} - \frac{\lambda}{\lambda_{\max}(\Sigma)} \right|,
\left| \min_{1\le i \le d}\{b_i, c_i\} - \frac{\lambda}{\lambda_{\min}(\Sigma)} \right|
\right),
\end{aligned}
\]
and the equality happens if and only if \(\Sigma\) and \(D\) are simultaneously diagonalizable.
\end{proof}
Finally, we can rewrite this result as 
\begin{equation}   \label{eq:opnF} 
\|F\|_{\mathrm{op}} \le 
\max\left(
\begin{split}
\left| \max_{1\le i \le d}\{b_i, c_i\} - \frac{\lambda}{\|\Sigma\|_{\mathrm{op}}} \right|,\\
\left| \min_{1\le i \le d}\{b_i, c_i\} - \frac{\lambda}{\lambda_{\min}(\Sigma)} \right|
\end{split}
\right),
\end{equation}
where $\lambda_{\min}(\Sigma)$ is the smallest eigenvalue of $\Sigma$, since $\lambda_{\max}(\Sigma)=\|\Sigma\|_{\mathrm{op}}$, due to $\Sigma$ being a covariance matrix and, hence, positive semidefinite.

\section{Proof of Theorem \ref{thm:over}}\label{app:pf}

\paragraph{Deterministic equivalent for $\mathcal R_{1}(\Sigma, \theta_{1}, \thetapop)$.}

Let $\mathcal R_k(\Sigma, \theta_k, \thetapop)$ be the excess risk of the estimator $\theta_k$ given by \eqref{eq:thetak}, i.e.,
\[
    \mathcal R_k(\Sigma, \theta_k, \thetapop) = \left\|\theta_k - \thetapop \right\|_\Sigma^2. 
\]
Having fixed the initialization $\theta_0$, the only randomness in $\mathcal R_{1}(\Sigma, \theta_{1}, \thetapop)$ comes from $(X^{(0)}, y^{(0)})$. This corresponds to the setting in which one trains from the (deterministic) vector of regression coefficients $\thetapop+D\theta_0$. The following lemma gives %
a deterministic equivalent for $\mathcal R_{1}(\Sigma, \theta_{1}, \thetapop)$, conditional on $\theta_0$.

\begin{lemma}\label{lemma:1step}
\revised{Let Assumption \ref{assum:model} hold.} Let $R>0$ be a constant such that $\thetapop, \theta_0\in B_p(R)$. Assume that $\kappa, 
\sigma, \lambda \in (1/M, M)$ and $\|\Sigma\|_{\mathrm{op}},\ \|\Sigma^{-1}\|_{\mathrm{op}} \le M$ for some constant $M>1$. Then, there exists a constant $C=C\left(M, R\right)$ such that, for any $\delta \in (0, 1/2]$, the following holds %
\begin{equation}\label{eq:det1}    
\sup_{\thetapop, \theta_0\in B_p(R)}
\Pr\left(\left|\mathcal R_{1}(\Sigma, \theta_{1}, \thetapop)-\fixedriskeq^{(1)}\left(\Sigma, \theta_0, \thetapop\right)\right|\ge \delta\right)
\le Cpe^{-p\delta^{4}/C},
\end{equation}
\revised{with probability at least $1-Cpe^{-p\delta^{4}/C}$,} where
\begin{equation} \label{eq:R1eq}   
\begin{aligned}
\fixedriskeq^{(1)}\left(\Sigma, \theta_0, \thetapop\right)
&= \left\|\left( \Sigma + \tau I_p \right)^{-1} \Sigma (\thetapop+D\theta_0) - \thetapop  \right\|_\Sigma^2 \\&\quad + \kappa \tr\left[ \Sigma^2 \left( \Sigma + \tau I_p \right)^{-2} \right]\frac{   \sigma^2 + \tau^2 \left\| \left( \Sigma + \tau I_p \right)^{-1} (\thetapop+D\theta_0) \right\|_\Sigma^2  }{ p - \kappa \tr\left[ \Sigma^2 \left( \Sigma + \tau I_p \right)^{-2} \right] },
\end{aligned}
\end{equation}
and $\tau$ is the unique solution of \eqref{eq:tau}. 
\end{lemma}

\begin{proof}
Note that we are generating labels using $\thetaperfok:=\thetapop+D\theta_0$ as a vector of regression coefficients. Thus, we can apply Theorem 3 by \cite{ildizhigh} (which utilizes the non-asymptotic characterization of the minimum norm interpolator by \cite{han2023distribution}), replacing $\beta^s$ with $\thetaperfok$ in that statement. This gives that \eqref{eq:det1} holds with $\fixedriskeq^{(1)}\left(\Sigma, \theta_0, \thetapop\right)$ replaced by $\tfixedriskeq^{(1)}\left(\Sigma, \thetapop, \thetaperfok\right)$ defined as 
\begin{equation}\label{eq:tildeR1}    
\tfixedriskeq^{(1)}\left(\Sigma, \thetapop, \thetaperfok\right)= \E_{g^{(1)}}\left[\left\|X^{(1)}\left(\Sigma, \thetaperfok, g^{(1)}\right)-\thetapop\right\|_\Sigma^2 \right] ,
\end{equation}
where
\begin{align}
X^{(1)}\left(\Sigma, \thetaperfok, g^{(1)}\right)
&=  (\Sigma + \tau I_p)^{-1} \Sigma \left[\thetaperfok + \frac{\Sigma^{-1/2} \gamma^{(1)}(\thetaperfok)g^{(1)}}{\sqrt{p}}\right] ,\label{eq:X1} \\
\left(\gamma^{(1)}(\thetaperfok)\right)^2
&= \kappa\left(\sigma^2 + \tfixedriskeq^{(1)}\left(\Sigma, \thetaperfok, \thetaperfok\right)\right),\label{eq:gamma1}
\end{align}
$\tau$ is the unique solution of \eqref{eq:tau} and $g^{(1)}\sim \mathcal N(0, I_p)$. By plugging \eqref{eq:X1} into \eqref{eq:tildeR1} and computing the expectation with respect to $g^{(1)}$, we get
\begin{equation}\label{eq:tildeR2}
    \tfixedriskeq^{(1)}\left(\Sigma, \thetapop, \thetaperfok\right)=\left\| \left(\Sigma + \tau I_p \right)^{-1} \Sigma \thetaperfok - \thetapop  \right\|_\Sigma^2 + \frac{\left(\gamma^{(1)}(\thetaperfok)\right)^2}{p} \tr \left[ \Sigma^2 \left(\Sigma + \tau I_p \right)^{-2} \right].
\end{equation}
Next, we solve the fixed point equation in $\gamma^{(1)}(\thetaperfok)$: 
\[
\begin{aligned}
    \left(\gamma^{(1)}(\thetaperfok)\right)^2
&= \kappa\left(\sigma^2 + \tfixedriskeq^{(1)}\left(\Sigma, \thetaperfok, \thetaperfok\right)\right)  \\&= \kappa\left(\sigma^2 + \left\| \left(\left(\Sigma + \tau I_p \right)^{-1} \Sigma - I_p\right) \thetaperfok  \right\|_\Sigma^2 + \frac{\left(\gamma^{(1)}(\thetaperfok)\right)^2}{p} \tr \left[ \Sigma^2 \left(\Sigma + \tau I_p \right)^{-2} \right]\right) \\
&= \kappa\left(\sigma^2 + \tau^2\left\| \left(\Sigma + \tau I_p \right)^{-1} \thetaperfok  \right\|_\Sigma^2 + \frac{\left(\gamma^{(1)}(\thetaperfok)\right)^2}{p} \tr \left[ \Sigma^2 \left(\Sigma + \tau I_p \right)^{-2} \right]\right).
\end{aligned}
\]
The last equality comes from
\[
I_p - \left(\Sigma + \tau I_p \right)^{-1} \Sigma  =  \left(\Sigma + \tau I_p \right)^{-1} \left( \Sigma + \tau I_p \right) - \left(\Sigma + \tau I_p \right)^{-1}\Sigma = \tau \left(\Sigma + \tau I_p \right)^{-1}.
\]
Rearranging gives that
\begin{equation}\label{eq:gamma1ex}
\left(\gamma^{(1)} (\thetaperfok)\right)^2 = \kappa\frac{ \sigma^2 + \tau^2 \left\| \left( \Sigma + \tau I_p \right)^{-1} \thetaperfok  \right\|_\Sigma^2  }{ 1 - \frac{1}{n} \tr\left[ \Sigma^2 \left( \Sigma + \tau I_p \right)^{-2} \right] },
\end{equation}
which plugged into \eqref{eq:tildeR2} gives the desired result.
\end{proof}

We note that the expression in \eqref{eq:R1eq} depends on $\theta_0$ and, in fact, it keeps depending on it even after neglecting terms of order $O(\|D\|_{\mathrm{op}}^2)$.

\paragraph{Deterministic equivalent for $\mathcal R_{2}(\Sigma, \theta_{2}, \thetapop)$.}  Next, by %
iterating twice the strategy of Lemma \ref{lemma:1step}, we derive
a deterministic equivalent for $\mathcal R_{2}(\Sigma, \theta_{2}, \thetapop)$. %

\begin{lemma}\label{lemma:2step}
\revised{Let Assumption \ref{assum:model} hold.} Let $R>0$ be a constant such that $\thetapop, \theta_0\in B_p(R)$. Assume that $\kappa, 
\sigma, \lambda \in (1/M, M)$ and $\|\Sigma\|_{\mathrm{op}},\ \|\Sigma^{-1}\|_{\mathrm{op}} \le M$ for some constant $M>1$. Then, there exists a constant $C=C\left(M, R\right)$ such that, for any $\delta \in (0, 1/2]$, the following holds %
\begin{equation}\label{eq:det2}    
\sup_{\thetapop, \theta_0\in B_p(R)}
\Pr\left(\left|\mathcal R_{2}(\Sigma, \theta_{2}, \thetapop)-\fixedriskeq^{(2)}\left(\Sigma, \theta_0, \thetapop\right)\right|\ge \delta\right)
\le Cpe^{-p\delta^{4}/C},
\end{equation}
\revised{with probability at least $1-Cpe^{-p\delta^{4}/C}$,} where
\begin{equation}  \label{eq:fixedriskeq2}  
\begin{aligned}
&\fixedriskeq^{(2)}\left(\Sigma, \theta_0, \thetapop\right)
= \left\| \left( \Sigma + \tau I_p \right)^{-1} \Sigma D \left( \Sigma + \tau I_p \right)^{-1} \Sigma (\thetapop+D\theta_0) - \tau(\Sigma+\tau I_p)^{-1}\thetapop \right\|_\Sigma^2  \\
&\quad+ \kappa \tr\left[ \Sigma \left( \Sigma + \tau I_p \right)^{-2} D  \Sigma^3 \left( \Sigma + \tau I_p \right)^{-2} D \right]
\frac{ \sigma^2 + \tau^2 \big\| \left( \Sigma + \tau I_p \right)^{-1} (\thetapop+D\theta_0) \big\|_\Sigma^2 }{p - \kappa \tr\left[ \Sigma^2 \left( \Sigma + \tau I_p \right)^{-2} \right] } \\
&\quad+ \kappa \tr\left[ \Sigma^2 \left( \Sigma + \tau I_p \right)^{-2} \right]
\frac{ \sigma^2 + \tau^2 \big\| \left( \Sigma + \tau I_p \right)^{-1} \left(\thetapop + D \left( \Sigma + \tau I_p \right)^{-1} \Sigma (\thetapop+D\theta_0) \right) \big\|_\Sigma^2 }{p - \kappa \tr\left[ \Sigma^2 \left( \Sigma + \tau I_p \right)^{-2} \right] } \\
&\quad+ \kappa^2 \tau^2  \tr\left[ \Sigma^2 \left( \Sigma + \tau I_p \right)^{-2} \right] 
\tr\left[ \Sigma \left( \Sigma + \tau I_p \right)^{-2} D  \Sigma \left( \Sigma + \tau I_p \right)^{-2} D \right]
\\
&\hspace{15em}\cdot\frac{ \sigma^2 + \tau^2 \big\| \left( \Sigma + \tau I_p \right)^{-1} (\thetapop+D\theta_0) \big\|_\Sigma^2 }
{\big( p - \kappa \tr\left[ \Sigma^2 \left( \Sigma + \tau I_p \right)^{-2} \right] \big)^2 },
\end{aligned}
\end{equation}
and $\tau$ is the unique solution of \eqref{eq:tau}. 
\end{lemma}

\begin{proof}
The proof extends the argument of Theorem 2 by \citep{ildizhigh} to the ridge regression case, and it applies the distributional characterization of the minimum norm interpolator by \cite{han2023distribution} twice. First, note that $\|\theta_{1}\|_2$ is bounded by a constant $C_1=C_1(R, M)$ independent of $n, p$, with probability at least $C_2 e^{-p/C_2}$, where $C_2=C_2(R, M)$ is a constant independent of $n, p$. This follows from a direct adaptation of Proposition 11 by \cite{ildizhigh}. Define $R':=\max(C_1, \|\thetapop\|_2)$. Then, upon conditioning on $\theta_{1}$, we can apply Lemma \ref{lemma:1step} (after re-defining $R$ to be $R'$), which gives that, for some constant $C_3=C_3(R, M)$,
\begin{equation}\label{eq:det3}    
\sup_{\thetapop, \theta_{1}\in B_p(R')}
\Pr\left(\left|\mathcal R_{2}(\Sigma, \theta_{2}, \thetapop)-\fixedriskeq^{(1)}\left(\Sigma, \theta_{1}, \thetapop\right)\right|\ge \delta\right)
\le C_3pe^{-p\delta^{4}/C_3},
\end{equation}
where $\fixedriskeq^{(1)}\left(\Sigma, \theta_{1}, \thetapop\right)$ is defined in \eqref{eq:R1eq}
and $\tau$ is the unique solution of \eqref{eq:tau}. 

We now evaluate the first term in the expression for $\fixedriskeq^{(1)}\left(\Sigma, \theta_{1}, \thetapop\right)$:
\begin{equation*}
\begin{split}
    \fixedriskeqa^{(1)}\left(\Sigma, \theta_{1}, \thetapop\right):&=\left\|\left( \Sigma + \tau I_p \right)^{-1} \Sigma (\thetapop+D\theta_{1}) - \thetapop  \right\|_\Sigma^2\\
    &=\left\|\left( \Sigma + \tau I_p \right)^{-1} \Sigma D\theta_{1} - \tau (\Sigma+\tau I_p)^{-1}\thetapop  \right\|_\Sigma^2.
    \end{split}
\end{equation*}
Let $M_1=\Sigma^{1/2}$, $M_2=(\Sigma+\tau I_p)^{-1}\Sigma D$ and $a=\tau(\Sigma+\tau I_p)^{-1}\thetapop$. Then, the function $\theta_{1}\mapsto \fixedriskeqa^{(1)}\left(\Sigma, \theta_{1}, \thetapop\right)$ can be expressed as
$$
f(\theta_{1})=\|M_1(M_2 \theta_{1}-a)\|_2^2,
$$
which has gradient
$$
\nabla f(\theta_{1})=2M_2^\top M_1^\top M_1(M_2 \theta_{1}-a).
$$
As $\|\theta_{1}\|_2\le C_1$, $f$ is Lipschitz and its Lipschitz constant is $2\|M_1\|_{\mathrm{op}}^2\|M_2\|_{\mathrm{op}}(\|M_1\|_{\mathrm{op}} C_1+\|a\|_2)$. As $\|M_1\|_{\mathrm{op}}, \|M_2\|_{\mathrm{op}}, C_1, \|a\|_2$ are all upper bounded by constants dependent only on $R, M$, the Lipschitz constant of $f$ is also upper bounded by a constant dependent only on $R, M$. Thus, 
an application of the distributional characterization by \cite{han2023distribution} (restated as Theorem 4 in \citep{ildizhigh}) gives that, for some constant $C_4=C_4(R, M)$,
\begin{equation}\label{eq:det4}    
\sup_{\thetapop, \theta_{1}\in B_p(R')}
\Pr\left(\left|\fixedriskeqa^{(1)}\left(\Sigma, \theta_{1}, \thetapop\right)-\tfixedriskeqa^{(1)}\left(\Sigma, \thetapop, \thetaperfok\right)\right|\ge \delta\right)
\le C_4pe^{-p\delta^{4}/C_4},
\end{equation}
where
\begin{equation}\label{eq:tfixedriskeqa}
    \tfixedriskeqa^{(1)}\left(\Sigma, \thetapop, \thetaperfok\right)= \E_{g^{(1)}}\left[\left\|(\Sigma+\tau I_p)^{-1}\Sigma D X^{(1)}\left(\Sigma, \thetaperfok, g^{(1)}\right)-\tau(\Sigma+\tau I_p)^{-1}\thetapop\right\|_\Sigma^2 \right].
\end{equation}
We recall from Lemma \ref{lemma:1step} that $\thetaperfok=\thetapop+D\theta_0$, $g^{(1)}\sim \mathcal N(0, I_p)$ and $X^{(1)}\left(\Sigma, \thetaperfok, g^{(1)}\right)$ is given by \eqref{eq:X1}. By plugging \eqref{eq:X1} into the RHS of \eqref{eq:tfixedriskeqa} and computing the expectation with respect to $g^{(1)}$, we have 
\begin{equation}\label{eq:tfixedriskeqa2}
    \begin{split}
\tfixedriskeqa^{(1)}\left(\Sigma, \thetapop, \thetaperfok\right)&= \E_{g^{(1)}}\Biggl[\biggl\|(\Sigma+\tau I_p)^{-1}\Sigma D(\Sigma+\tau I_p)^{-1}\Sigma\thetaperfok-\tau(\Sigma+\tau I_p)^{-1}\thetapop\\
&\hspace{5em}+(\Sigma+\tau I_p)^{-1}\Sigma D(\Sigma+\tau I_p)^{-1}\Sigma^{1/2}\frac{ \gamma^{(1)}(\thetaperfok)g^{(1)}}{\sqrt{p}}\biggr\|_\Sigma^2 \Biggr]\\
&=\biggl\|(\Sigma+\tau I_p)^{-1}\Sigma D(\Sigma+\tau I_p)^{-1}\Sigma\thetaperfok-\tau(\Sigma+\tau I_p)^{-1}\thetapop\biggr\|_\Sigma^2 \\
&\hspace{5em}+\frac{\left(\gamma^{(1)}(\thetaperfok)\right)^2}{p}\tr\left[ \Sigma \left( \Sigma + \tau I_p \right)^{-2} D  \Sigma^3 \left( \Sigma + \tau I_p \right)^{-2} D \right],
    \end{split}
\end{equation}
where in the last step we have used the circulant property of the trace. By using the expression for $\gamma^{(1)}(\thetaperfok)$ in \eqref{eq:gamma1ex} and recalling that $\thetaperfok=\thetapop+D\theta_0$, one readily obtains that the RHS of \eqref{eq:tfixedriskeqa2} coincides with the first two lines of the RHS of \eqref{eq:fixedriskeq2}.

Finally, we evaluate the second term in the expression for $\fixedriskeq^{(1)}\left(\Sigma, \theta_{1}, \thetapop\right)$:
\begin{equation*}
\begin{split}
    \fixedriskeqb^{(1)}\left(\Sigma, \theta_{1}, \thetapop\right):&=\kappa \tr\left[ \Sigma^2 \left( \Sigma + \tau I_p \right)^{-2} \right]\frac{   \sigma^2 + \tau^2 \left\| \left( \Sigma + \tau I_p \right)^{-1} (\thetapop+D\theta_{1}) \right\|_\Sigma^2  }{ p - \kappa \tr\left[ \Sigma^2 \left( \Sigma + \tau I_p \right)^{-2} \right] }.
    \end{split}
\end{equation*}
Let $M_1=\Sigma^{1/2}$, $M_2=(\Sigma+\tau I_p)^{-1}D$ and $a=(\Sigma+\tau I_p)^{-1}\thetapop$. Then, the function $\theta_{1}\mapsto \left\| \left( \Sigma + \tau I_p \right)^{-1} (\thetapop+D\theta_{1}) \right\|_\Sigma^2$ can be expressed as
$$
f(\theta_{1})=\|M_1(M_2 \theta_{1}+a)\|_2^2,
$$
which has gradient
$$
\nabla f(\theta_{1})=2M_2^\top M_1^\top M_1(M_2 \theta_{1}+a).
$$
As $\|\theta_{1}\|_2\le C_1$, $f$ is Lipschitz and its Lipschitz constant is $2\|M_1\|_{\mathrm{op}}^2\|M_2\|_{\mathrm{op}}(\|M_1\|_{\mathrm{op}} C_1+\|a\|_2)$. As $\|M_1\|_{\mathrm{op}}, \|M_2\|_{\mathrm{op}}, C_1, \|a\|_2$ are all upper bounded by constants dependent only on $R, M$, the Lipschitz constant of $f$ is also upper bounded by a constant dependent only on $R, M$. Note that the quantity $|p-\kappa \tr\left[\Sigma^2(\Sigma+\tau I_p)^{-2}\right]|$ is lower bounded by a constant dependent only on $R, M$, as a consequence of Proposition 2.1 in \cite{han2023distribution}. Thus, we have that the function $\theta_{1}\mapsto\fixedriskeqb^{(1)}\left(\Sigma, \theta_{1}, \thetapop\right)$ is Lipschitz and its Lipschitz constant is $C_5=C_5(R, M)$. Hence,
another application of the distributional characterization by \cite{han2023distribution} (cf.\ Theorem 4 in \citep{ildizhigh}) gives that, for some constant $C_6=C_6(R, M)$,
\begin{equation}\label{eq:det5}    
\sup_{\thetapop, \theta_{1}\in B_p(R')}
\Pr\left(\left|\fixedriskeqb^{(1)}\left(\Sigma, \theta_{1}, \thetapop\right)-\tfixedriskeqb^{(1)}\left(\Sigma, \thetapop, \thetaperfok\right)\right|\ge \delta\right)
\le C_6pe^{-p\delta^{4}/C_6},
\end{equation}
where
\begin{equation}\label{eq:tfixedriskeqb}
    \begin{split}
\tfixedriskeqb^{(1)}&\left(\Sigma, \thetapop, \thetaperfok\right)= \kappa \tr\left[ \Sigma^2 \left( \Sigma + \tau I_p \right)^{-2} \right]\\
  &  \cdot \frac{   \sigma^2 + \tau^2 \mathbb E_{g^{(1)}}\left[\left\| \left( \Sigma + \tau I_p \right)^{-1} \left(\thetapop+DX^{(1)}\left(\Sigma, \thetaperfok, g^{(1)}\right)\right) \right\|_\Sigma^2 \right] }{ p - \kappa \tr\left[ \Sigma^2 \left( \Sigma + \tau I_p \right)^{-2} \right] }.
  \end{split}
\end{equation}
By using \eqref{eq:X1} and computing the expectation with respect to $g^{(1)}$, we have 

\begin{equation}\label{eq:tfixedriskeqb2}
\begin{split}    
\mathbb E_{g^{(1)}}&\left[    \left\| \left( \Sigma + \tau I_p \right)^{-1} \left(\thetapop+DX^{(1)}\left(\Sigma, \thetaperfok, g^{(1)}\right)\right) \right\|_\Sigma^2\right]\\&=    \mathbb E_{g^{(1)}}\Biggl[\biggl\| \left( \Sigma + \tau I_p \right)^{-1} \thetapop+\left( \Sigma + \tau I_p \right)^{-1}D\left( \Sigma + \tau I_p \right)^{-1}\Sigma \thetaperfok\\
&\qquad\qquad\qquad  +\left( \Sigma + \tau I_p \right)^{-1} D\left( \Sigma + \tau I_p \right)^{-1}\Sigma^{1/2}\frac{ \gamma^{(1)}(\thetaperfok)g^{(1)}}{\sqrt{p}} \biggr\|_\Sigma^2\Biggr]\\
&=\left\| \left( \Sigma + \tau I_p \right)^{-1} \thetapop+\left( \Sigma + \tau I_p \right)^{-1}D\left( \Sigma + \tau I_p \right)^{-1}\Sigma \thetaperfok\right\|_\Sigma^2\\
&\qquad\qquad\qquad+\frac{\left(\gamma^{(1)}(\thetaperfok)\right)^2}{p}\tr\left[ \Sigma \left( \Sigma + \tau I_p \right)^{-2} D  \Sigma \left( \Sigma + \tau I_p \right)^{-2} D \right],
\end{split}
\end{equation}
where in the last step we have used the circulant property of the trace. By plugging \eqref{eq:tfixedriskeqb2} into \eqref{eq:tfixedriskeqb}, using the expression for $\gamma^{(1)}(\thetaperfok)$ in \eqref{eq:gamma1ex} and recalling that $\thetaperfok=\thetapop+D\theta_0$, one readily obtains that the RHS of \eqref{eq:tfixedriskeqb} coincides with the last two lines of the RHS of \eqref{eq:fixedriskeq2}. As $\fixedriskeq^{(1)}\left(\Sigma, \theta_{1}, \thetapop\right)=\fixedriskeqa^{(1)}\left(\Sigma, \theta_{1}, \thetapop\right)+\fixedriskeqb^{(1)}\left(\Sigma, \theta_{1}, \thetapop\right)$, the desired result readily follows by combining 
\eqref{eq:det3}, \eqref{eq:det4} and \eqref{eq:det5}.
\end{proof}

\paragraph{Concluding the argument.}
Note that
\begin{equation*}
    \begin{split}
 \tr\big[\Sigma (\Sigma+\tau I_p)^{-2} D \Sigma^{3} (\Sigma+\tau I_p)^{-2} D\big]&= O(\|D\|_{\mathrm{op}}^{2}),\\
\tr\big[\Sigma (\Sigma+\tau I_p)^{-2} D \Sigma (\Sigma+\tau I_p)^{-2} D\big] &= O(\|D\|_{\mathrm{op}}^{2}).
    \end{split}
\end{equation*}
Furthermore, we have
\begin{equation*}
    \begin{split}
        &\left\| \left( \Sigma + \tau I_p \right)^{-1} \Sigma D \left( \Sigma + \tau I_p \right)^{-1} \Sigma (\thetapop+D\theta_0) - \tau(\Sigma+\tau I_p)^{-1}\thetapop \right\|_\Sigma^2\\
        & =\left\| \left( \Sigma + \tau I_p \right)^{-1} \Sigma D \left( \Sigma + \tau I_p \right)^{-1} \Sigma \thetapop - \tau(\Sigma+\tau I_p)^{-1}\thetapop \right\|_\Sigma^2+O(\|D\|_{\mathrm{op}}^{2})\\
        & =\left\|\tau(\Sigma+\tau I_p)^{-1}\thetapop\right\|_\Sigma^2\\
        &\qquad\qquad -2\tau\langle (\Sigma+\tau I_p)^{-1}\thetapop, \Sigma \left( \Sigma + \tau I_p \right)^{-1} \Sigma D \left( \Sigma + \tau I_p \right)^{-1} \Sigma \thetapop\rangle+O(\|D\|_{\mathrm{op}}^{2})\\
        &=\tau
\langle \thetapop, (\Sigma+\tau I_p )^{-1}\left(\tau I_p
-2
( \Sigma+\tau I_p )^{-1}
\Sigma^2 D
\right)\Sigma( \Sigma+\tau I_p )^{-1}
\thetapop\rangle+O(\|D\|_{\mathrm{op}}^{2}). 
    \end{split}
\end{equation*}
Similarly, we have
\begin{equation*}
    \begin{split}
&\left\| \left( \Sigma + \tau I_p \right)^{-1} \left(\thetapop + D \left( \Sigma + \tau I_p \right)^{-1} \Sigma (\thetapop+D\theta_0) \right) \right\|_\Sigma^2  \\
&=\left\| \left( \Sigma + \tau I_p \right)^{-1} \left(\thetapop + D \left( \Sigma + \tau I_p \right)^{-1} \Sigma \thetapop \right) \right\|_\Sigma^2+O(\|D\|_{\mathrm{op}}^{2})\\
&=\left\| \left( \Sigma + \tau I_p \right)^{-1} \thetapop  \right\|_\Sigma^2+2\langle \left( \Sigma + \tau I_p \right)^{-1}\thetapop, \Sigma\left( \Sigma + \tau I_p \right)^{-1}D\left( \Sigma + \tau I_p \right)^{-1}\Sigma\thetapop \rangle +O(\|D\|_{\mathrm{op}}^{2})\\
&=\langle \thetapop,\left( \Sigma+\tau I_p \right)^{-1}
\left(I_p+2\left( \Sigma+\tau I_p \right)^{-1}
\Sigma D\right)\Sigma\left( \Sigma+\tau I_p \right)^{-1}\thetapop\rangle+O(\|D\|_{\mathrm{op}}^{2}).
    \end{split}
\end{equation*}
Recalling the definitions \eqref{eq:defdet} and \eqref{eq:fixedriskeq2}, we conclude that  
\begin{equation}\label{eq:equalityc}    
\fixedriskeq^{(2)}\left(\Sigma, \theta_0, \thetapop\right)=\fixedriskeq\left(\Sigma, \thetapop, D, \lambda\right)+O(\|D\|_{\mathrm{op}}^{2}).
\end{equation}
Thus, the desired result follows from \eqref{eq:equalityc} and Lemma \ref{lemma:2step}. 

\section{\revised{Extension to sub-Gaussian data}}\label{app:extension}

\revised{Throughout this appendix, we relax Assumption \ref{assum:model} as follows.}

\revised{\begin{assumption}[Regression performative model -- relaxed assumption]
For $\theta \in \R^p$, samples from $\D(\theta)$ are taken i.i.d.\ with features $x$ drawn independently of $\theta$ and such that $\Sigma^{-1/2}x$ has independent, zero mean, unit variance and uniformly sub-Gaussian entries.  The label $y$ is given by
\vspace{-.3em}
\begin{equation}\label{eq:data-bis}    
y = x^\top \thetapop + x^\top D \theta + w, \quad w \sim \N(0, \sigma^2).
\end{equation}
We assume $p = 2d$, $(\thetapop)^{\top} = (a^\top, 0)$ with $a$ having zero mean and covariance $I_d/d$, and $D = \di(b, c)$ where $b, c \in \R^d$ with $\|b\|_{\infty}, \|c\|_{\infty} < 1$. We further assume that $a\sqrt{d}$ has sub-Gaussian norm upper bounded by a universal constant (independent of $d$).
\label{assum:model-bis}
\end{assumption}}

\revised{\begin{theorem}[Excess risk -- over-parameterized, relaxed assumptions]\label{thm:over-rel}
Let Assumption \ref{assum:model-bis} hold. Let $R>0$ be a constant s.t.\ $\thetapop\in B_p(R)$ and let $\theta_0$ be sampled uniformly on the unit sphere. Assume that $\kappa, \sigma, \lambda\in (1/M, M)$ and $\|\Sigma\|_{\mathrm{op}},\ \|\Sigma^{-1}\|_{\mathrm{op}} \le M$ for some constant $M>1$. Then, there exists a constant $C=C\left(M, R\right)$ such that for any $\delta \in (0,1/2]$, with probability at least $1-C\delta^{-7}p^{-1/8}$, %
\begin{equation}    
\left|\mathcal{R}(\Sigma, \theta_{2}, \thetapop)-\fixedriskeq\left(\Sigma, \thetapop, D, \lambda\right)\right|\le \delta+O(\|D\|_{\mathrm{op}}^2),
\end{equation}
where $\fixedriskeq\left(\Sigma, \thetapop, D, \lambda\right)$ is given by (\ref{eq:defdet}). 
\end{theorem}}

\revised{\begin{lemma}[Norm control]\label{lemma:norm}
In the setting of Theorem \ref{thm:over-rel}, we have that
    \begin{align}
\|\theta_1\|_2&\le C,\label{eq:condnorm1}\\
\|\theta_2\|_2&\le C,\label{eq:condnorm1-bis}\\
\|\theta_0\|_\infty&\le C\frac{\log p}{\sqrt{p}},\label{eq:condnorm2}\\
\|\thetapop\|_\infty&\le C\frac{\log p}{\sqrt{p}},\label{eq:condnorm3}\\
\|\theta_1\|_\infty&\le C\frac{\log p}{\sqrt{p}},\label{eq:condnorm4}
    \end{align}    
    with probability at least $1-Ce^{-\log^2 p/C}$, where $C=C(R, M)$ is a constant depending only on $R, M$ (and not on $n, p$).
\end{lemma}}

\begin{proof}
\revised{  We start by proving (\ref{eq:condnorm1}). The claim follows by extending the argument of Proposition 11 in \cite{ildizhigh} and we repeat it here for completeness. Recall that
  $$
  \theta_{1} = \frac{1}{p} \left(\frac{1}{p} X^{0  \top} X^{0} + \lambda I_p\right)^{-1} X^{0 \top} X^0(\thetapop+D\theta_0)+\frac{1}{p} \left(\frac{1}{p} X^{0  \top} X^{0} + \lambda I_p\right)^{-1} X^{0 \top}w. 
  $$
 Note that 
$$
 \left\|\frac{1}{p} \left(\frac{1}{p} X^{0  \top} X^{0} + \lambda I_p\right)^{-1} X^{0 \top} X^0\right\|_{\mathrm{op}}\le 1,
 $$
 which implies that 
 $$
\left\| \frac{1}{p} \left(\frac{1}{p} X^{0  \top} X^{0} + \lambda I_p\right)^{-1} X^{0 \top} X^0(\thetapop+D\theta_0)\right\|_2\le C_1,
 $$
 for some constant $C_1=C_1(R, M)$. Next, we can write
 \begin{equation*}
 \begin{split}
\left\|\frac{1}{p} \left(\frac{1}{p} X^{0  \top} X^{0} + \lambda I_p\right)^{-1} X^{0 \top}w\right\|_2^2 &= \frac{w^\top X^0}{p}\left(\frac{1}{p} X^{0  \top} X^{0} + \lambda I_p\right)^{-2}\frac{X^{0 \top}w}{p}\\
&\le \frac{w^\top w}{p}\left\|\frac{1}{p}X^0\left(\frac{1}{p} X^{0  \top} X^{0} + \lambda I_p\right)^{-2}X^{0 \top}\right\|_{\mathrm{op}}.
\end{split}
\end{equation*}
Using Bernstein's inequality, we have that $w^\top w/p$ is upper bounded by $C_2=C_2(R, M)$ with probability at least $1-e^{-p/C_2}$. Furthermore, $
\left\|\frac{1}{p}X^0\left(\frac{1}{p} X^{0  \top} X^{0} + \lambda I_p\right)^{-2}X^{0 \top}\right\|_{\mathrm{op}}\le \left\|\frac{1}{p}X^0\left(\frac{1}{p} X^{0  \top} X^{0} + \lambda I_p\right)^{-2}X^{0 \top}\right\|_{\mathrm{op}}
$ is also upper bounded by a universal constant. Thus, an application of the triangle inequality gives (\ref{eq:condnorm1}). Repeating the same argument with $\theta_1$ in place of $\theta_0$ and $X^1$ in place of $X^0$ readily gives (\ref{eq:condnorm1-bis}).}

\revised{Let $v\in \mathbb R^p$ be a vector such that $v\sqrt{p}$ has sub-Gaussian norm upper bounded by a universal constant (independent of $p$). We will now show that
\begin{equation}\label{eq:inftyn}
    \|v\|_\infty\le C\frac{\log p}{\sqrt{p}},
\end{equation}
with probability at least $1-e^{-\log^2 p}$. To see this, it suffices to note that the $j$-th coordinate $v_j\sqrt{p}$ is sub-Gaussian with sub-Gaussian norm upper bounded by a universal constant. Thus, 
$$
\mathbb P(|v_j\sqrt{p}|>t)\le 2e^{-t^2/C_3},
$$
for some universal constant $C_3$. Taking $t=C\log p$ and doing a union bound over $j\in \{1, \ldots, p\}$ gives (\ref{eq:inftyn}).}

\revised{Since $\theta_0$ is sampled uniformly on the sphere, (\ref{eq:condnorm2}) is implied by (\ref{eq:inftyn}). Therefore, $\theta_0\sqrt{p}$ has sub-Gaussian norm upper bounded by a universal constant (independent of $p$). Furthermore,  (\ref{eq:condnorm3}) is implied by (\ref{eq:inftyn}) since $\thetapop$ satisfies Assumption \ref{assum:model-bis}. Finally, letting $\|\cdot\|_{\psi_2}$ denote the sub-Gaussian norm of a vector, we have
\begin{equation}
\begin{split}
    \|\theta_1\sqrt{p}\|_{\psi_2}&\le \left\|\frac{1}{p} \left(\frac{1}{p} X^{0  \top} X^{0} + \lambda I_p\right)^{-1} X^{0 \top} X^0\right\|_{\mathrm{op}} (\|\thetapop\sqrt{p}\|_{\psi_2}+\|\theta_0\sqrt{p}\|_{\psi_2})\\
    &\hspace{10em}+\left\|\frac{1}{\sqrt{p}} \left(\frac{1}{p} X^{0  \top} X^{0} + \lambda I_p\right)^{-1} X^{0 \top} \right\|_{\mathrm{op}}\left\|w\right\|_{\psi_2}\\
    &\le \|\thetapop\sqrt{p}\|_{\psi_2}+\|\theta_0\sqrt{p}\|_{\psi_2}+\left\|w\right\|_{\psi_2},
\end{split}
\end{equation}
which is upper bounded by a universal constant. 
Thus, (\ref{eq:condnorm4}) is also implied by (\ref{eq:inftyn}) and the proof is complete.}
\end{proof}

\revised{\begin{lemma}\label{lemma:1step-bis}
Let Assumption \ref{assum:model-bis} hold. Let $R>0$ be a constant such that $\thetapop\in B_p(R)$ and let $\theta_0$ be sampled uniformly on the unit sphere. Assume that $\kappa, 
\sigma, \lambda \in (1/M, M)$ and $\|\Sigma\|_{\mathrm{op}},\ \|\Sigma^{-1}\|_{\mathrm{op}} \le M$ for some constant $M>1$. Then, there exists a constant $C=C\left(M, R\right)$ such that, for any $\delta \in (0, 1/2]$, the following holds %
\begin{equation}\label{eq:det1-ter}    
\sup_{\thetapop, \theta_0\in B_p(R)}
\Pr\left(\left|\mathcal R_{1}(\Sigma, \theta_{1}, \thetapop)-\fixedriskeq^{(1)}\left(\Sigma, \theta_0, \thetapop\right)\right|\ge \delta\right)
\le Cpe^{-p\delta^{4}/C},
\end{equation}
with probability at least $1-C\delta^{-7}p^{-1/8}$,
where $\fixedriskeq^{(1)}$ is given by (\ref{eq:R1eq}).
\end{lemma}}

\begin{proof}
\revised{By Lemma \ref{lemma:norm}, we have that $\thetapop+D\theta_0$ satisfies the delocalization condition of Proposition 10.3 by \cite{han2023distribution}. This implies that the hypotheses of Theorem 2.4 by \cite{han2023distribution} are satisfied when we train using $\thetapop+D\theta_0$. Thus, we can now follow the same steps as in Lemma \ref{lemma:1step} which invokes Theorem 3 by \cite{ildizhigh}. In particular, Theorem 3 by \cite{ildizhigh} uses Theorem 4 therein plus the bound on $\|\theta_1\|_2$ given by Lemma \ref{lemma:norm}. Thus, it suffices to replace the application of Theorem 4 by \cite{ildizhigh} with the application of Theorem 2.4 by \cite{han2023distribution}, and the desired result readily holds.} 
\end{proof}

\revised{
\begin{proof}[Proof of Theorem \ref{thm:over-rel}]
    By Lemma \ref{lemma:norm}, we have that $\thetapop+D\theta_0$ and $\thetapop+D\theta_1$ satisfy the delocalization condition of Proposition 10.3 by \cite{han2023distribution}. This implies that the hypotheses of Theorem 2.4 by \cite{han2023distribution} are satisfied when we train using either $\thetapop+D\theta_0$ or $\thetapop+D\theta_1$ as vector of regression coefficients. Consequently, the desired result is obtained by following the same steps as in the proof of Theorem \ref{thm:over}, the only differences being that \emph{(i)} we apply Lemma \ref{lemma:1step-bis} in place of Lemma \ref{lemma:1step}, and \emph{(ii)} we apply Theorem 2.4 by \cite{han2023distribution} in place of Theorem 4 by \cite{ildizhigh}. This requires an upper bound on $\|\theta_1\|_2, \|\theta_2\|_2$ which is provided by Lemma \ref{lemma:norm}.
\end{proof}}

\section{Proof of Theorem \ref{thm:equiv}}\label{app:pfequiv}

We start by computing explicitly $\mathbb E_{\thetapop}\fixedriskeq\left(\Sigma, \thetapop, D, \lambda\right)$.

\begin{lemma}\label{lemma:explicit}
Consider the setting of Theorem \ref{thm:over}, assume that $a$ has covariance $I_d/d$, and let $\Sigma=\begin{bmatrix}I_d&\rho I_d\\ \rho I_d&I_d\end{bmatrix}$. Then, we have that
\begin{equation*}
\begin{split}    
\mathbb E_{\thetapop}\fixedriskeq\left(\Sigma, \thetapop, D, \lambda\right)&=\widetilde{\mathcal R}(D, \lambda, \rho)+ O(\bar b\rho^2+\rho^4),\\
 \widetilde{\mathcal R}(D, \lambda, \rho)&:=\mathcal R_0(\lambda, \rho) + \bar b  A_1(\lambda)   + \bar c \rho^2 A_2(\lambda),
\end{split}
\end{equation*}
where $\bar b=\tr[\di(b)]/d, \bar c=\tr[\di(c)]/d$ and the auxiliary functions 
 $\mathcal R_0(\lambda, \rho)$, $A_1(\lambda)$, and $A_2(\lambda)$ %
 are given by 
\begin{equation}   \label{eq:explexpr} 
\begin{aligned}
\mathcal R_0(\lambda, \rho) &= \frac{\tau^{2}}{(1+\tau)^{2}}
+  \frac{\kappa}{(1+\tau)^{2}-\kappa} \left( \sigma^{2}+\frac{\tau^{2}}{(1+\tau)^{2}} \right)\\&\hspace{-1em} + \rho^2 \left( \frac{\tau^{2} (1-2\tau)}{(1+\tau)^{4}}
+ \frac{\kappa \tau^{2} (1-2\tau)}{(1+\tau)^{4} \left((1+\tau)^{2}-\kappa \right)}
+ \frac{\kappa\tau(\tau-2)}{\left((1+\tau)^{2}-\kappa\right)^{2}}\left(\sigma^{2}+\frac{\tau^{2}}{(1+\tau)^{2}}\right) \right), \\
A_1(\lambda) &= - \frac{2\tau}{(1+\tau)^{3}}
+ \frac{2\kappa\tau^{2}}{(1+\tau)^{3}\left((1+\tau)^{2}-\kappa\right)}, \\
A_2(\lambda) &= - \frac{4\tau^{3}}{(1+\tau)^{5}} + \frac{2 \kappa \tau^{3} (\tau^{2}-1)}{(1+\tau)^{6}\left((1+\tau)^{2}-\kappa\right)}.
\end{aligned}
\end{equation}
\end{lemma}

\begin{proof}
Given a $p\times p$ matrix $M$, let us denote by $(M)_1$ its top-left $d\times d$ block. For any $M\in \mathbb R^{p\times p}$, we have 
\begin{equation*}
\begin{split}
        \mathbb E_{\thetapop}\left[\langle \thetapop, M\thetapop\rangle\right]&=    \mathbb E_{\thetapop}\left[ (\thetapop)^\top M\thetapop\right]=    \mathbb E_{\thetapop}\left[\tr\left[ (\thetapop)^\top M\thetapop\right]\right]\\
        &=    \mathbb E_{\thetapop}\left[\tr\left[M  \thetapop(\thetapop)^\top\right]\right]=    \tr\left[(M)_1\right]/d,
        \end{split}
\end{equation*}
where the third equality uses the circulant property of the trace and the last one that $(\thetapop)^{\top} = (a^\top, 0)$ with $a$ having covariance $I_d/d$. Thus, from \eqref{eq:defdet}, we have
\begin{equation}  \label{eq:defdetexp}  
\begin{aligned}
\mathbb E_{\thetapop}\fixedriskeq\left(\Sigma, \thetapop, D, \lambda\right)
&= 
\frac{\tau^{2}}{d}
\tr\left[
\left(
\Sigma
\left( \Sigma+\tau I_p \right)^{-2}
\right)_{1}
\right]
\\
&\quad
+ \kappa
\tr\left[\Sigma^{2}\left( \Sigma+\tau I_p \right)^{-2}\right]
\frac{
\sigma^{2}
+ 
\frac{\tau^{2}}{d}
\tr\left[
\left(\Sigma
\left( \Sigma+\tau I_p \right)^{-2}
\right)_{1}
\right]
}{
p - \kappa
\tr\left[\Sigma^{2}\left( \Sigma+\tau I_p \right)^{-2}\right]
}
\\
&\quad
- 
\frac{2\tau}{d}
\tr\left[
\left(
\left( \Sigma+\tau I_p \right)^{-2}
\Sigma^{2} D
\left( \Sigma+\tau I_p \right)^{-1}
\Sigma
\right)_{1}
\right]
\\
&\quad
+ \frac{2\kappa\tau^{2}}{d}
\tr\left[\Sigma^{2}\left( \Sigma+\tau I_p \right)^{-2}\right]
\frac{
\tr\left[
\left(
\left( \Sigma+\tau I_p \right)^{-2}
\Sigma D
\left( \Sigma+\tau I_p \right)^{-1}
\Sigma
\right)_{1}
\right]
}{
p - \kappa
\tr\left[\Sigma^{2}\left( \Sigma+\tau I_p \right)^{-2}\right]
}.
\end{aligned}
\end{equation}
Note that
\(
\Sigma +\tau I_p=
\begin{bmatrix}1+\tau&\rho\\ \rho&1+\tau\end{bmatrix}\otimes I_d
\)
has inverse
\(
(\Sigma+\tau I_p)^{-1}=
\frac{1}{(1+\tau)^2-\rho^2}
\begin{bmatrix}1+\tau&-\rho\\ -\rho&1+\tau\end{bmatrix}\otimes I_d.
\)
Furthermore,
\begin{equation*}
\begin{split}
    (\Sigma+\tau I_p)^{-2}
&=\frac{1}{\left((1+\tau)^2-\rho^2\right)^2}
\begin{bmatrix}(1+\tau)^2+\rho^2&-2\rho(1+\tau)\\ -2\rho(1+\tau)&(1+\tau)^2+\rho^2\end{bmatrix}\otimes I_d,\\
\Sigma^2&=
\begin{bmatrix}1+\rho^2&2\rho\\ 2\rho&1+\rho^2\end{bmatrix}\otimes I_d.
\end{split}    
\end{equation*}
A direct block multiplication gives
\begin{align*}
\tr\left[\left(\Sigma(\Sigma+\tau I_p)^{-2}\right)_1\right]
&= d\ \frac{(1+\tau)^2-\rho^2(1+2\tau)}{\left((1+\tau)^2-\rho^2\right)^2}, \\
\tr\left[\Sigma^{2}(\Sigma+\tau I_p)^{-2}\right]
&= 2d\ \frac{(1+\tau-\rho^2)^2+\rho^2\tau^2}{\left((1+\tau)^2-\rho^2\right)^2},\\
\tr\left[\left((\Sigma+\tau I_p)^{-2}\Sigma^{2}D(\Sigma+\tau I_p)^{-1}\Sigma\right)_1\right]
&=\frac{(1+\tau-\rho^2)\left((1+\tau-\rho^2)^2+\rho^2\tau^2\right)}{\left((1+\tau)^2-\rho^2\right)^3}\tr[\di(b)]\\
&\qquad
 +\frac{2(1+\tau-\rho^2)\rho^2\tau^2}{\left((1+\tau)^2-\rho^2\right)^3}\tr[\di(c)],\\
\tr\left[\left((\Sigma+\tau I_p)^{-2}\Sigma D(\Sigma+\tau I_p)^{-1}\Sigma\right)_1\right]
&=\frac{(1+\tau-\rho^2)\left((1+\tau)(1+\tau-\rho^2)-\rho^2\tau\right)}{\left((1+\tau)^2-\rho^2\right)^3}\tr[\di(b)]\\
&\qquad
+\frac{\rho^2\tau\left(\rho^2+\tau^2-1\right)}{\left((1+\tau)^2-\rho^2\right)^3}\tr[\di(c)].
\end{align*}
Expanding each rational function at $\rho=0$ using
\begin{equation*}
    \begin{split}        
\frac{1}{( (1+\tau)^2-\rho^2 )^2}
&=\frac{1}{(1+\tau)^4}\left(1+\frac{2\rho^2}{(1+\tau)^2}\right)+O(\rho^4), \\
 \frac{1}{( (1+\tau)^2-\rho^2 )^3}
&=\frac{1}{(1+\tau)^6}\left(1+\frac{3\rho^2}{(1+\tau)^2}\right)+O(\rho^4),
    \end{split}
\end{equation*}
yields, to order $\rho^2$,
\begin{align*}
\frac{1}{d}\tr\left[\left(\Sigma(\Sigma+\tau I_p)^{-2}\right)_1\right]
&= \left(\frac{1}{(1+\tau)^2}+\rho^2\frac{1-2\tau}{(1+\tau)^4}\right)+O(\rho^4),\\
\frac{1}{d}\tr\left[\Sigma^{2}(\Sigma+\tau I_p)^{-2}\right]
&= \frac{2}{(1+\tau)^2}+2\rho^2\frac{\tau^2-2\tau}{(1+\tau)^4}+O(\rho^4),\\
\frac{1}{d}\tr\left[\left((\Sigma+\tau I_p)^{-2}\Sigma^{2}D(\Sigma+\tau I_p)^{-1}\Sigma\right)_1\right]
&=\frac{\bar b}{(1+\tau)^{3}}
+ \rho^{2}\left(\frac{\tau^{2}-3\tau}{(1+\tau)^{5}}\bar b+\frac{2\tau^{2}}{(1+\tau)^{5}}\bar c\right)
+O(\rho^4),\\
\frac{1}{d}\tr\left[\left((\Sigma+\tau I_p)^{-2}\Sigma D(\Sigma+\tau I_p)^{-1}\Sigma\right)_1\right]
&=\frac{\bar b}{(1+\tau)^{3}}
+ \rho^{2}\left(\frac{1-3\tau}{(1+\tau)^{5}}\bar b+\frac{\tau(\tau^{2}-1)}{(1+\tau)^{6}}\bar c\right)
+O(\rho^4).
\end{align*}
Moreover, we have that
\[
\frac{\kappa\operatorname{tr}\left[\Sigma^{2}(\Sigma+\tau I_p)^{-2}\right]}
{p-\kappa\operatorname{tr}\left[\Sigma^{2}(\Sigma+\tau I_p)^{-2}\right]}
= \frac{\kappa}{(1+\tau)^{2}-\kappa}
+ \rho^{2}\frac{\kappa\tau(\tau-2)}{\left((1+\tau)^{2}-\kappa\right)^{2}}
+ O(\rho^{4}).
\]
Plugging these into \eqref{eq:defdetexp} gives the claimed result. 
\end{proof}

Let us further define
\begin{equation}
\tau^*(D, \rho):=\arg\min_{\tau\ge 0}\widetilde{\mathcal R}(D, \lambda, \rho),\qquad \tau_0^*(\rho):=\arg\min_{\tau\ge 0}\mathcal R_0(\lambda, \rho), \qquad  \tau_0:=\tau_0^*(0).
\end{equation}
Then, the following result proves an expression for $\tau^*(D, \rho)$, up to order $\rho^2$. %

\begin{lemma}\label{lemma:taustar}
In the setting of Theorem \ref{thm:equiv}, we have that
\begin{align*}
\tau^*(D, \rho)
&= \tau^{*}_{0}(\rho) + \bar b \left( B_3(\sigma, \kappa) + O(\rho^2)\right) + \bar c\left( \rho^2 C_3(\sigma, \kappa) + O(\rho^4) \right) + O(\bar b^2+\bar c^2),
\end{align*}
where
\begin{equation}\label{eq:tau0}
    \begin{aligned}
\tau_0 &= \frac{1 + \kappa + \kappa \sigma^2 + \sqrt{(1 + \kappa + \kappa \sigma^2)^2 - 4\kappa}}{2} - 1,\\
    \tau^{*}_{0}(\rho) &= \tau_{0} -\rho^{2}
\frac{\kappa\tau_{0}^{2}}
{(1+\tau_{0})\left((1+\tau_{0})^{2}-\kappa\right)} + O(\rho^4), \\
B_3(\sigma, \kappa) &= -
\frac{2(1+\tau_{0})^{4}-3(\kappa+1)(1+\tau_{0})^{3}
+4\kappa(1+\tau_{0})^{2}+\kappa(\kappa+1)(1+\tau_{0})-2\kappa^{2}}
{(1+\tau_{0})^{2}\left((1+\tau_{0})^{2}-\kappa\right)}, \\
C_3(\sigma, \kappa) &= - \frac{\tau_{0}^{2} \left(
 4\tau_{0}^{4}+(6-3\kappa)\tau_{0}^{3}-(6+3\kappa)\tau_{0}^{2}
+(\kappa^{2}+9\kappa-14)\tau_{0}-3\kappa^{2}+9\kappa-6 \right)}{(1+\tau_{0})^{4}\left((1+\tau_{0})^{2}-\kappa\right)}.
\end{aligned}
\end{equation}
\end{lemma}

\begin{proof}
A direct differentiation gives
\[
\begin{aligned}
\frac{\mathrm{d}}{\mathrm{d}\tau}\widetilde{\mathcal R}(D,\lambda,\rho)
&= \frac{2}{1+\tau}\left(
\frac{\tau}{(1+\tau)^{2}-\kappa}
-\frac{\kappa\left(\sigma^{2}(1+\tau)^{2}+\tau^{2}\right)}{\left((1+\tau)^{2}-\kappa\right)^{2}}
\right) \\
&\quad+\rho^2 \left( \frac{2\tau\left(\tau^{2}-4\tau+1\right)}{(1+\tau)^{5}} +\frac{2\kappa\tau\left((1+\tau)^{2}\left(3\tau^{2}-5\tau+1\right)-\kappa\left(\tau^{2}-4\tau+1\right)\right)}{(1+\tau)^{5}\left((1+\tau)^{2}-\kappa\right)^{2}} \right.\\ &\hspace{10em} -\frac{2\kappa^2\left(\kappa\sigma^{2}(\tau-1)(1+\tau)^{3}+\kappa\tau^{2}(\tau^{2}+\tau-3) \right)}{(1+\tau)^{3}\left((1+\tau)^{2}-\kappa\right)^{3}} \\& \left. \hspace{10em} -\frac{2\kappa \left(\sigma^{2}(1+\tau)^{4}(\tau^{2}-4\tau+1)+\tau^{2}(1+\tau)^{2}(\tau^{2}-5\tau+3)\right)}{(1+\tau)^{3}\left((1+\tau)^{2}-\kappa\right)^{3}}   \right)
\\
&\quad+\bar b \left(
\frac{4\tau-2}{(1+\tau)^{4}}+\frac{\kappa\tau(4-6\tau)}{(1+\tau)^{4}\left((1+\tau)^{2}-\kappa \right)}
-\frac{4\kappa^{2}\tau^{2}}{(1+\tau)^{4} \left((1+\tau)^{2}-\kappa\right)^{2}}
\right)\\
&\quad+\bar c \rho^2 \left( \frac{4\tau^{2}(2\tau-3)}{(1+\tau)^{6}} + \frac{2\kappa\tau^{2}(1+\tau)\left(\kappa(\tau^{2}-6\tau+3)-(1+\tau)^{2}(3\tau^{2}-8\tau+3)\right)}{(1+\tau)^{7}\left((1+\tau)^{2}-\kappa\right)^{2}}\right).
\end{aligned}
\]
With this explicit derivatives, the stationarity equation $\frac{\mathrm{d}}{\mathrm{d}\tau}\widetilde{\mathcal R}(D, \lambda, \rho)=0$ is equivalent to
\(
\frac{F\left(\tau,\bar b,\bar c,\rho^{2}\right)}{(1 + \tau)^7 ((1+\tau)^2 - \kappa)^3} =0,
\)
where
\[
F\left(\tau,\bar b,\bar c,\rho^{2}\right)
= F_{0}(\tau)
\rho^2 F_{\rho}(\tau) + \bar b F_{b}(\tau)
 + \bar c \rho^2 F_{\rho c}(\tau),
\]
\[
F_{0}(\tau)
= 2(1+\tau)^{7}\left(\kappa-(1+\tau)^{2}\right)
\left(\kappa\sigma^{2}\tau+\kappa\sigma^{2}+\kappa\tau-\tau^{2}-\tau\right),
\]
\[
F_{b}(\tau)
= 2(1+\tau)^{4}\left((1+\tau)^{2}-\kappa\right)
\left(
(\tau-1)\kappa^{2}
+(\tau+1)(2-2\tau-3\tau^{2})\kappa
+(1+\tau)^{3}(2\tau-1)
\right),
\]
\[
\begin{aligned}
F_{\rho}(\tau)
= -2(1+\tau)^{5}\left(
\right.&
\kappa^{2}\left(\tau(\tau^{2}+\tau-1)+\sigma^{2}(1+\tau)^{2}(\tau-1)\right)\\
&\quad
+\kappa(1+\tau)^{2}\left(\tau(\tau^{2}-6\tau+2)+\sigma^{2}(1+\tau)(\tau^{2}-4\tau+1)\right)\\
&\left.\quad
-\tau(1+\tau)^{3}(\tau^{2}-4\tau+1)
\right),
\end{aligned}
\]
\[
F_{\rho c}(\tau)
= 2\tau^{2}(1+\tau)^{2}\left((1+\tau)^{2}-\kappa\right)
\left(
(\tau-3)\kappa^{2}
-3(\tau+1)(\tau^{2}-3)\kappa
+2(\tau+1)^{3}(2\tau-3)
\right).
\]

Setting $\bar b=\bar c=\rho^{2}=0$ yields
\[
(1+\tau)^{2}-(1+\kappa+\kappa\sigma^{2})(1+\tau)+\kappa=0,
\]
and the desired minimum corresponds to its largest solution, which is given by $\tau_{0}$ as expressed in the statement. It is easy to see that 
\[
\begin{aligned}
\partial_{\tau}F\left(\tau_{0},0,0,0\right))
&= -2(1+\tau_0)^7\left((1+\tau_0)^2-\kappa\right)\left(2(1+\tau_0)-(1+\kappa+\kappa\sigma^2)\right) \\&= -2(1+\tau_0)^7\left((1+\tau_0)^2-\kappa\right)\sqrt{(1+\kappa+\kappa\sigma^2)^2-4\kappa}\neq 0.
\end{aligned}
\]
Therefore, the implicit function theorem gives a smooth map
\(
\tau^{*}(\bar b,\bar c,\rho^{2})
\)
with $\tau^{*}(0,0,0)=\tau_{0}$ and $F(\tau^{*},\cdot)=0$.
Differentiating $F=0$ at $(\tau_{0},0,0,0)$ in each small parameter and dividing by $\partial_{\tau}F(\tau_{0},0,0,0)$ yields the linear expansion for $\tau^{*}-\tau_{0}$. The coefficient for $\rho^2$ in $\tau_0^*(\rho)$ is given by
\[
 \partial_{\rho^2}\tau^*(0,0,0) = - \frac{\partial_{\rho^2}F}{\partial_{\tau}F} \Bigg|_{(\tau_0,0,0,0)} = - \frac{F_{\rho}(\tau_0)}{\partial_{\tau}F_0(\tau_0,0,0,0)}.
\]
The $\bar b$ coefficient, $B_3$, is
\[
 B_3(\sigma, \kappa) = \partial_{\bar b}\tau^*(0,0,0) = - \frac{\partial_{\bar b}F}{\partial_{\tau}F} \Bigg|_{(\tau_0,0,0,0)} = - \frac{F_b(\tau_0)}{\partial_{\tau}F_0(\tau_0,0,0,0)}.
\]
The $\bar c \rho^2$ coefficient, $C_3$, is found from the mixed partial derivative:
\[
 C_3(\sigma, \kappa) = \partial_{\rho^2}\partial_{\bar c}\tau^*(0,0,0) = - \frac{\partial_{\rho^2}\partial_{\bar c}F}{\partial_{\tau}F} \Bigg|_{(\tau_0,0,0,0)} = - \frac{F_{\rho c}(\tau_0)}{\partial_{\tau}F_0(\tau_0,0,0,0)}.
\]
Substituting the expressions for $\partial_{\tau}F_0(\tau_0)$, $F_{\rho}(\tau_0)$, $F_b(\tau_0)$, and $F_{\rho c}(\tau_0)$ and cancelling common factors gives the coefficients as stated in \eqref{eq:tau0}.
\end{proof}

As $\lambda$ and $\tau$ are linked by the fixed point equation \eqref{eq:tau}, an application of Lemma \ref{lemma:taustar} readily gives that 
\begin{equation}
\begin{split}    
\lambdaeqs(D, \rho)&= \lambdaeqsz(\rho)+\bar b (B_1(\sigma, \kappa)+O(\rho^2)) +\bar c \rho^2( C_1(\sigma, \kappa)+O(\rho^2))+O(\bar b^2+\bar c^2),
\end{split}
\end{equation}
where
\begin{equation}\label{eq:lfor}
\begin{aligned}
    \lambdaeqsz(\rho) &= \tau_{0}\left(\kappa^{-1}-\frac{1}{1+\tau_{0}}\right)  + \rho^{2}\left(\tau_{0}\left(\frac{1}{(1+\tau_{0})^{2}}-\frac{1}{(1+\tau_{0})^{3}}\right)\right.\\
&\hspace{10em}\left.-\left(\kappa^{-1}-\frac{1}{(1+\tau_{0})^{2}}\right)
\frac{\kappa\tau_{0}^{2}}{(1+\tau_{0})\left((1+\tau_{0})^{2}-\kappa\right)}
\right) + O(\rho^4), \\
 B_1(\sigma, \kappa) &= -
\frac{2(1+\tau_{0})^{4}-3(\kappa+1)(1+\tau_{0})^{3}
+4\kappa(1+\tau_{0})^{2}+\kappa(\kappa+1)(1+\tau_{0})-2\kappa^{2}}
{\kappa(1+\tau_{0})^{4}},\\
C_1(\sigma, \kappa) &= - 
\frac{\tau_{0}^{2} \left(
4\tau_{0}^{4}+(6-3\kappa)\tau_{0}^{3}-(6+3\kappa)\tau_{0}^{2}
+(\kappa^{2}+9\kappa-14)\tau_{0}-3\kappa^{2}+9\kappa-6 \right)}{\kappa(1+\tau_{0})^{6}}.
\end{aligned}
\end{equation}
This proves \eqref{eq:thmequivl}. %
Next, the corollary below proves %
\eqref{eq:thmequivR}.

\begin{corollary}
Consider the setting of Theorem \ref{thm:equiv} and let $\tau_0$ be given by \eqref{eq:tau0}. Then, we have that %
\begin{equation}
    \fixedriskeqs(D, \rho)= \fixedriskeqs(\rho)+\bar b (B_2(\sigma, \kappa)+O(\rho^2))+\bar c \rho^2( C_2(\sigma, \kappa)+O(\rho^2))+O(\bar b^2+\bar c^2),
\end{equation}
where
\begin{equation}\label{eq:Rfor}    
\begin{aligned}
    \fixedriskeqs(\rho) &= \frac{\tau_{0}^{2}}{(1+\tau_{0})^{2}}
+\frac{\kappa}{(1+\tau_{0})^{2}-\kappa}\left(\sigma^{2}+\frac{\tau_{0}^{2}}{(1+\tau_{0})^{2}}\right) \\&\quad +\rho^{2}\left(
\frac{\tau_{0}^{2}(1-2\tau_{0})}{(1+\tau_{0})^{4}}
+ \frac{\kappa \tau_{0}^{2}(1-2\tau_{0})}{(1+\tau_{0})^{4} \left( (1+\tau_{0})^{2}-\kappa \right)}
\right.\\
&\hspace{10em}\left.+ \frac{\kappa\tau_{0}(\tau_{0}-2)}{\left((1+\tau_{0})^{2}-\kappa\right)^{2}}\left(\sigma^{2}+\frac{\tau_{0}^{2}}{(1+\tau_{0})^{2}}\right) \right) + O(\rho^4), \\
B_2(\sigma, \kappa) &= -\frac{2\tau_{0}}{(1+\tau_{0})^{3}}
+\frac{2\kappa\tau_{0}^{2}}{(1+\tau_{0})^{3}\left((1+\tau_{0})^{2}-\kappa\right)}, \\
C_2(\sigma, \kappa) &= - \frac{4\tau_{0}^{3}}{(1+\tau_{0})^{5}}
+ \frac{2 \kappa \tau_{0}^{3} (\tau_{0}^{2}-1)}{(1+\tau_{0})^{6} \left((1+\tau_{0})^{2}-\kappa \right)}.
\end{aligned}
\end{equation}
\end{corollary}

\begin{proof}
Let us re-define $\widetilde{\mathcal R}(D, \lambda, \rho)$ given in  \eqref{eq:explexpr} as $\widetilde{R}(\tau,\bar b,\bar c,\rho^{2})$ to emphasize its dependence on $\tau,\bar b,\bar c$. %
By definition of $\tau_{0}$, we have $\partial_{\tau}\widetilde{R}(\tau_{0},0,0,0)=0$. Furthermore, from Lemma \ref{lemma:taustar}, we have
\[
\tau^{*}(D, \rho)=\tau_{0}
+O(\bar b + (1 + \bar c) \rho^2 ).
\]
A first–order Taylor expansion of $\widetilde{R}(\tau^{*}(D, \rho), \bar b, \bar c,\rho^{2})$ around $(\tau;\bar b,\bar c,\rho^{2})=(\tau_{0};0,0,0)$ gives
\[
\begin{aligned}
\widetilde{R}(\tau^{*}(D, \rho), \bar b, \bar c,\rho^{2})
&=
\widetilde{R}(\tau_{0},\bar b, \bar c,\rho^{2})
+\partial_{\tau}\widetilde{R}(\tau_{0},0,0,0)(\tau^{*}(D, \rho)-\tau_{0})
\\
&\quad
+~O\left((\tau^{*}(D, \rho)-\tau_{0})\bar b\right)
+O\left((\tau^{*}(D, \rho)-\tau_{0})\rho^{2}\right)
+O(\bar b^2+\bar c^2+\rho^4).
\end{aligned}
\]
As $\partial_{\tau}\widetilde{R}(\tau_{0},0,0,0)=0$, we conclude that %
\[
\widetilde{R}(\tau^{*},\bar b, \bar c, \rho^{2})
=
\widetilde{R}(\tau_{0}, \bar b, \bar c, \rho^2)
+O(\bar b^2+\bar c^2+\rho^4),
\]
and substituting $\tau=\tau_{0}$ in \eqref{eq:explexpr} gives the claimed expansion.
\end{proof}

We now move to the proof of \eqref{eq:relations1}, which follows from the lemma below.

\begin{lemma}\label{lem:transition-kappa-3}
Let $B_1(\sigma, \kappa)$ be given by \eqref{eq:lfor}. Then, for any $\kappa>1$, 
$B_1(\kappa,\cdot)$ has exactly one zero $\sigma_{B_1}(\kappa)>0$, with
\[
B_1(\kappa,\sigma)\ge 0 \ \text{for } 0\le \sigma\le \sigma_{B_1}(\kappa),
\qquad
B_1(\kappa,\sigma)\le 0 \ \text{for } \sigma\ge \sigma_{B_1}(\kappa).
\]
Moreover, as $\kappa\to\infty$,
\[
\sigma_{B_1}^{2}(\kappa)
=\frac{1}{2}
-\frac{7}{18}\kappa^{-1}
+O\left(\kappa^{-2}\right).
\]
\end{lemma}

\begin{proof}
Let us define the shorthands
\begin{equation}
    s(\sigma):=1+\tau_0,\qquad N_{B_1}(s,\kappa):=2s^{4}-3(\kappa+1)s^{3}+4\kappa s^{2}+\kappa(\kappa+1)s-2\kappa^{2},
\end{equation}
with $\tau_0$ given by \eqref{eq:tau0}. Note that 
\[
s(\sigma) = \frac{1 + \kappa + \kappa \sigma^2 + \sqrt{(1 + \kappa + \kappa \sigma^2)^2-4\kappa}}{2} \ge \frac{1 + \kappa + \sqrt{(1 + \kappa)^2-4\kappa}}{2} = \kappa.
\]
Now let us also define
\(
\Phi(s):=-N_{B_1}(s,\kappa)/\left(\kappa s^{4}\right)
\)
for $s\ge \kappa$.
A direct calculation gives the factorization
\[
\frac{\mathrm d}{\mathrm ds}\Phi(s)
=\frac{-N_{B_1}'(s)s+4N_{B_1}(s)}{\kappa s^{5}}
=\frac{-(s^{2}-\kappa)\left(3(\kappa+1)s-8\kappa\right)}{\kappa s^{5}}.
\]
For $s\ge \kappa$ we have $s^{2}-\kappa>0$, hence $\Phi'(s)$ changes sign only once at
\(
s_{*}:=\frac{8\kappa}{3(\kappa+1)}
\).
If $\kappa\ge 5/3$, then $s_{*}\le \kappa$ and $\Phi$ is strictly decreasing on $[\kappa,\infty)$.
If $1<\kappa<5/3$, then $\kappa<s_{*}$ and $\Phi$ is increasing on $[\kappa,s_{*})$ and strictly decreasing on $(s_{*},\infty)$.

Note that $B_1(\kappa,\sigma)=-N_{B_1}\left(s(\sigma),\kappa\right)/(\kappa s(\sigma)^{4})=\Phi\left(s(\sigma)\right),
\)
$s(\sigma)$ is strictly increasing in $\sigma$, and
\(
\Phi\left(s(\sigma)\right)\to -2/\kappa
\)
as $\sigma\to\infty$ (since $s(\sigma)\to \kappa\sigma^{2}$ and $N_{B_1}(s,\kappa)\to 2s^{4}$).
Furthermore, $s(0)=\kappa$ and
\[
N_{B_1}(\kappa,\kappa)=-\kappa^{2}(\kappa-1)^{2}<0
\implies
B_1(\kappa,0)=-\frac{N_{B_1}(\kappa,\kappa)}{\kappa^{5}}>0 .
\]
Therefore, $B_1(\kappa,\sigma)$ is strictly decreasing on $[0,\infty)$ if $\kappa\ge 5/3$, and for $1<\kappa<5/3$ it increases for small $\sigma$ and then strictly decreases; in either case, since $B_1(\kappa,0)>0$ and $\lim_{\sigma\to\infty}B_1(\kappa,\sigma)=-2/\kappa<0$, it crosses $0$ exactly once, which proves the existence and uniqueness of $\sigma_{B_1}(\kappa)$.
At the crossing $B_1(\kappa,\sigma_{B_1})=0$, hence $N_{B_1}\left(s(\sigma_{B_1}),\kappa\right)=0$.

Now
let $\varepsilon:=\kappa^{-1}$ and write $s=\kappa c$.
Dividing $N_{B_1}(\kappa c,\kappa)=0$ by $\kappa^{4}$ yields the analytic equation
\[
F(\varepsilon,c)=0,
\qquad
F(\varepsilon,c):=2c^{4}-3(1+\varepsilon)c^{3}+4\varepsilon c^{2}+(\varepsilon+\varepsilon^{2})c-2\varepsilon^{2}.
\]
At $\varepsilon=0$,
\(
F(0,c)=2c^{4}-3c^{3}
\)
has the positive root
\(
c_{0}=\tfrac32,
\)
and
\(
\partial_{c}F(0,c_{0})
=8c_{0}^{3}-9c_{0}^{2}
=\tfrac{27}{4}\neq 0.
\)
By the implicit function theorem there exists a unique analytic branch $c(\varepsilon)$ with $c(0)=\tfrac32$, having the expansion
\(
c(\varepsilon)=\frac{3}{2}+c_{1}\varepsilon+O(\varepsilon^{2}).
\)
Substituting into $F(\varepsilon,c)=0$ gives that, up to first order,
\[
\frac{27}{4}c_{1}+\frac{3}{8}=0
\implies
c_{1}=-\frac{1}{18}.
\]
Thus,
\[
\sigma_{B_1}^{2}
=\frac{(s_{c}-1)(s_{c}-\kappa)}{\kappa s_{c}}
=\left(1-\frac{1}{\kappa c(\varepsilon)}\right)\left(c(\varepsilon)-1\right),
\]
with
\(
s_{c}=\kappa c(\varepsilon)
=\frac{3}{2}\kappa-\frac{1}{18}+O(\kappa^{-1}).
\)
Substituting $c(\varepsilon)=\tfrac32-\tfrac{1}{18}\varepsilon+O(\varepsilon^{2})$ and expanding yields
\[
\sigma_{B_1}^{2}
=\frac{1}{2}
-\frac{7}{18}\kappa^{-1}
+O(\kappa^{-2}).
\]
The analyticity of $c(\varepsilon)$ implies the remainder $O(\varepsilon^{2})$ in $c$ and, consequently, the remainder $O(\kappa^{-2})$ in the displayed expansion.
\end{proof}

Next, we move to the proof of \eqref{eq:relations2}, which follows from the lemma below. 

\begin{lemma}\label{lemma:rel2}
Let $C_1(\sigma, \kappa)$ be given by \eqref{eq:lfor}.
Then, for every $\kappa\ge2$ and all $\sigma\ge0$,
\(
C_1(\kappa,\sigma)\le0.
\)
\end{lemma}

\begin{proof}
Let $s(\sigma) =1+\tau_{0}$, with $\tau_0$ given by \eqref{eq:tau0}, and note that \(\tau_0^2 / (\kappa s(\sigma)^6) > 0\).
Thus, the sign of $C_1$ is the opposite of the sign of $N_{C_1}(s(\sigma)-1,\kappa)$, where
\[
N_{C_1}(t,\kappa)=4t^{4}+(6-3\kappa)t^{3}-(6+3\kappa)t^{2}
+(\kappa^{2}+9\kappa-14)t-3\kappa^{2}+9\kappa-6.
\]
At $\sigma=0$, one has $s(0)-1=\kappa-1$, and a direct substitution yields
\[
N_{C_1}(\kappa-1,\kappa):=\kappa^{4}-3\kappa^{3}+2\kappa^{2}
=\kappa^{2}(\kappa-1)(\kappa-2).
\]
Moreover, differentiating in $t$ gives
\[
N_{C_1}''(t,\kappa)=6(-3\kappa t-\kappa+8t^{2}+6t-2)>0,
\]
for all $t\ge \kappa-1$ and $\kappa\ge2$, so $N_{C_1}'(\cdot,\kappa)$ is increasing on $[\kappa-1,\infty)$. As
\[
N_{C_1}'(\kappa-1,\kappa)=\kappa(\kappa-2)(7\kappa-3)\ge0,
\]
for $\kappa\ge2$, it follows that $N_{C_1}(\cdot,\kappa)$ is increasing on $[\kappa-1,\infty)$. Therefore, for every $\sigma\ge0$,
\[
N_{C_1}(s(\sigma)-1,\kappa)\ \ge\ N_{C_1}(\kappa-1,\kappa)=\kappa^{2}(\kappa-1)(\kappa-2)\ge0
\quad\text{for }\kappa\ge2.
\]
Since the prefactor is positive, $C_1(\kappa,\sigma)\le0$ for all $\sigma$ as soon as $\kappa\ge2$.
\end{proof}

\begin{lemma}\label{lemma:rel3}
Let $B_2(\sigma, \kappa)$ be given by \eqref{eq:Rfor}.
Then, for every $\kappa > 1$ and all $\sigma\ge0$,
\(
B_2(\kappa,\sigma)\le0.
\)
\end{lemma}

\begin{proof}
We have
\[
\begin{aligned}
B_2(\sigma,\kappa)
&=-\frac{2\tau_0}{(1+\tau_0)^3}
+\frac{2\kappa\tau_0^2}{(1+\tau_0)^3\bigl((1+\tau_0)^2-\kappa\bigr)}\\
&=\frac{2\tau_0}{(1+\tau_0)^3}
\left[
-1+\frac{\kappa\tau_0}{(1+\tau_0)^2-\kappa}
\right]\\
&=\frac{2\tau_0}{(1+\tau_0)^3}
\frac{(1+\tau_0)\bigl(\kappa-(1+\tau_0)\bigr)}{(1+\tau_0)^2-\kappa}\\
&=\frac{2\tau_0}{(1+\tau_0)^2}
\frac{\kappa-(1+\tau_0)}{(1+\tau_0)^2-\kappa}.
\end{aligned}
\]

Since \((1 + \tau_0)^2 > (1 + \tau_0) \ge \kappa\), we have $B_2(\sigma, \kappa)\le 0$.
\end{proof}

\begin{lemma}\label{lemma:rel4}
Let $C_2(\sigma, \kappa)$ be given by \eqref{eq:Rfor}.
Then, for every $\kappa > 1$ and all $\sigma\ge0$,
\(
C_2(\kappa,\sigma)\le0.
\)
\end{lemma}

\begin{proof}
Let us again write $s(\sigma)=1+\tau_{0}>0$ and note $\tau_{0}^{2}-1=(s(\sigma)-1)^{2}-1=s(\sigma)^{2}-2s(\sigma)$. Then, substituting and simplifying, we have
\[
C_2(\kappa,\sigma)
=-\frac{4\tau_{0}^{3}}{s(\sigma)^{5}}
+\frac{2\tau_{0}^{3}}{s(\sigma)^{6}}
\frac{\kappa(s(\sigma)^{2}-2s(\sigma))}{s(\sigma)^{2}-\kappa}
=\frac{2\tau_{0}^{3}}{s(\sigma)^{6}}\left(
\frac{\kappa s(\sigma)^{2}-2s(\sigma)^{3}}{s(\sigma)^{2}-\kappa}\right)
=\frac{2\tau_{0}^{3}}{s(\sigma)^{4}}
\frac{\kappa-2s(\sigma)}{s(\sigma)^{2}-\kappa}.
\]
Note that \(s(\sigma) > \kappa\), and therefore $s(\sigma)^{2}-\kappa>0$ and $\kappa-2s(\sigma)\le \kappa-2\kappa=-\kappa<0$. Because $\tau_{0}>0$ for $\kappa>1$, the prefactor $2\tau_{0}^{3}/s(\sigma)^{4}>0$. Therefore $C_2(\kappa,\sigma)\le0$.
\end{proof}

 Combining the results from Lemmas~\ref{lem:transition-kappa-3},~\ref{lemma:rel2}, ~\ref{lemma:rel3} and ~\ref{lemma:rel4}  concludes the proof of Theorem \ref{thm:equiv}.

\section{\revised{
Test risk evaluated on $\mathcal D(\theta)$ in the over-parameterized setting}}\label{app:test-bis}
\revised{Let $\overline{\mathcal R}_k(\Sigma, \theta_k, \thetapop)$ be the excess risk of the estimator $\theta_k$ given by \eqref{eq:thetak} evaluated on $\mathcal D(\theta)$, i.e.,
\[
    \overline{\mathcal R}_k(\Sigma, \theta_k, \thetapop) = \left\|\theta_k - (\thetapop+D\theta_{k-1}) \right\|_\Sigma^2. 
\]
}
\revised{\begin{theorem}[Excess risk on $\mathcal D(\theta)$-- over-parameterized]\label{thm:over2}
Let $R>0$ be a constant s.t.\ $\thetapop, \theta_0\in B_p(R)$. Assume that $\kappa, \sigma, \lambda\in (1/M, M)$ and $\|\Sigma\|_{\mathrm{op}},\ \|\Sigma^{-1}\|_{\mathrm{op}} \le M$ for some constant $M>1$. Then, there exists a constant $C=C\left(M, R\right)$ such that for any $\delta \in (0,1/2]$, with probability at least $1-Cpe^{-p\delta^{4}/C}$, %
\begin{equation}    
\left|\overline{\mathcal{R}}_2(\Sigma, \theta_{2}, \thetapop)-\ofixedriskeq\left(\Sigma, \thetapop, D, \lambda\right)\right|\le \delta+O(\|D\|_{\mathrm{op}}^2),
\end{equation}
where
\begin{equation}   \label{eq:defdet2} 
\begin{aligned}
&\ofixedriskeq\left(\Sigma, \thetapop, D, \lambda\right)
\\
&
\hspace{-.2em}=\hspace{-.2em}\frac{
\sigma^2 \kappa
\tr\left[\Sigma^{2}\left( \Sigma+\tau I_p \right)^{-2}\right]\hspace{-.2em}+\hspace{-.2em} p
\tau^{2}
\langle \thetapop,\hspace{-.2em}\left( \Sigma+\tau I_p \right)^{-1}\hspace{-.2em}
\left(\hspace{-.1em}I_p\hspace{-.1em}+\hspace{-.1em}2\left( \Sigma+\tau I_p \right)^{-1}
\Sigma D\hspace{-.1em}\right)\hspace{-.1em}\Sigma\left( \Sigma+\tau I_p \right)^{-1}\hspace{-.1em}\thetapop\rangle
}{
p - \kappa
\tr\left[\Sigma^{2}\left( \Sigma+\tau I_p \right)^{-2}\right]
},
\end{aligned}
\end{equation}
and $\tau$ is the unique solution of \eqref{eq:tau}.
\end{theorem}}

\revised{\begin{proof}
The argument is analogous to that used to prove Theorem \ref{thm:over}, and we only report differences. Using the same approach as Lemma \ref{lemma:1step}, we have
\begin{equation}\label{eq:det1-bis}  
\sup_{\thetapop, \theta_0\in B_p(R)}
\Pr\left(\left|\overline{\mathcal R}_{1}(\Sigma, \theta_{1}, \thetapop)-\ofixedriskeq^{(1)}\left(\Sigma, \theta_0, \thetapop\right)\right|\ge \delta\right)
\le Cpe^{-p\delta^{4}/C},
\end{equation}
where
\begin{equation} \label{eq:R1eq-bis} 
\begin{aligned}
\ofixedriskeq^{(1)}\left(\Sigma, \theta_0, \thetapop\right)
&= %
\frac{   \sigma^2\kappa \tr\left[ \Sigma^2 \left( \Sigma + \tau I_p \right)^{-2} \right] + p\tau^2 \left\| \left( \Sigma + \tau I_p \right)^{-1} (\thetapop+D\theta_0) \right\|_\Sigma^2  }{ p - \kappa \tr\left[ \Sigma^2 \left( \Sigma + \tau I_p \right)^{-2} \right] }.
\end{aligned}
\end{equation}
Next, using the same approach\footnote{In fact, the derivation is simpler since the term corresponding to $\fixedriskeqa^{(1)}\left(\Sigma, \theta_{1}, \thetapop\right)$ here is absent.} as Lemma \ref{lemma:2step}, we have
\begin{equation}\label{eq:det2-bis} 
\sup_{\thetapop, \theta_0\in B_p(R)}
\Pr\left(\left|\overline{\mathcal R}_{2}(\Sigma, \theta_{2}, \thetapop)-\ofixedriskeq^{(2)}\left(\Sigma, \theta_0, \thetapop\right)\right|\ge \delta\right)
\le Cpe^{-p\delta^{4}/C},
\end{equation}
where
\begin{equation}  \label{eq:fixedriskeq2-bis}  
\begin{aligned}
&\ofixedriskeq^{(2)}\left(\Sigma, \theta_0, \thetapop\right)
\\
&=  
\frac{ \sigma^2\kappa \tr\left[ \Sigma^2 \left( \Sigma + \tau I_p \right)^{-2} \right] + p\tau^2 \big\| \left( \Sigma + \tau I_p \right)^{-1} \left(\thetapop + D \left( \Sigma + \tau I_p \right)^{-1} \Sigma (\thetapop+D\theta_0) \right) \big\|_\Sigma^2 }{p - \kappa \tr\left[ \Sigma^2 \left( \Sigma + \tau I_p \right)^{-2} \right] } \\
&+ p\kappa \tau^2   
\tr\left[ \Sigma \left( \Sigma + \tau I_p \right)^{-2} D  \Sigma \left( \Sigma + \tau I_p \right)^{-2} D \right]
\cdot\frac{ \sigma^2 + \tau^2 \big\| \left( \Sigma + \tau I_p \right)^{-1} (\thetapop+D\theta_0) \big\|_\Sigma^2 }
{\big( p - \kappa \tr\left[ \Sigma^2 \left( \Sigma + \tau I_p \right)^{-2} \right] \big)^2 }.
\end{aligned}
\end{equation}
Noting that the quantity in the second line is $O(\|D\|_{\mathrm{op}}^2)$ and that
\begin{equation}
\begin{split}    
    \big\| &\left( \Sigma + \tau I_p \right)^{-1} \left(\thetapop + D \left( \Sigma + \tau I_p \right)^{-1} \Sigma (\thetapop+D\theta_0) \right) \big\|_\Sigma^2\\
    &=\langle \thetapop,\hspace{-.2em}\left( \Sigma+\tau I_p \right)^{-1}\hspace{-.2em}
\left(\hspace{-.1em}I_p\hspace{-.1em}+\hspace{-.1em}2\left( \Sigma+\tau I_p \right)^{-1}
\Sigma D\hspace{-.1em}\right)\hspace{-.1em}\Sigma\left( \Sigma+\tau I_p \right)^{-1}\hspace{-.1em}\thetapop\rangle
+O(\|D\|_{\mathrm{op}}^2)
\end{split}
\end{equation}
concludes the argument. 
\end{proof}}

\revised{\begin{lemma}\label{lemma:explicit-bis}
Consider the setting of Theorem \ref{thm:over2}, assume that $a$ has covariance $I_d/d$, and let $\Sigma=\begin{bmatrix}I_d&\rho I_d\\ \rho I_d&I_d\end{bmatrix}$. Then, we have that
\begin{equation*}
\begin{split}
\mathbb E_{\thetapop}\ofixedriskeq\left(\Sigma, \thetapop, D, \lambda\right)&=\overline{\mathcal R}(D, \lambda, \rho)+ O(\bar b\rho^2+\rho^4),\\
\overline{\mathcal R}(D, \lambda, \rho)&:=\mathcal R_0(\lambda, \rho) + \bar b \overline{A}_1(\lambda) + \bar c \rho^2 \overline{A}_2(\lambda),
\end{split}
\end{equation*}
where $\bar b=\tr[\di(b)]/d, \bar c=\tr[\di(c)]/d$, the auxiliary function $\mathcal R_0(\lambda, \rho)$ is given by \eqref{eq:explexpr}, and the new auxiliary functions $\overline{A}_1(\lambda)$ and $\overline{A}_2(\lambda)$ are given by
\begin{equation} \label{eq:explexpr-bis}
\begin{aligned}
\overline{A}_1(\lambda) &= \frac{2\tau^2}{(1+\tau)((1+\tau)^{2}-\kappa)}, \\
\overline{A}_2(\lambda) &= \frac{2\tau^3(\tau^{2}-1)}{(1+\tau)^{4}((1+\tau)^{2}-\kappa)}.
\end{aligned}
\end{equation}
\end{lemma}}

\revised{\begin{proof}
Using $\mathbb E_{\thetapop}\left[\langle \thetapop, M\thetapop\rangle\right]= \tr\left[(M)_1\right]/d$ from the proof of Lemma \ref{lemma:explicit}, we take the expectation of \eqref{eq:defdet2}:
\begin{equation} \label{eq:defdetexp-bis}
\begin{aligned}
&\mathbb E_{\thetapop}\ofixedriskeq\left(\Sigma, \thetapop, D, \lambda\right)
\\
&
=
\frac{
\sigma^2 \kappa
\tr\left[\Sigma^{2}\left( \Sigma+\tau I_p \right)^{-2}\right]
}{
p - \kappa
\tr\left[\Sigma^{2}\left( \Sigma+\tau I_p \right)^{-2}\right]
}
\\
&\quad
+
\frac{p \tau^2}{p - \kappa
\tr\left[\Sigma^{2}\left( \Sigma+\tau I_p \right)^{-2}\right]}
\frac{1}{d}
\tr\left[
\left(
\Sigma
\left( \Sigma+\tau I_p \right)^{-2}
\right)_{1}
\right]
\\
&\quad
+
\frac{p \tau^2}{p - \kappa
\tr\left[\Sigma^{2}\left( \Sigma+\tau I_p \right)^{-2}\right]}
\frac{2}{d}
\tr\left[
\left(
\left( \Sigma+\tau I_p \right)^{-2}
\Sigma D
\left( \Sigma+\tau I_p \right)^{-1}
\Sigma
\right)_{1}
\right].
\end{aligned}
\end{equation}
The first two terms correspond to the risk with $D=0$ (i.e., $\bar b = \bar c = 0$). By the same computations as in the proof of Lemma \ref{lemma:explicit}, these terms combine to $\mathcal R_0(\lambda, \rho) + O(\rho^4)$.
The third term, which depends on $D$, requires approximations for its two factors. The first factor is new:
\begin{equation*}
\begin{split}
\frac{p}{p - \kappa \tr\left[\Sigma^{2}(\Sigma+\tau I_p)^{-2}\right]}
&= 1 + \frac{\kappa\operatorname{tr}\left[\Sigma^{2}(\Sigma+\tau I_p)^{-2}\right]}{p-\kappa\operatorname{tr}\left[\Sigma^{2}(\Sigma+\tau I_p)^{-2}\right]} \\
&= 1 + \left( \frac{\kappa}{(1+\tau)^{2}-\kappa} + O(\rho^{2}) \right)
= \frac{(1+\tau)^2}{(1+\tau)^{2}-\kappa} + O(\rho^{2}).
\end{split}
\end{equation*}
For the second factor, we use the trace expansion from Lemma \ref{lemma:explicit}:
\begin{equation*}
\frac{1}{d}\tr\left[\left((\Sigma+\tau I_p)^{-2}\Sigma D(\Sigma+\tau I_p)^{-1}\Sigma\right)_1\right]
=\frac{\bar b}{(1+\tau)^{3}}
+ \rho^{2}\left(\frac{1-3\tau}{(1+\tau)^{5}}\bar b+\frac{\tau(\tau^{2}-1)}{(1+\tau)^{6}}\bar c\right)
+O(\rho^4).
\end{equation*}
We multiply these two factors by $2\tau^2$ (from \eqref{eq:defdetexp-bis}) and keep only the terms of order $O(\bar b)$ and $O(\bar c \rho^2)$:
\begin{align*}
\bar b \overline{A}_1(\lambda) &= \left( \frac{(1+\tau)^2}{(1+\tau)^{2}-\kappa} \right) \left( 2\tau^2 \frac{\bar b}{(1+\tau)^{3}} \right) = \bar b \frac{2\tau^2}{(1+\tau)((1+\tau)^{2}-\kappa)}, \\
\bar c \rho^2 \overline{A}_2(\lambda) &= \left( \frac{(1+\tau)^2}{(1+\tau)^{2}-\kappa} \right) \left( 2\tau^2 \rho^2 \bar c \frac{\tau(\tau^2-1)}{(1+\tau)^6} \right) = \bar c \rho^2 \frac{2\tau^3(\tau^{2}-1)}{(1+\tau)^{4}((1+\tau)^{2}-\kappa)}.
\end{align*}
Adding these terms to $\mathcal R_0(\lambda, \rho)$ yields the claimed expansion.
\end{proof}}

\revised{\begin{lemma}\label{lemma:taustar-bis}
In the setting of Lemma \ref{lemma:explicit-bis}, we have that
\begin{align*}
\tau^*(D, \rho)
&= \tau^{*}_{0}(\rho) + \bar b \left( \overline{B}_3(\sigma, \kappa) + O(\rho^2)\right) + \bar c\left( \rho^2 \overline{C}_3(\sigma, \kappa) + O(\rho^4) \right) + O(\bar b^2+\bar c^2),
\end{align*}
where $\tau_{0}$ and $\tau^{*}_{0}(\rho)$ are given by \eqref{eq:tau0}, and
\begin{equation}\label{eq:tau-coeffs-bis}
\begin{aligned}
\overline{B}_3(\sigma, \kappa) &= - 
\frac{\tau_0 \left( (1+\tau_0)^2 (2-\tau_0) - \kappa (\tau_0+2) \right)}{(1+\tau_0) \left( (1+\tau_0)^2 - \kappa \right)}, \\
\overline{C}_3(\sigma, \kappa) &= - 
\frac{\tau_0^2 \left( (4\tau_0-3)(1+\tau_0)((1+\tau_0)^2-\kappa) - \tau_0(\tau_0-1)(5(1+\tau_0)^2-3\kappa) \right)}{(1+\tau_0)^3 \left( (1+\tau_0)^2 - \kappa \right)}.
\end{aligned}
\end{equation}
\end{lemma}}

\revised{\begin{proof}
A direct differentiation of $\overline{\mathcal R}(D,\lambda,\rho)$ from Lemma \ref{lemma:explicit-bis} gives
\[
\begin{aligned}
\frac{\mathrm{d}}{\mathrm{d}\tau}\overline{\mathcal R}(D,\lambda,\rho)
&= \frac{\mathrm{d}}{\mathrm{d}\tau}\mathcal R_0(\lambda, \rho)
+\bar b \frac{\mathrm{d}}{\mathrm{d}\tau}\overline{A}_1(\lambda)
+\bar c \rho^2 \frac{\mathrm{d}}{\mathrm{d}\tau}\overline{A}_2(\lambda).
\end{aligned}
\]
The first term is identical to that in the proof of Lemma \ref{lemma:taustar}. The new derivatives are:
\[
\begin{aligned}
\frac{\mathrm{d}}{\mathrm{d}\tau}\overline{A}_1(\lambda) &=
\frac{2\tau(1+\tau)^2(2-\tau) - 2\kappa\tau(2+\tau)}{(1+\tau)^2((1+\tau)^2-\kappa)^2}, \\
\frac{\mathrm{d}}{\mathrm{d}\tau}\overline{A}_2(\lambda) &=
\frac{2\tau^2 \left( (4\tau-3)(1+\tau)((1+\tau)^2-\kappa) - \tau(\tau-1)(5(1+\tau)^2-3\kappa) \right)}{(1+\tau)^4((1+\tau)^2-\kappa)^2}.
\end{aligned}
\]
With these explicit derivatives, the stationarity equation $\frac{\mathrm{d}}{\mathrm{d}\tau}\overline{\mathcal R}(D, \lambda, \rho)=0$ is equivalent to
\(
\frac{\overline{F}\left(\tau,\bar b,\bar c,\rho^{2}\right)}{(1 + \tau)^7 ((1+\tau)^2 - \kappa)^3} =0,
\)
where
\[
\overline{F}\left(\tau,\bar b,\bar c,\rho^{2}\right)
= F_{0}(\tau)
+ \rho^2 F_{\rho}(\tau) + \bar b \overline{F}_{b}(\tau)
+ \bar c \rho^2 \overline{F}_{\rho c}(\tau).
\]
The functions $F_{0}(\tau)$ and $F_{\rho}(\tau)$ are identical to those defined in the proof of Lemma \ref{lemma:taustar}. The new functions are
\[
\overline{F}_{b}(\tau)
= 2\tau(1+\tau)^5 ((1+\tau)^2-\kappa) \left( (1+\tau)^2(2-\tau) - \kappa(2+\tau) \right),
\]
\[
\overline{F}_{\rho c}(\tau)
= 2\tau^2(1+\tau)^3 ((1+\tau)^2-\kappa) \left( (4\tau-3)(1+\tau)((1+\tau)^2-\kappa) - \tau(\tau-1)(5(1+\tau)^2-3\kappa) \right).
\]
Setting $\bar b=\bar c=\rho^{2}=0$ yields the same equation for $\tau_{0}$ as in Lemma \ref{lemma:taustar}. The partial derivative $\partial_{\tau}\overline{F}\left(\tau_{0},0,0,0\right)$ is also unchanged:
\[
\partial_{\tau}\overline{F}\left(\tau_{0},0,0,0\right))
= -2(1+\tau_0)^7\left((1+\tau_0)^2-\kappa\right)\sqrt{(1+\kappa+\kappa\sigma^2)^2-4\kappa}\neq 0.
\]
Therefore, the implicit function theorem gives a smooth map
\(
\tau^{*}(\bar b,\bar c,\rho^{2})
\)
with $\tau^{*}(0,0,0)=\tau_{0}$ and $\overline{F}(\tau^{*},\cdot)=0$.
Differentiating $\overline{F}=0$ at $(\tau_{0},0,0,0)$ and dividing by $\partial_{\tau}\overline{F}(\tau_{0},0,0,0)$ yields
\[
\overline{B}_3(\sigma, \kappa) = - \frac{\overline{F_b}(\tau_0)}{\partial_{\tau}F(\tau_{0},0,0,0)}, \qquad
\overline{C}_3(\sigma, \kappa) = - \frac{\overline{F}_{\rho c}(\tau_0)}{\partial_{\tau}F(\tau_{0},0,0,0)}.
\]
Substituting the expressions for $\overline{F_b}(\tau_0)$, $\overline{F}_{\rho c}(\tau_0)$, and $\partial_{\tau}F(\tau_{0},0,0,0)$ and cancelling common factors gives the coefficients as stated in \eqref{eq:tau-coeffs-bis}.
\end{proof}}

\revised{As $\lambda$ and $\tau$ are linked by the fixed point equation \eqref{eq:tau}, an application of Lemma \ref{lemma:taustar-bis} readily gives that
\begin{equation}\label{eq:lambdastar-bis}
\begin{split}
\lambdaeqs(D, \rho)&= \lambdaeqsz(\rho)+\bar b (\overline{B}_1(\sigma, \kappa)+O(\rho^2)) +\bar c \rho^2( \overline{C}_1(\sigma, \kappa)+O(\rho^2))+O(\bar b^2+\bar c^2),
\end{split}
\end{equation}
where $\lambdaeqsz(\rho)$ is unchanged from \eqref{eq:lfor}, and the new coefficients are
\begin{equation}\label{eq:lfor-bis}
\begin{aligned}
\overline{B}_1(\sigma, \kappa) &= -
\frac{\tau_0 \left( (1+\tau_0)^2 (2-\tau_0) - \kappa (\tau_0+2) \right)}{\kappa (1+\tau_0)^3}, \\
\overline{C}_1(\sigma, \kappa) &= -
\frac{\tau_0^2 \left( (4\tau_0-3)(1+\tau_0)((1+\tau_0)^2-\kappa) - \tau_0(\tau_0-1)(5(1+\tau_0)^2-3\kappa) \right)}{\kappa (1+\tau_0)^5}.
\end{aligned}
\end{equation}}

\revised{Next, the corollary below provides the expansion for the optimal equilibrium risk.}

\revised{\begin{corollary}\label{cor:risk-bis}
Consider the setting of Lemma \ref{lemma:explicit-bis} and let $\tau_0$ be given by \eqref{eq:tau0}. Then, we have that
\begin{equation}
\overline{\fixedriskeqs}(D, \rho)= \fixedriskeqs(\rho)+\bar b (\overline{B}_2(\sigma, \kappa)+O(\rho^2))+\bar c \rho^2( \overline{C}_2(\sigma, \kappa)+O(\rho^2))+O(\bar b^2+\bar c^2),
\end{equation}
where $\fixedriskeqs(\rho)$ is given by \eqref{eq:Rfor}, and
\begin{equation}\label{eq:Rfor-bis}
\begin{aligned}
\overline{B}_2(\sigma, \kappa) &= \frac{2\tau_0^2}{(1+\tau_0)((1+\tau_0)^{2}-\kappa)}, \\
\overline{C}_2(\sigma, \kappa) &= \frac{2\tau_0^3(\tau_0^{2}-1)}{(1+\tau_0)^{4}((1+\tau_0)^{2}-\kappa)}.
\end{aligned}
\end{equation}
\end{corollary}}

\revised{\begin{lemma}\label{lemma:B1-bis}
Let $\overline{B}_1(\sigma,\kappa)$ be given by \eqref{eq:lfor-bis}. Then, for any $\kappa \ge 2$ and all $\sigma \ge 0$,
\(
\overline{B}_1(\sigma,\kappa) \ge 0.
\)
\end{lemma}}

\revised{\begin{proof} Let $s(\sigma) = 1+\tau_0$.
By the definition \eqref{eq:lfor-bis}, we can write
\[
\overline{B}_1(\sigma,\kappa)
= - \frac{\tau_0 N_{\overline{B}_1}(s(\sigma),\kappa)}{\kappa s(\sigma)^3},
\qquad
N_{\overline{B}_1}(s,\kappa)
:= s^2(3-s) - \kappa(s+1).
\]
From the construction we have $s(\sigma) \ge \kappa$ and here we assume $\kappa \ge 2$, so in particular $s(\sigma) \ge 2$. Since $\tau_0>0$, $\kappa>0$ and $s(\sigma)>0$, the prefactor
\(
- \frac{\tau_0}{\kappa s(\sigma)^3} < 0.
\)
Therefore, to prove $\overline{B}_1(\sigma,\kappa) > 0$ it suffices to show
\[
N_{\overline{B}_1}(s,\kappa) < 0 \qquad \text{for all } s \ge \kappa \ge 2.
\]
We analyse $N_{\overline{B}_1}$ as a function of $s$ (with $\kappa$ fixed). Its derivatives are
\[
N_{\overline{B}_1}'(s,\kappa) = 6s - 3s^2 - \kappa,
\qquad
N_{\overline{B}_1}''(s,\kappa) = 6 - 6s.
\]
For $s \ge 2$ we have $N_{\overline{B}_1}''(s,\kappa) < 0$, so $N_{\overline{B}_1}$ is concave on $[2,\infty)$, and hence on $[\kappa,\infty)$ since $\kappa \ge 2$.
We first evaluate $N_{\overline{B}_1}$ and its derivative at the boundary point $s=\kappa$:
\[
N_{\overline{B}_1}(\kappa,\kappa)
= \kappa^2(3-\kappa) - \kappa(\kappa+1)
= -\kappa(\kappa-1)^2 < 0,
\]
and
\[
N_{\overline{B}_1}'(\kappa,\kappa)
= 6\kappa - 3\kappa^2 - \kappa
= \kappa(5-3\kappa).
\]
For $\kappa \ge 2$ we have $5-3\kappa < 0$, so
\(
N_{\overline{B}_1}'(\kappa,\kappa) \le 0.
\)
Since $N_{\overline{B}_1}$ is concave on $[\kappa,\infty)$, its derivative $N_{\overline{B}_1}'(s,\kappa)$ is non-increasing in $s$ on this interval. Hence, for all $s \ge \kappa$,
\(
N_{\overline{B}_1}'(s,\kappa) \le N_{\overline{B}_1}'(\kappa,\kappa) \le 0,
\)
so $N_{\overline{B}_1}(\cdot,\kappa)$ is non-increasing on $[\kappa,\infty)$. Together with $N_{\overline{B}_1}(\kappa,\kappa)<0$ this implies
\[
N_{\overline{B}_1}(s,\kappa) \le N_{\overline{B}_1}(\kappa,\kappa) \le 0
\qquad \text{for all } s \ge \kappa \ge 2,
\]
which proves the lemma.
\end{proof}}

\revised{\begin{lemma}\label{lemma:B2-bis}
Let $ \overline{B}_2(\sigma, \kappa)$ be given by \eqref{eq:Rfor-bis}.
Then, for every $\kappa > 1$ and all $\sigma\ge0$,
\(
 \overline{B}_2(\kappa,\sigma) \ge 0.
\)
\end{lemma}}
\revised{\begin{proof}
Recall the definition 
\[ \overline{B}_2(\sigma, \kappa) = \frac{2\tau_0^2}{(1+\tau_0)((1+\tau_0)^{2}-\kappa)}.\]
Both the numerator and the denominator are positive for \(\kappa > 1\) and \(\sigma > 0\) (as shown in the previous lemmas), therefore the claim holds.
\end{proof}}

\revised{\begin{lemma}\label{lemma:C2-bis}
Let $ \overline{C}_2(\sigma, \kappa)$ be given by \eqref{eq:Rfor-bis}.
Then, for every $\kappa \ge 2$ and all $\sigma\ge0$,
\(
 \overline{C}_2(\kappa,\sigma) \ge 0.
\)
\end{lemma}}
\revised{\begin{proof}
Recall the definition \[ \overline{C}_2(\sigma, \kappa) = \frac{2\tau_0^3(\tau_0^{2}-1)}{(1+\tau_0)^{4}((1+\tau_0)^{2}-\kappa)}.\]
Let $s(\sigma) = 1+\tau_0$.  The denominator is strictly positive for $\kappa > 1$. The term $2\tau_0^3$ is also strictly positive.
Thus, the sign of $ \overline{C}_2$ is determined by the sign of $(\tau_0^2 - 1)$.
We rewrite this term as:
\[
\tau_0^2 - 1 = (s(\sigma)-1)^2 - 1 = s(\sigma)^2 - 2s(\sigma) = s(\sigma)(s(\sigma)-2).
\]
Since $s(\sigma) \ge \kappa > 1$, $s(\sigma)$ is positive. The sign is therefore determined by $(s(\sigma) - 2)$.
We are given $\kappa \ge 2$, which gets
\(
s(\sigma) \ge  \kappa \ge 2.
\)
Therefore, $s(\sigma) - 2 \ge 0$ for all $\sigma \ge 0$ and the claim holds.
\end{proof}}

\section{Details for the experimental setup}
\label{app:real}

For both datasets, the curves are obtained by running $100$ equally spaced values of $\lambda$ with the same splits, so that the observations focus on the performative effect. Data is split uniformly at random across the different steps.

\paragraph{Housing.} We keep all features of the dataset and normalize them. We center the target feature since we use a linear regression without intercept. Following \citet{NEURIPS2024_7de66547}, we fix the features affected by the performative effect to be \texttt{MedInc}, \texttt{AveBedrms}, and \texttt{AveOccup}, with all values of $b$ set equal. 

\paragraph{LSAC.} We keep only one feature in cases of redundant encoding, drop features that are too strongly correlated with the target \texttt{GPA} ($\rho > 0.6$), and randomly select roughly half of the features to be affected by the performative effect. The names of these features are reported in \Cref{tab:lsac_features}. We normalize all features and center the target. All coefficients of $b$ are equal. 

\paragraph{Empirical Covariance.} In \Cref{fig:empiricalcovarianceLSAC,fig:empcovHousing}, we observe that the features do not follow the assumptions made on the data matrix $X$ in the theoretical part, despite exhibiting similar behavior in the experiments. This illustrates that our findings on how to scale regularization remain useful for more general datasets.

\begin{table}[h]
\centering
\caption{Features of the LSAC dataset}
\label{tab:lsac_features}
\resizebox{0.99\textwidth}{!}{%
\begin{tabular}{ll}
\toprule
\textbf{Category} & \textbf{Feature name} \\
\midrule
Redundant & \texttt{male} (same as \texttt{sex}), \texttt{parttime} (same as \texttt{fulltime}), \texttt{decile1} (same as \texttt{decile1b}) \\
With $\rho>0.6$ & \texttt{ugpa}, \texttt{index6040}, \texttt{dnn bar pass prediction} \\
With  $b_{\text{feat}} = \bar{b}$ & \texttt{Unnamed0}, \texttt{decile1b}, \texttt{decile3}, \texttt{other}, \texttt{asian}, \texttt{black}, \texttt{hisp}, \texttt{pass bar}, \texttt{tier} \\
\bottomrule
\end{tabular}}
\end{table}

\begin{figure}
\centering
    \begin{subfigure}{0.45\textwidth}
        \includegraphics[width=\linewidth]{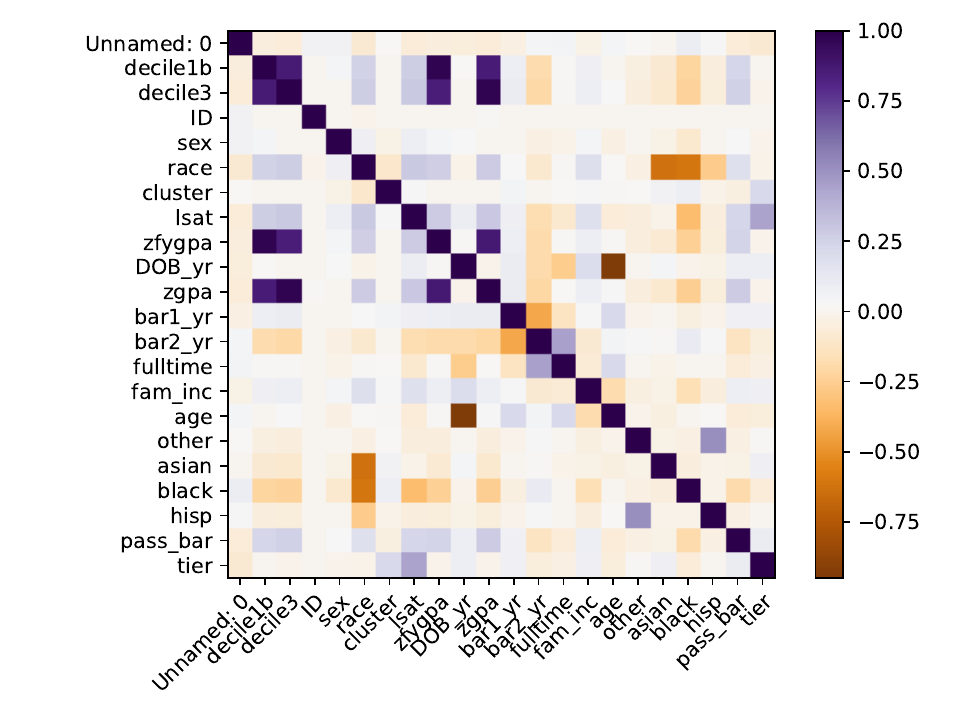}
    \caption{\revised{Empirical covariance of LSAC dataset.}}
    \label{fig:empiricalcovarianceLSAC}
    \end{subfigure}
    \hfill
    \begin{subfigure}{0.45\textwidth}
        \includegraphics[width=\linewidth]{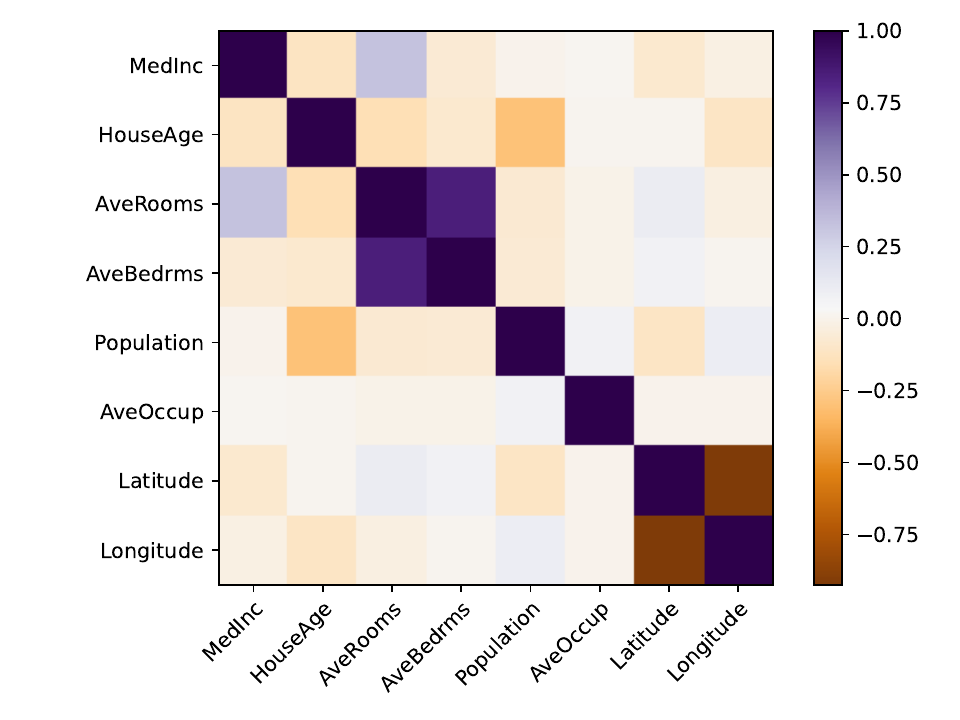}
        \caption{Empirical covariance of Housing dataset.}
        \label{fig:empcovHousing}
    \end{subfigure}
\end{figure}

%% file: example_paper.bib
@string{NeurIPS = {Conference on Neural Information Processing Systems (NeurIPS)}}

@string{ICML = {International Conference on Machine Learing (ICML)}}

@string{AISTATS = {Conference on Uncertainty in Artificial Intelligence (AISTATS)}}

@string{AAAI = {Conference on Artificial Intelligence (AAAI)}}

@string{ICLR = {International Conference on Learning Representations (ICLR)}}

@string{FACCT = {Conference on Fairness, Accountability and Transparency (FAccT)}}

@inproceedings{ildizhigh,
    author = {Ildiz, Muhammed Emrullah and Gozeten, Halil Alperen and Taga, Ege Onur and Mondelli, Marco and Oymak, Samet},
    booktitle = ICLR,
    title = {High-dimensional Analysis of Knowledge Distillation: Weak-to-Strong Generalization and Scaling Laws},
    year = {2025}
}

@article{hastie2022surprises,
    author = {Hastie, Trevor and Montanari, Andrea and Rosset, Saharon and Tibshirani, Ryan J},
    ignorepages = {949},
    journal = {Annals of statistics},
    number = {2},
    title = {Surprises in high-dimensional ridgeless least squares interpolation},
    volume = {50},
    year = {2022}
}

@inproceedings{kolossovtowards,
    author = {Kolossov, Germain and Montanari, Andrea and Tandon, Pulkit},
    booktitle = ICLR,
    title = {Towards a statistical theory of data selection under weak supervision},
    year = {2024}
}

@inproceedings{patil2024optimal,
    author = {Patil, Pratik and Du, Jin-Hong and Tibshirani, Ryan J},
    booktitle = ICML,
    ignorepages = {39908--39954},
    title = {Optimal ridge regularization for out-of-distribution prediction},
volume =235,
    year = {2024}
}

@article{song2024generalization,
    author = {Song, Yanke and Bhattacharya, Sohom and Sur, Pragya},
    journal = {arXiv preprint arXiv:2406.13944},
    title = {Generalization error of min-norm interpolators in transfer learning},
    year = {2024}
}

@inproceedings{mallinar2024minimumnorm,
    author = {Neil Rohit Mallinar and Austin Zane and Spencer Frei and Bin Yu},
    booktitle = ICML,
    title = {Minimum-Norm Interpolation Under Covariate Shift},
       volume = {235},
 year = {2024}
}

@inproceedings{richards2021asymptotics,
    author = {Richards, Dominic and Mourtada, Jaouad and Rosasco, Lorenzo},
    booktitle = AISTATS,
volume = 130,
    ignorepages = {3889--3897},
    title = {Asymptotics of ridge (less) regression under general source condition},
    year = {2021}
}

@article{rezaei2025high,
  title={High-dimensional Analysis of Synthetic Data Selection},
  author={Rezaei, Parham and Kovacevic, Filip and Locatello, Francesco and Mondelli, Marco},
  journal={arXiv preprint arXiv:2510.08123},
  year={2025}
}

@inproceedings{wu2020optimal,
    author = {Wu, Denny and Xu, Ji},
    booktitle = NeurIPS,
    ignorepages = {10112--10123},
    title = {On the optimal weighted $\ell_2$ regularization in overparameterized linear regression},
    volume = {33},
    year = {2020}
}

@article{montanari2019generalization,
    archiveprefix = {arXiv},
    author = {Andrea Montanari and Feng Ruan and Youngtak Sohn and Jun Yan},
    title = {The generalization error of max-margin linear classifiers: Benign overfitting and high dimensional asymptotics in the overparametrized regime},
    journal = {Annals of statistics},
    year = {2025},
volume = 53,
number =2
}

@article{yang2020precise,
    author = {Yang, Fan and Zhang, Hongyang R and Wu, Sen and Re, Christopher and Su, Weijie J},
    ignorepages = {1--88},
    journal = {Journal of Machine Learning Research},
    number = {113},
    title = {Precise High-Dimensional Asymptotics for Quantifying Heterogeneous Transfers},
    volume = {26},
    year = {2025}
}

@article{tsigler2023benign,
    author = {Tsigler, Alexander and Bartlett, Peter L},
    ignorepages = {1--76},
    journal = {Journal of Machine Learning Research},
    number = {123},
    title = {Benign overfitting in ridge regression},
    volume = {24},
    year = {2023}
}

@article{deng2022model,
    author = {Deng, Zeyu and Kammoun, Abla and Thrampoulidis, Christos},
    ignorepages = {435--495},
    journal = {Information and Inference: A Journal of the IMA},
    number = {2},
    publisher = {Oxford University Press},
    title = {A model of double descent for high-dimensional binary linear classification},
    volume = {11},
    year = {2022}
}

@inproceedings{jain2024scaling,
    author = {Jain, Ayush and Montanari, Andrea and Sasoglu, Eren},
    booktitle = NeurIPS,
    ignorepages = {110246--110289},
    title = {Scaling laws for learning with real and surrogate data},
    volume = {37},
    year = {2024}
}

@inproceedings{chang2021provable,
    author = {Chang, Xiangyu and Li, Yingcong and Oymak, Samet and Thrampoulidis, Christos},
    booktitle = AAAI,
    ignorepages = {6974--6983},
    title = {Provable benefits of overparameterization in model compression: From double descent to pruning neural networks},
    volume = {35},
    year = {2021}
}

@article{Bartlett_2020,
    author = {Bartlett, Peter L. and Long, Philip M. and Lugosi, Gábor and Tsigler, Alexander},
    ignorepages = {30063–30070},
    journal = {Proceedings of the National Academy of Sciences},
    number = {48},
    publisher = {Proceedings of the National Academy of Sciences},
    title = {Benign overfitting in linear regression},
    volume = {117},
    year = {2020}
}

@article{han2023distribution,
    author = {Han, Qiyang and Xu, Xiaocong},
    journal = {arXiv preprint arXiv:2307.02044},
    title = {The distribution of ridgeless least squares interpolators},
    year = {2023}
}

@article{hardt2023performative,
    author = {Moritz Hardt and
Celestine Mendler{-}D{\"{u}}nner},
journal = {arXiv preprint arXiv:2310.16608},
    title = {Performative Prediction: Past and Future},
    year = {2023}
}

@inproceedings{izzo2021learn,
    author = {Izzo, Zachary and Zou, James and Ying, Lexing},
    booktitle = AISTATS,
    no-editor = {Camps-Valls, Gustau and Ruiz, Francisco J. R. and Valera, Isabel},
    no-month = {28--30 Mar},
    no-pages = {3998--4035},
    no-publisher = {PMLR},
    no-series = {Proceedings of Machine Learning Research},
    no-url = {https://proceedings.mlr.press/v151/izzo22a.html},
    no-volume = {151},
    title = { How to Learn when Data Gradually Reacts to Your Model },
        volume = {151},
year = {2022}
}

@inproceedings{miller_outside,
    author = {John Miller and Juan C. Perdomo and Tijana Zrnic},
    booktitle = ICML,
    no-url = {https://api.semanticscholar.org/CorpusID:231942295},
    title = {Outside the Echo Chamber: Optimizing the Performative Risk},
volume ={139},
    year = {2021}
}

@inproceedings{perdomo_performative,
    author = {Perdomo, Juan and Zrnic, Tijana and Mendler-D{\"u}nner, Celestine and Hardt, Moritz},
    booktitle = ICML,
    no-editor = {III, Hal Daumé and Singh, Aarti},
    no-month = {13--18 Jul},
    no-pages = {7599--7609},
    no-publisher = {PMLR},
    no-series = {Proceedings of Machine Learning Research},
    no-url = {https://proceedings.mlr.press/v119/perdomo20a.html},
    volume = {119},
    title = {Performative Prediction},
    year = {2020}
}

@inproceedings{ribeiro2023regularization,
    author = {Antonio H. Ribeiro and Dave Zachariah and Francis Bach and Thomas B. Sch{\"o}n},
    booktitle = NeurIPS,
    no-url = {https://openreview.net/pdf?id=K8gLHZIgVW},
  volume={36},
    title = {Regularization properties of adversarially-trained linear regression},
    year = {2023}
}

@article{zezulka_performativity_2023,
    author = {Zezulka, Sebastian and Genin, Konstantin},
    no-doi = {10.48550/arXiv.2310.08349},
    no-url = {http://arxiv.org/abs/2310.08349},
    journal = {arXiv preprint arXiv:2310.08349},
    publisher = {arXiv},
    title = {Performativity and {Prospective} {Fairness}},
    urldate = {2023-11-22},
    year = {2023}
}

@inproceedings{brown2022performative,
    author = {Brown, Gavin and Hod, Shlomi and Kalemaj, Iden},
    booktitle = AISTATS,
    no-pages = {6045--6061},
    organization = {PMLR},
    title = {Performative prediction in a stateful world},
    year = {2022}
}

@inproceedings{pmlr-v130-bechavod21a,
    author = {Bechavod, Yahav and Ligett, Katrina and Wu, Steven and Ziani, Juba},
    booktitle = AISTATS,
    no-pages = {1234--1242},
    no-url = {https://proceedings.mlr.press/v130/bechavod21a.html},
    title = { Gaming Helps! Learning from Strategic Interactions in Natural Dynamics },
    volume = {130},
    year = {2021}
}

@article{JMLR:v24:22-0131,
    author = {Adhyyan Narang and Evan Faulkner and Dmitriy Drusvyatskiy and Maryam Fazel and Lillian J. Ratliff},
    ignorepages = {1--56},
    journal = {Journal of Machine Learning Research},
    number = {202},
    title = {Multiplayer Performative Prediction: Learning in Decision-Dependent Games},
    no-url = {http://jmlr.org/papers/v24/22-0131.html},
    volume = {24},
    year = {2023}
}

@inproceedings{pmlr-v202-wang23ap,
    author = {Wang, Xiaolu and Yau, Chung-Yiu and Wai, Hoi To},
    booktitle = ICML,
    ignorepages = {36514--36540},
    title = {Network Effects in Performative Prediction Games},
    no-url = {https://proceedings.mlr.press/v202/wang23ap.html},
    volume = {202},
    year = {2023}
}

@inproceedings{pmlr-v81-ensign18a,
    author = {Ensign, Danielle and Friedler, Sorelle A. and Neville, Scott and Scheidegger, Carlos and Venkatasubramanian, Suresh},
    booktitle = FACCT,
    ignorepages = {160--171},
    title = {Runaway Feedback Loops in Predictive Policing},
    no-url = {https://proceedings.mlr.press/v81/ensign18a.html},
    volume = {81},
    year = {2018}
}

@book{morgenstern1928wirtschaftsprognose,
    author = {Morgenstern, Oskar},
    isbn = {978-3709121139},
    publisher = {Springer},
    title = {Wirtschaftsprognose: Eine Untersuchung ihrer Voraussetzungen und Möglichkeiten},
    year = {1928}
}

@inproceedings{pmlr-v202-taori23a,
    author = {Taori, Rohan and Hashimoto, Tatsunori},
    booktitle = ICML,
    ignorepages = {33883--33920},
    title = {Data Feedback Loops: Model-driven Amplification of Dataset Biases},
    no-url = {https://proceedings.mlr.press/v202/taori23a.html},
    volume = {202},
    year = {2023}
}

@article{Ursu2015,
    author = {Ursu,  Raluca Mihaela},
    journal = {SSRN Electronic Journal},
    publisher = {Elsevier BV},
    title = {The Power of Rankings: Quantifying the Effects of Rankings on Online Consumer Search and Choice},
    no-url = {http://dx.doi.org/10.2139/ssrn.2729325},
    year = {2015}
}

@inproceedings{pmlr-v235-pan24d,
    author = {Pan, Alexander and Jones, Erik and Jagadeesan, Meena and Steinhardt, Jacob},
    booktitle = ICML,
    ignorepages = {39154--39200},
    title = {Feedback Loops With Language Models Drive In-Context Reward Hacking},
    no-url = {https://proceedings.mlr.press/v235/pan24d.html},
    volume = {235},
    year = {2024}
}

@inproceedings{NEURIPS2024_7de66547,
    author = {Cyffers, Edwige and Pydi, Muni Sreenivas and Atif, Jamal and Capp\'{e}, Oliver},
    booktitle = NeurIPS,
    ignorepages = {68144--68160},
    title = {Optimal Classification under Performative Distribution Shift},
    no-url = {https://proceedings.neurips.cc/paper_files/paper/2024/file/7de665476d0adc8a54d3b8744f932bbf-Paper-Conference.pdf},
    volume = {37},
    year = {2024}
}

@inproceedings{bombari2025spuriouscorrelationshighdimensional,
    author = {Simone Bombari and Marco Mondelli},
    booktitle = ICML,
    title = {Spurious Correlations in High Dimensional Regression: The Roles of Regularization, Simplicity Bias and Over-Parameterization},
    no-url = {https://arxiv.org/abs/2502.01347},
    year = {2025}
}

@article{fawzi2016analysisclassifiersrobustnessadversarial,
  title={Analysis of classifiers’ robustness to adversarial perturbations},
  author={Fawzi, Alhussein and Fawzi, Omar and Frossard, Pascal},
  journal={Machine learning},
  volume={107},
  number={3},
  pages={481--508},
  year={2018},
  publisher={Springer}
}

@inproceedings{li2022,
    author = {Li, Qiang and Yau, Chung-Yiu and Wai, Hoi-To},
    title = {Multi-agent Performative Prediction with Greedy Deployment and Consensus Seeking Agents},
    year = {2022},
booktitle = NeurIPS,
volume={35}
}

@article{Drusvyatskiy2023,
    author = {Drusvyatskiy,  Dmitriy and Xiao,  Lin},
    ignorepages = {954–998},
    journal = {Mathematics of Operations Research},
    month = {May},
    number = {2},
    publisher = {Institute for Operations Research and the Management Sciences (INFORMS)},
    title = {Stochastic Optimization with Decision-Dependent Distributions},
    no-url = {http://dx.doi.org/10.1287/moor.2022.1287},
    volume = {48},
    year = {2023}
}

@inproceedings{pmlr-v235-ben-dov24a,
    author = {Ben-Dov, Omri and Fawkes, Jake and Samadi, Samira and Sanyal, Amartya},
    booktitle = ICML,
    ignorepages = {3443--3461},
    title = {The Role of Learning Algorithms in Collective Action},
    no-url = {https://proceedings.mlr.press/v235/ben-dov24a.html},
    volume = {235},
    year = {2024}
}

@article{doi:10.1073/pnas.1903070116,
    author = {Mikhail Belkin  and Daniel Hsu  and Siyuan Ma  and Soumik Mandal },
    eprint = {https://www.pnas.org/doi/pdf/10.1073/pnas.1903070116},
    ignorepages = {15849-15854},
    journal = {Proceedings of the National Academy of Sciences},
    number = {32},
    title = {Reconciling modern machine-learning practice and the classical bias–variance trade-off},
    no-url = {https://www.pnas.org/doi/abs/10.1073/pnas.1903070116},
    volume = {116},
    year = {2019}
}

@InProceedings{pmlr-v162-jagadeesan22a,
  title = 	 {Regret Minimization with Performative Feedback},
  author =       {Jagadeesan, Meena and Zrnic, Tijana and Mendler-D{\"u}nner, Celestine},
  booktitle = ICML,
  year = 	 {2022},
  editor = 	 {Chaudhuri, Kamalika and Jegelka, Stefanie and Song, Le and Szepesvari, Csaba and Niu, Gang and Sabato, Sivan},
  volume = 	 {162},
  series = 	 {Proceedings of Machine Learning Research},
  publisher =    {PMLR},
  no-url = 	 {https://proceedings.mlr.press/v162/jagadeesan22a.html},
}

@misc{bracale2025learningdistributionmapreverse,
      title={Learning the Distribution Map in Reverse Causal Performative Prediction}, 
      author={Daniele Bracale and Subha Maity and Moulinath Banerjee and Yuekai Sun},
      year={2025},
      eprint={2405.15172},
      archivePrefix={arXiv},
      primaryClass={stat.ML},
      no-url={https://arxiv.org/abs/2405.15172}, 
}

@article{Tsoy2025OnTI,
  title={On the Impact of Performative Risk Minimization for Binary Random Variables},
  author={Nikita Tsoy and Ivan Kirev and Negin Rahimiyazdi and Nikola Konstantinov},
  journal=ICML,
  year={2025},
  volume={abs/2502.02331},
  no-url={https://api.semanticscholar.org/CorpusID:276107955}
}

@misc{demirel2024adjustingpretrainedbackbonesperformativity,
      title={Adjusting Pretrained Backbones for Performativity}, 
      author={Berker Demirel and Lingjing Kong and Kun Zhang and Theofanis Karaletsos and Celestine Mendler-Dünner and Francesco Locatello},
      year={2024},
      eprint={2410.04499},
      archivePrefix={arXiv},
      primaryClass={cs.LG},
      no-url={https://arxiv.org/abs/2410.04499}, 
}

@article{dropoutversusl2,
  author  = {Gabriel Clara and Sophie Langer and Johannes Schmidt-Hieber},
  title   = {Dropout Regularization Versus l2-Penalization in the Linear Model},
  journal = {Journal of Machine Learning Research},
  year    = {2024},
  volume  = {25},
  number  = {204},
  ignorepages   = {1--48},
  no-url     = {http://jmlr.org/papers/v25/23-0803.html}
}

@misc{wang2023constrainedoptimizationdecisiondependentdistributions,
      title={Constrained Optimization with Decision-Dependent Distributions}, 
      author={Zifan Wang and Changxin Liu and Thomas Parisini and Michael M. Zavlanos and Karl H. Johansson},
      year={2023},
      eprint={2310.02384},
      archivePrefix={arXiv},
      primaryClass={math.OC},
      no-url={https://arxiv.org/abs/2310.02384}, 
}

@article{Cai2021label,
  title = {Transfer learning for nonparametric classification: Minimax rate and adaptive classifier},
  volume = {49},
  ISSN = {0090-5364},
  url = {http://dx.doi.org/10.1214/20-AOS1949},
  DOI = {10.1214/20-aos1949},
  number = {1},
  journal = {The Annals of Statistics},
  publisher = {Institute of Mathematical Statistics},
  author = {Cai,  T. Tony and Wei,  Hongji},
  year = {2021},
  month = feb 
}

@inproceedings{label2,
author = {Zhu, Yilun and Zhang, Jianxin and Gangrade, Aditya and Scott, Clayton},
title = {Label noise: ignorance is bliss},
year = {2024},
isbn = {9798331314385},
publisher = {Curran Associates Inc.},
address = {Red Hook, NY, USA},
booktitle = NEURIPS,
articleno = {3701},
numpages = {42},
location = {Vancouver, BC, Canada},
series = {NIPS '24}
}

@misc{nikolalabel,
      title={PAC Learnability in the Presence of Performativity}, 
      author={Ivan Kirev and Lyuben Baltadzhiev and Nikola Konstantinov},
      year={2026},
      eprint={2510.08335},
      archivePrefix={arXiv},
      primaryClass={stat.ML},
      url={https://arxiv.org/abs/2510.08335}, 
}

@InProceedings{pmlr-v206-mofakhami23a,
  title = 	 {Performative Prediction with Neural Networks},
  author =       {Mofakhami, Mehrnaz and Mitliagkas, Ioannis and Gidel, Gauthier},
  booktitle = 	 {Proceedings of The 26th International Conference on Artificial Intelligence and Statistics},
  pages = 	 {11079--11093},
  year = 	 {2023},
  editor = 	 {Ruiz, Francisco and Dy, Jennifer and van de Meent, Jan-Willem},
  volume = 	 {206},
  series = 	 {Proceedings of Machine Learning Research},
  month = 	 {25--27 Apr},
  publisher =    {PMLR},
  url = 	 {https://proceedings.mlr.press/v206/mofakhami23a.html}
}
